\documentclass[A4paper,10pt]{article}

\usepackage[english]{babel}

\usepackage{hyperref}

\usepackage{amsfonts}
\usepackage{amsmath}
\usepackage{amsthm}
\usepackage{amssymb}
\usepackage{indentfirst}
\usepackage{color}
\usepackage{tabularx}
\usepackage{graphicx}

\usepackage[a4paper,top=3.5cm,bottom=1cm,left=2.5cm,right=2.5cm]{geometry}

\usepackage{subcaption}

\usepackage{mathrsfs}
\usepackage{mathtools}
\usepackage{enumitem}
\usepackage[normalem]{ulem}

\usepackage[utf8]{inputenc}
\usepackage[T1]{fontenc}
\usepackage[english]{babel}

\numberwithin{equation}{section}

\theoremstyle{plain}
\begingroup
\theoremstyle{plain}
\newtheorem{theorem}{Theorem}[section]
\newtheorem{corollary}[theorem]{Corollary}
\newtheorem{proposition}[theorem]{Proposition}
\newtheorem{lemma}[theorem]{Lemma}
\theoremstyle{definition}
\newtheorem{definition}[theorem]{Definition}

\theoremstyle{remark}
\newtheorem{remark}[theorem]{Remark}
\endgroup

\theoremstyle{definition}
\theoremstyle{remark}

\mathchardef\emptyset="001F

\newcommand{\R}{\mathbb{R}}
\newcommand{\E}{\mathcal{E}}

\newcommand{\B}{\mathrm{B}}
\newcommand{\dive}{\mathrm{div}}
\newcommand{\supp}{\mathrm{supp}}

\newcommand{\di}{\mathrm{d}}

\newcommand{\I}{\mathrm{I}}

\newcommand{\F}{\mathcal{F}}

\newcommand{\J}{\mathcal{J}}

\newcommand{\sP}{\mathcal{P}}
\newcommand{\Lip}{\mathrm{Lip}}

\newcommand{\coloneq }{\hspace{1pt}\raisebox{0.74pt}{\scalebox{0.8}{:}}\hspace{-2.2pt}=}

\title{Data-driven Evolutions of Critical Points}

\author{Stefano Almi\footnote{Department of Mathematics, University of Vienna, Oskar-Morgenstern-Platz 1, 1090 Vienna, Austria, Email: stefano.almi@univie.ac.at}, 
Massimo Fornasier\footnote{Department of Mathematics, Technical University of Munich, Boltzmannstrasse 3, 85748 Garching bei M\"unchen, Germany, Email: massimo.fornasier@ma.tum.de}, 
and Richard Huber\footnote{Institute of Mathematics and Scientific Computing, Karl-Franzens University of Graz, Heinrichstrasse 36/III, 8010 Graz, Austria, Email: richard.huber@uni-graz.at}}

\date{}

\begin{document}
\maketitle

\begin{abstract}
In this paper we are concerned with the learnability of energies from data obtained by observing time evolutions of their critical points starting at random initial equilibria. 
As a byproduct of our theoretical framework we introduce the novel concept of {\it mean-field limit of critical point evolutions} and of their energy balance as a new form of transport. We formulate the energy learning as a variational problem, minimizing the discrepancy of energy competitors from fulfilling the equilibrium condition along any trajectory of critical points originated at random initial equilibria.  By $\Gamma$-convergence arguments we prove the convergence of minimal solutions obtained from finite number of observations to the exact energy in a suitable sense. 
The abstract framework is actually fully constructive and numerically implementable. Hence, the approximation of the energy from a finite number of observations of past evolutions allows to simulate further evolutions, which are fully data-driven.
As we aim at a precise quantitative analysis, and to provide concrete examples of tractable solutions, we present analytic and numerical results on the reconstruction of an elastic energy for a one-dimensional model of thin nonlinear-elastic rod.
\end{abstract}

\tableofcontents

\section{Introduction}\label{s.intro}

\subsection{Evolutions of critical points}\label{s.ecp}

Many time-dependent phenomena in physics, biology,  social, and economical sciences as well as iterative algorithms in machine learning can be modelled by a function $x \colon[0,T] \to \mathcal H$, where~$\mathcal H$ represents the space of states of the physical, biological, social system, or digital data, which evolves from an initial configuration $x(0)=x_0$  towards a more convenient state or a new equilibrium.  The space $\mathcal H$ can  be a conveniently chosen Hilbert space. 
This often implicitly assumes that~$x$ evolves driven by a minimization process of a potential energy $E\colon [0,T]  \times \mathcal H \to \mathbb R$.  In this preliminary introduction we consciously avoid specific assumptions on~$E$, as we wish to keep a rather general view. 
Inspired by physics for which conservative forces are the derivatives of the potential energies, one can often describe the evolution as satisfying a gradient flow equation of the type
\begin{equation}\label{gradientflow}
\dot x(t) = - \nabla_x E(t,x(t)), \quad x(0)=x_0 \,,
\end{equation}
where~$\nabla_x E(t,x)$ is some notion of differential of~$E$ (in the simplest case~$\nabla_x$ may represent the Frech\'et derivative of the energy~$E$; in other cases it might already take into consideration additional constraints which are binding the states to a certain sets, i.e., $x(t) \in \mathcal K(t) \subset \mathcal H$).

Physical systems naturally tend to minimize the potential energy. For this fundamental reason the study of steady states in physical systems or critical points of the energy is of utmost relevance, given the expected frequency for such states to occur. However, once a critical point $x^*$ is reached, i.e., $\nabla_x  E(t^*,x^*)=0$, the dynamics is not supposed to further progress, unless some of the constraining conditions are changing, leading to a modified energetic profile. In this case, the evolution would restart and tend again by gradient flow to another critical point satisfying the new constraints.  In view of the relevance of critical points, it is often of interest to exclusively observe their dynamics, rather than record the transitions between them.  If we imagine now to collapse to an instant the time of realization of the - microscopically in time -  gradient descent evolution, we could interpret the dynamics - macroscopically in time - as the instantaneous hopping from a critical point to another critical point. This time re-parametrization can be rather conveniently realized as limit for $\varepsilon \to 0$ of a singularly perturbed version of \eqref{gradientflow}
\begin{equation}\label{e.gf}
\left\{\begin{array}{ll}
\varepsilon\dot{x}(t)=-\nabla_{x}E(t,x(t))\,,\\[1mm]
x(0)=x_{0}\,,
\end{array}\right.
\end{equation}
for a rescaling parameter $\varepsilon>0$ and a choice of $x_0$ fulfilling the criticality condition $\nabla_x E(0,x_0)=0$. In view of the vanishing parameter~$\varepsilon$, the trajectories would tend in the limit to have unbounded velocity (in the rescaled time) and therefore classical compactness arguments, such as Ascoli-Arzel\'a Theorem, would fail to characterize limit trajectories for $\varepsilon \to 0$.  Luckily recent works \cite{MR3705699, SS20172, SS2017, MR2318261} explored {\it ad hoc} compactness methods along solution trajectories~$x_{\varepsilon}$ under suitable smoothness assumptions on the energy $E$ and certain generic conditions, so-called {\it transversality conditions} \cite{nla.cat-vn1161769,MR3324377,MR2318261}, on the sets of critical points $C(t)=\{\nabla_x E(t,x)=0\}$ (compare assumptions (E1)-(E4) below). The more restrictive assumption of all is perhaps the request for the state space~$\mathcal H$ of being of finite dimension, i.e., $\mathcal H = \mathbb R^d$. 
We are informed of work in progress~\cite{pers_comm}, which will relax this latter request to arbitrary Hilbert spaces, but in this present paper we will restrict ourselves to the available compactness results in~\cite{MR3705699}. Therefore, we will assume throughout this paper that indeed $\mathcal H = \mathbb R^d$, which is enough for most numerical applications. In fact, any problem with infinite dimensional state space would eventually need to be discretized and reduced to finite dimensions in order to be numerically computable. The main compactness result in~\cite{MR3705699} may be summarized as follows

\begin{theorem}[Agostiniani and Rossi, 2017] \label{AR-Thm}
Let~$\varepsilon_{n}\to0$, let $x_{0}^{n} \to x_{0} $ in~$\R^{d}$, and let~$x_{\varepsilon_{n}}$ be the solution of~\eqref{e.gf} associated to~$\varepsilon_{n}$ and to the initial condition~$(x_{0}^{n})$. Then, for all $1 \leq p < \infty$, there exists a trajectory $x\in L^p((0,T),\mathbb R^d)$ and a positive Radon measure $\nu\in \mathcal{M}^{+}_{b}(0,T)$ such that the following properties hold:
\begin{itemize}
\item[(a)] up to a subsequence, $x_{\varepsilon_{n}}\to x$ in~$L^p((0,T),\mathbb R^d)$ for every $p\in[1,+\infty)$, and pointwise for all $t\in [0,T]$;

 \item[(b)] for every $t\in[0,T]$ the pointwise limit function $x(\cdot)$ constructed in (a) admits left and right limits $x_{-}(t)$ and~$x_{+}(t)$, respectively, and $x_{\pm}(t) \in C(t)$;
 
 \item[(c)] the set $J\coloneq \{t\in[0,T]:\, \nu(\{t\}) >0\}$ is at most countable and coincides with the set of discontinuity points of~$x(\cdot)$;

\item[(d)] for every $s,t \in [0,T]$ it holds
\begin{equation}\label{e.enbalance0}
{E}(t,x_{+}(t)) + \nu([s,t]) = {E}(s, x_{-}(s)) + \int_{s}^{t} \partial_{t} E(\tau, x(\tau))\di \tau\,.
\end{equation}
\end{itemize}
\end{theorem}

\subsection{Learning the energy from observation of the dynamics}

The evolution of critical points $t \mapsto x(t)$ obtained by Theorem~\ref{AR-Thm} fulfills in particular the energy conservation principle~\eqref{e.enbalance0} and it is fully driven and explained by the energy~$E$ itself.
In some relevant cases the   energy  that  governs  a  system  can be derived  theoretically or accurately measured experimentally, as it happens in the first principle physics; in most of the cases, the energy needs to be approximated by solving the inverse problem of fitting the data.  Model selection and parameter estimation methods are employed to determine  the form of the governing energy. Data-driven estimations are needed, for instance, in training algorithms in machine learning \cite{NIPS2018_7892,ewen17,ewen18,GERSTBERGER1997117}, and in data assimilation for models in continuum mechanics~\cite{KIRCHDOERFER201681,MR3799091} computational sociology~\cite{BFHM16,Lu14424,lu2019learning} or economics~\cite{DDvolatility,cre03,Egger_2005}.
However, even the problem of determining whether time shots of a linear dynamical systems do fulfill physically meaningful models, in particular have Markowian dynamics, is computationally intractable~\cite{PhysRevLett.108.120503}. 
For nonlinear models, the intractability of learning the system corresponds to the complexity of determining  the  set  of  appropriate  candidate  functions  to  fit    the  data. In order to break the curse of dimensionality of learning dynamical systems, one requires prior
knowledge  on  the  system  and  the  potential  structure  of  the  governing  equations.
For instance, in the sequence of recent  papers~\cite{trwa16,sctrwa17-1,sctrwa17} the authors assume that the governing equations are of first order and  can be written as sparse polynomials, i.e., linear combinations of few monomial terms.

In this work we aim at bridging, in the specific setting of {deterministic} evolutions of critical points, the well-developed theory of mean-field equations with modern  approaches of approximation theory and machine learning. We provide a mathematical framework for the reliable identification of the governing energy from data obtained by direct observations of corresponding time-dependent evolutions. We would like to obtain results which ensure the learning of energies without the need of more restrictive assumptions than (E1)-(E4). Moreover,  the approximation of the energy from a finite number of observations of past evolutions allows to simulate further evolutions, which are then fully data-driven.

First of all, we need to formalize what we mean by {\it observations of  time evolutions}. In this paper we will assume that we are allowed to observe multiple realizations of evolutions of critical points for the same energy function~$E$, starting from different critical points. For that, we need to further modify the model by a suitable correction
\begin{equation}\label{e.gf2}
\left\{\begin{array}{ll}
\varepsilon\dot{x}(t)=-\nabla_{x}E(t,x(t))+\nabla_{x}E(0,x_{0})\,,\\[1mm]
x(0)=x_{0}\,.
\end{array}\right.
\end{equation}
so that, whatever is the choice of $x_0$, the system starts from an equilibrium. This simply means assuming that an additional force is added at the beginning to allow the state~$x_0$ to be an equilibrium. Then we allow ourselves to draw at random independently several instances of the initial conditions $x_0^1, \dots, x_0^N, \dots$ according to a fixed  probability distribution $\mu_0 \in \mathcal P_c(\mathbb R^d)$ with compact support. For each of the picked initial conditions we can finally observe corresponding evolutions of critical points
$t \mapsto x^i(t)$, for $i=1,\dots,N, \dots$. We need then to devise a constructive method to infer the energy~$E$ from the observed trajectories. Our approach goes through five fundamental theoretical results:\\

\noindent  {\bf 1. Compactness of controlled evolutions of critical points.}
In the equation~\eqref{e.gf2} a correction force has been added to ensure an arbitrary initial datum~$x_0$ to be an equilibrium. As two trajectories~$x(t), \tilde x(t)$ originating from distinct initial equilibria $x_0, \tilde x_0$, respectively, may intersect at any $t \in [0,T]$ and to promote a unique flow along characteristics (see Remark~\ref{r.2} and Remark~\ref{r.3.2}), we modify further the model to take the form of an augmented controlled system:
\begin{equation}\label{e.gf3}
\left\{\begin{array}{lll}
\varepsilon\dot{x}(t)=-\nabla_{x}E(t,x(t))+u(t)\,,\\[1mm]
\dot{u}(t)=f(u(t))\,,\\[1mm]
(x(0),u(0))=(x_{0},u_{0})\in\R^{d}\times\R^{d}\,,
\end{array}\right.
\end{equation}
The particular the choice of $f \equiv 0$ and $u_0= \nabla_x E(0,x_0)$ yields back~\eqref{e.gf2}. Our first result, Theorem~\ref{t.2}, is the generalization of the compactness Theorem~\ref{AR-Thm} to the controlled system~\eqref{e.gf3}. 
Although the system is not anymore in the form of gradient flow and Theorem~\ref{AR-Thm} cannot be directly applied, we show that the techniques in~\cite{MR3705699} can be adapted to~\eqref{e.gf3} without introducing significant technical issues.
\\
\noindent {\bf 2. Mean-field limit of evolutions of critical points.} 
While the initial condition~$x_0$ is distributed as~$\mu_0$, we need then to clarify how the trajectories of critical points $x(t)$ are distributed at any time $t \in [0,T]$. Informally, we should explain how the initial probability distribution~$\mu_0$ gets {\it transported along trajectories of critical points} to the probability distribution~$\mu(t)$ at any time $t \in [0,T]$, so that $x(t) \sim \mu(t)$. We approach this issue under the modeling assumption that the evolutions of critical points are the result of singularly perturbed limit of systems of the type~\eqref{e.gf3}. In fact, for $\varepsilon>0$ established results in gradient flow theory~\cite{MR2401600} allows to describe the evolution of any system of the type~\eqref{e.gf3} (assume here  for simplicity $f \equiv 0$ and $u_0= \nabla_x E(0,x_0)$) by considering solutions $\eta_\varepsilon \in AC([0,T],\mathcal P_c(\mathbb R^d\times\mathbb R^d ))$ of mean-field equations of the type (see~\eqref{e.5} below)
$$
\varepsilon \partial_t \eta_\varepsilon(t,x,u) = \operatorname{div}_x ((\nabla_x E(t,x) - u)\eta_\varepsilon(t,x,u)), \quad \eta_\varepsilon(0) = (\operatorname{id},E(0,\cdot))_\#\mu_0 \,.
$$
In Section~\ref{s.L}  we take advantage of the newly established compactness argument Theorem~\ref{t.2} and the superposition principle introduced in~\cite[Theorem 8.2.1]{MR2401600} to derive Theorem~\ref{t.1} and Proposition~\ref{p.3.1} to describe the probability valued trajectory $t \mapsto \eta(t)$ (below we use equivalently also the notation $\eta_t = \eta(t)$) representing the time dependent distribution of  evolutions of critical points as a suitable form of limit of $t \mapsto\eta_\varepsilon(t)$ for $\varepsilon \to 0$. The main characterization of the limit is given by
\begin{equation}\label{char}
\int_{0}^{T}\int_{\R^{d}\times\R^{d}} |\nabla_{x}E(t,x)-u|^{2}  \di \eta_{t}(x,u)\, \di t=0 \,,
\end{equation}
which shows that the first marginal of $\eta(t)$ is supported on critical-type points.
\\
\noindent {\bf 3. Mean-field limit of the energy balance.} We further show that the evolution $t \mapsto \eta_t$ is also fulfilling in a suitable sense a generalization of the energy balance~\eqref{e.enbalance0},  Theorem~\ref{MF_EB_thm}, which explains how the energy~$E(t,x(t))$ is actually distributed at the time $t \in [0,T]$ for the initial condition $x_0 \sim \mu_0$. The result is obtained by a simple, but also thoughtful,  
reformulation of the energy balance using a Lebesgue charaterization of left and right limits and the use again of the compactness argument Theorem~\ref{t.2}. To our knowledge, Theorem~\ref{t.1}, Proposition~\ref{p.3.1}, and Theorem~\ref{MF_EB_thm} are the first form of mean-field limit of evolutions of critical points available in the literature. 
\\
\noindent {\bf 4. A variational model for the energy learning.}  Inspired by the characterization \eqref{char}, we formulate the problem of learning the true energy $E$ responsible of driving the dynamics from observations of evolutions of critical points as the minimization of the functional 
\begin{equation}\label{en_func}
\mathcal J_\eta(\hat E) = \int_{0}^{T}\int_{\R^{d}\times\R^{d}} |\nabla_{x} \hat E(t,x)-u|^{2}  \di \eta_{t}(x,u)\, \di t \,,
\end{equation}
on a suitable compact class in~$W^{2,\infty}$ of competitor energies~$\hat E$.
 To make our approach fully constructive we actually assume to observe only a finite number $N$ of evolutions of critical points and use a finite dimensional set $V_N \Subset W^{2,\infty}$ of competitors~$\hat E$. By a $\Gamma$-convergence argument  for $N\to \infty$, we derive in Theorem~\ref{p.10} an approximation result of the true energy~$E$, which was driving the observed  dynamics. \\
 \noindent {\bf 5. Data driven evolutions of critical points.} Once the energy is learned, it is then possible with the estimated energy $\hat E$ to simulate further evolutions. Corollary~\ref{thm:DDE} guarantees that the simulated evolutions, which are fully data-driven,  approximate ``true'' evolutions that would have been generated by using the original energy $E$. 
\\

The abstract framework described by the steps 1.-5. is actually fully constructive and numerically implementable. As we aim at a precise quantitative analysis, and to provide an example of tractable solutions, we approach the learning of the governing energy of the evolution for specific models inspired by continuum mechanics. In particular, in Section~\ref{s.example} we present analytic and numerical results on the reconstruction of the nonlinear elastic energy for a one-dimensional model of thin elastic rod. 

Let us stress that this particular example is by no means the unique possible application of our general framework and we envisage many other possible applications for data-driven models in physics, biology,  social, and economical sciences as well as training algorithms in machine learning.

\medskip

\noindent  {\bf Notation.} Given $T\in(0,+\infty)$ and $p\in[1,+\infty)$, we denote with~$\Gamma^{p}_{T}$ the space $L^{p}([0,T];\R^{d})$. Furthermore, we set $\Lambda \coloneq C([0,T];\R^{d})$.

Let $(X, d)$ be a separable metric space. We denote with $\mathcal{M}_{b}(X)$ the set of bounded Radon measures on~$X$ and with $\mathcal{M}^{+}_b(X)$ the subset of positive bounded Radon measures. The symbol $\mathcal{P}(X)$ stands for the set of probability measures on~$X$, $\mathcal{P}_{c}(X)$ indicates the set of probability measures with compact support in~$X$, and $\mathcal{P}_{1}(X)$ denotes the set of probability measures with bounded first moment, i.e., measures $\mu\in \mathcal{P}(X)$ such that
\begin{displaymath}
\int_{X} d(x, \bar{x}) \,\di \mu (x) <+\infty \qquad\text{for some $\bar{x} \in X$}.
\end{displaymath}

Let $(Y, d')$ be another separable metric space,  $r\colon X \to Y$ a Borel map, and $\mu \in \mathcal{M}_{b}(X)$. We define the push-forward $r_{\#}\mu \in \mathcal{M}_{b}(Y)$ of~$\mu$ through~$r$ by the relation $r_{\#}\mu (B) \coloneq \mu(r^{-1}(B))$ for every~$B$ Borel subset of~$Y$. For every $\mu, \nu \in \mathcal{P}_{1}(X)$, the $1$-Wasserstein distance $W_{1}(\mu, \nu)$ is defined by (see, e.g.,~\cite[Section~7.1]{MR2401600})
\begin{displaymath}
W_{1}(\mu, \nu)\coloneq \inf \Big\{ \int_{X} d(x,y) \, \di \gamma(x, y):\, \gamma \in \Gamma(\mu, \nu)\Big\}\,,
\end{displaymath}
where $\Gamma(\mu, \nu) \coloneq \{\gamma\in \mathcal{P}(X\times X):\, (\pi_{1})_{\#}\gamma = \mu \text{ and } (\pi_{2})_{\#}\gamma = \nu\}$ and $\pi_{i}\colon X\times X \to X$, $\pi_{i}(x_{1}, x_{2}) = x_{i}$ for $i=1, 2$. We also recall that if $\mu, \nu \in \mathcal{P}_{c}(X)$ it holds
\begin{displaymath}
W_{1}(\mu, \nu) = \sup \Big\{ \int_{X} \varphi(x) \, \di (\mu - \nu) : \, \text{$\varphi \in \rm{Lip}(X, \R)$ such that  $\rm{Lip(\varphi)}\leq1$}\Big\}\,.
\end{displaymath}

Finally, given an interval $I\subseteq \R$, we denote with~$BV(I)$  the space of functions of bounded variations in~$I$, that is, the space of $L^1_{\rm{loc}}$ functions~$v \colon I \to \R$ whose distributional derivative $\mathrm{D} v$ belongs to~$\mathcal{M}_{b}(I)$. We refer, for instance, to~\cite[Section~3.2]{Ambrosio2000} for further details.


\section{Controlled evolutions of critical points}\label{s.ODE}

\subsection{Main assumptions and motivations}

We start by fixing the main assumptions, which we will hold valid for the rest of the paper.
 Let $T>0$ and let us consider an energy functional $E\colon [0,T]\times\R^{d}\to\R$ satisfying the following conditions:
\begin{itemize}
\item[(E1)] $E\in C^{1,1}_{\rm{loc}}([0,T]\times\R^{d})$, meaning that $E\in C^{1}([0,T]\times\R^{d})$ with $\nabla_{x}E$ locally Lipschitz continuous in~$\R^{d}$;

\item[(E2)] there exist $C_{1}, C_{2}>0$ such that for every $(t,x)\in[0,T]\times\R^{d}$
\begin{displaymath}
|\partial_{t} E(t,x)|\leq C_{1}+ C_{2} E(t,x)\,;
\end{displaymath}

\item[(E3)] there exist $C_{3},\, C_{4} > 0$ and $p>1$ such that for every $(t,x)\in[0,T]\times\R^{d}$
\begin{displaymath}
C_{3}|x|^{p} - C_{4}\leq E(t,x)\,.
\end{displaymath}

\item[(E4)] for every $(t,u)\in [0,T] \times \R^{d}$, the set
\begin{displaymath}
C(t,u)\coloneq\{x\in\R^{d}:\,\nabla_{x}E(t,x)=u\}
\end{displaymath}
contains only isolated points.
\end{itemize}

\begin{remark}
Let us briefly comment on the above assumptions. Hypotheses (E2) is typical in the framework of evolutions of critical points or rate-independent systems, and it is useful to prove the boundedness of trajectories exploiting Gronwall's type arguments. Condition~(E3) implies the more common compactness of sublevels of the driving energy~$E$. The need of an explicit bound from below will be clarified below (see in particular~\eqref{e.gradflow}, Lemma~\ref{l.1}, and Proposition~\ref{p.1}). 
Finally, assumption~(E4) has been considered, e.g., in~\cite{MR3705699, SS20172, SS2017, MR2318261}, in the uncontrolled case $u=0$ of equation \eqref{e.gf}. This kind of hypothesis has been proven, at the state of the art, to be quite useful to show compactness of trajectories of~\eqref{e.gf} in the limit as~$\varepsilon\to 0$. In this paper, we focus on the perturbed system~\eqref{e.gradflow}, where a control~$u\in\R^{d}$ is added. In order to  be able again to show compactness of trajectories of~\eqref{e.gradflow} as~$\varepsilon\to 0$, we need the stronger requirement~(E4). We refer to Section~\ref{s.CT} for a discussion on the compactness issue. Roughly speaking, we need that the energy~$E(t,\cdot)$ has no affine regions, for every~$t\in[0,T]$. We remark that~(E4) and its corresponding assumption~$(E_3)$ in~\cite{MR3705699} are both technical and likely equivalently ``artificial''. It is indeed clear that~(E4) implies~$(E_3)$. On the other hand, if~$E$ does not satisfy~(E4) for some~$u$, then the linear perturbation~$E(t,x) - u\cdot x$ does not satisfy~$(E_3)$.
\end{remark}

Here we are interested in studying the system \eqref{e.gf2}. As already mentioned in the introduction, the reason for adding the term~$\nabla_{x}E(0,x_{0})$ to the usual gradient flow system~\eqref{e.gf} is twofold. On the one hand, in the limit as~$\varepsilon\to 0$ we want to avoid jump discontinuities at time~$t=0$ and ensure that $x_0$ is an equilibrium from the very beginning. Since the limits of trajectories of~\eqref{e.gf2} are expected to satisfy $\nabla_{x}E(t,x(t))=\nabla_{x}E(0,x_{0})$, jumps at $t=0$ will not appear. On the other hand, the drift~$\nabla_{x}E(0,x_{0})$ can be exploited to add randomness to~\eqref{e.gf}. This can be done simply by assuming that the initial data are distributed according to a certain probability measure $\mu_{0}\in\sP(\R^{d})$. In fact, in what follows, we aim first at obtaining the mean-field description of~\eqref{e.gf2} for fixed~$\varepsilon>0$ (Section~\ref{s.CE}, standard), and then pass to the limit in the mean-field (or continuity) equation as $\varepsilon\to0$.

\begin{remark}\label{r.traj}
On the one hand, we notice that for every initial datum $x_{0}\in\R^{d}$ there exists unique a solution to~\eqref{e.gf2}. On the other hand, we could show easily some examples of energies~$E$ satisfying (E1)-(E4) and such that, for two different initial data $x_{0}, \tilde x_{0}\in\R^{d}$, the corresponding solutions of~\eqref{e.gf2} cross each other at some time $t\in(0,T)$. For this reason, it is more convenient to study~\eqref{e.gf2} and its mean-field limit in terms of pairs curve-initial datum $(x(\cdot), x_{0})$, in order to ensure uniqueness of transport along characteristics, see Remark~\ref{r.2} below.
\end{remark}

In view of the above comments, we consider the more general system
\begin{equation}\label{e.gradflow}
\left\{\begin{array}{lll}
\varepsilon\dot{x}(t)=-\nabla_{x}E(t,x(t))+u(t)\,,\\[1mm]
\dot{u}(t)=f(u(t))\,,\\[1mm]
(x(0),u(0))=(x_{0},u_{0})\in\R^{d}\times\R^{d}\,.
\end{array}\right. 
\end{equation}
We further assume 
\begin{itemize}
\item[(E5)] $f \in \rm{Lip}_{\rm{loc}} (\R^d,\R^d)$ with linear growth
$
|f( u) | \leq C_f (1+ |u|).
$
\end{itemize} 
Assumption (E5) ensures well-posedness and continuity of solutions from  initial data of 
\begin{displaymath}
\left\{\begin{array}{ll}
\dot{u}(t)=f(u(t))\,,\\[1mm]
u(0)=u_{0}
\end{array}\right.
\end{displaymath}
in the time interval~$[0,T]$. 

\begin{remark}\label{r.5}
The system~\eqref{e.gradflow} is a generalization of~\eqref{e.gf2} in the pair $(x(\cdot),x_{0})$. We have indeed substituted the drift $\nabla_{x}E(0,x_{0})$ with a control parameter~$u$ governed by the ODE $\dot{u}=f(u)$. Clearly,~\eqref{e.gf2} can be recovered by setting $f\equiv0$ and $u_{0}=\nabla_{x}E(0,x_{0})$ in~\eqref{e.gradflow}.
\end{remark}

For every $(t,x,u)\in[0,T]\times\R^{d}\times\R^{d}$ we define the corrected energy function $\overline{E}(t,x,u)\coloneq E(t,x)-u{\,\cdot\,} x$. In the following lemma, we collect the properties of~$\overline{E}$.

\begin{lemma}\label{l.1}
The following facts hold:
\begin{itemize}
\item[$(a)$] for every compact subset~$K$ of~$\R^{d}$, there exist $C_{K,1}, C_{K,2}>0$ such that
\begin{equation}\label{e.1}
|\partial_{t}\overline{E}(t,x,u)|\leq C_{K,1}+C_{K,2}\overline{E}(t,x,u)\qquad\text{for every $x\in\R^{d}$ and every $u\in K$}\,;
\end{equation}

\item[$(b)$] $\overline{E}\in C^{1,1}_{loc}([0,T]\times\R^{d}\times\R^{d})$.

\end{itemize}
\end{lemma}

\begin{proof}
Property~$(b)$ follows from the smoothness of~$E$. Statement~$(a)$ follows from~(E2) and~(E3). Indeed,
\begin{displaymath}
\begin{split}
\partial_{t}\overline{E}(t,x,u)&=\partial_{t}E(t,x) \leq C_{2} E(t,x)+C_{1} = C_{2}\overline{E}(t,x,u)+C_1+C_{\delta}|u|^{p'}+\delta|x|^{p}\\
&\leq C_{2}\overline{E}(t,x,u)+C_{1}+C_{4}\delta/C_3+ C_{\delta}|u|^{p'}+\delta/C_3 E(t,x)\,.
\end{split}
\end{displaymath}
By construction of~$\overline{E}$, we have that
\begin{displaymath}
E(t,x)\leq \overline{E}(t,x,u)+|u||x|\leq\overline{E}(t,x,u)+C_{\delta}|u|^{p'}+\delta|x|^{p}\leq \overline{E}(t,x,u) +C_{4}\delta/C_3+ C_{\delta}|u|^{p'}+ \delta/C_3 E(t,x)\,.
\end{displaymath}
Choosing~$\delta>0$ so small that~$\delta/C_3<1/2$, we get that
\begin{equation}\label{e.EbarE}
E(t,x)\leq 2\overline{E}(t,x,u)+ 2 C_{\delta}|u|^{p'}+ 2 C_{4}\delta/C_3\,.
\end{equation}
Therefore, we deduce~\eqref{e.1} as soon as $u \in K$.
\end{proof}

\begin{remark}\label{r.2}
Under hypotheses~(E1)-(E5), we have that for every pair of initial conditions~$(x_{0},u_{0})\in\R^{d}\times\R^{d}$ and every~$\varepsilon>0$ there exists unique a solution $(x_{\varepsilon}(\cdot),u_{\varepsilon}(\cdot))$ of system~\eqref{e.gradflow}. In particular, it is useful to introduce the flow map~$Y_{\varepsilon,t}\colon\R^{d}\times\R^{d}\to\R^{d}\times\R^{d}$ defined as $Y_{\varepsilon,t}(x_{0},u_{0})\coloneq (x_{\varepsilon}(t),u_{\varepsilon}(t))$, where~$(x_{\varepsilon}(\cdot),u_{\varepsilon}(\cdot))$ is the solution of~\eqref{e.gradflow} with initial condition~$(x_{0},u_{0})$.  Notice that $u_{\varepsilon}(\cdot)$ does not explicitly depend on~$\varepsilon>0$, only~$x_{\varepsilon}(\cdot)$ does.
\end{remark}

\begin{proposition}\label{p.1}
Let $(x_{0},u_{0})\in\R^{d}\times\R^{d}$ and let~$\varepsilon>0$. Let $(x_{\varepsilon}(\cdot),u_{\varepsilon}(\cdot))$ be the solution of~\eqref{e.gradflow}. Then, the following facts hold:
\begin{itemize}

\item[$(a)$] for every $s\leq t\in[0,T]$
\begin{equation}\label{e.2}
\begin{split}
\overline{E}(t,x_{\varepsilon}(t),u_{\varepsilon}(t))&=\overline{E}(s,x_{\varepsilon}(s),u_{\varepsilon}(s))+\int_{s}^{t}\partial_{\tau}\overline{E}(\tau,x_{\varepsilon}(\tau),u_{\varepsilon}(\tau))\, \di \tau-\frac{\varepsilon}{2}\int_{s}^{t}|\dot{x}_{\varepsilon}(\tau)|^{2}\, \di \tau\\
&\qquad-\frac{1}{2\varepsilon}\int_{s}^{t}|\nabla_{x}\overline{E}(\tau,x_{\varepsilon}(\tau),u_{\varepsilon}(\tau))|^{2}\, \di \tau - \int_{s}^{t}f(u_{\varepsilon}(\tau)){\,\cdot\,}x_{\varepsilon}(\tau)\, \di \tau\,;
\end{split}
\end{equation}

\item[$(b)$] $(x_{\varepsilon}(t),u_{\varepsilon}(t))$ is bounded in~$\R^{d}\times\R^{d}$, uniformly with respect to $t \in [0,T]$ and $\varepsilon>0$;

\item[$(c)$] there exists $\overline{C}>0$ independent of $\varepsilon$ such that
\begin{displaymath}
\int_{0}^{T}|\nabla_{x}\overline{E}(t,x_{\varepsilon}(t),u_{\varepsilon}(t))|^{2}\, \di t\leq \overline{C} \varepsilon\,.
\end{displaymath}
It is useful to observe that this bound implies
\begin{equation}\label{finitemass}
\int_{0}^{T} \left ( \frac{\varepsilon}{2} |\dot x_\varepsilon(t) |^2 + \frac{1}{2 \varepsilon} | \nabla_x \overline{E}(t,x_\varepsilon(t),u_\varepsilon(t)) |^2 \right) \, \di t\leq \overline{C}.
\end{equation}
\end{itemize}
\end{proposition}

\begin{proof} For~$\varepsilon>0$, the map $t \mapsto \overline{E}( t, x_{\varepsilon}(t), u_{\varepsilon}(t))$ is differentiable.  The energy balance in~$(a)$ follows by chain-rule, recalling the time derivatives as in~\eqref{e.gradflow}.  We report below the explicit computation as it will be quite useful several times below.
First of all we notice that from~\eqref{e.gradflow} 
\begin{displaymath}
\begin{split}
\frac{\varepsilon}{2} |\dot x_\varepsilon(t) |^2 &= \frac{1}{2 \varepsilon} | \nabla_x \overline{E}(t,x_\varepsilon(t),u_\varepsilon(t)) |^2, \quad \mbox{ or } \\
\frac{1}{\varepsilon}  | \nabla_x \overline{E}(t,x_\varepsilon(t),u_\varepsilon(t)) |^2 &=\frac{\varepsilon}{2} |\dot x_\varepsilon(t) |^2 + \frac{1}{2 \varepsilon} | \nabla_x \overline{E}(t,x_\varepsilon(t),u_\varepsilon(t)) |^2.
\end{split}
\end{displaymath}
Then
\begin{eqnarray*}
\frac{d}{dt} \overline{E}(t, x_\varepsilon(t),u_\varepsilon(t)) &=& \frac{\partial}{\partial t} \overline{E}(t, x_\varepsilon(t),u_\varepsilon(t)) - f(u_\varepsilon(t)) \cdot x_\varepsilon(t) - \frac{1}{\varepsilon}  |\nabla_x \overline{E}(t,x_\varepsilon(t),u_\varepsilon(t))|^2\\
&=&  \frac{\partial}{\partial t} \overline{E}(t, x_\varepsilon(t),u_\varepsilon(t)) - f(u_\varepsilon(t)) \cdot x_\varepsilon(t) - \frac{\varepsilon}{2} |\dot x_\varepsilon(t) |^2 - \frac{1}{2 \varepsilon} | \nabla_x \overline{E}(t,x_\varepsilon(t),u_\varepsilon(t)) |^2
\end{eqnarray*}
By integration we obtain (a). We address now (b).
Being $u_{\varepsilon}(t)$ uniformly bounded in terms of initial datum $u_{0}$, applying Lemma~\ref{l.1} and arguing as in~\eqref{e.EbarE} we get that
\begin{displaymath}
\overline{E}(t,x_{\varepsilon}(t),u_{\varepsilon}(t))\leq \overline{E}(0,x_0,u_0)+ C_{1} \int_{0}^{t} \overline{E}(\tau,x_{\varepsilon}(\tau),u_{\varepsilon}(\tau))\, \di \tau+ TC_{2}\,.
\end{displaymath}
for some $C_{1}, C_{2}>0$ independent of~$\varepsilon$. By Gronwall lemma,~$\overline{E}(t,x_{\varepsilon}(t),u_{\varepsilon}(t))\leq ( \overline{E}(0,x_0,u_0)+T C_{2}) e^{C_{1}T}$. In view of~(E3) and of the boundedness of~$u_{\varepsilon}$, we get that~$x_{\varepsilon}(t)$ is bounded in~$\R^{d}$ uniformly with respect to~$\varepsilon$ and $t \in [0,T]$ Hence also~$(b)$ is proved. Finally, property $(c)$ is a consequence of $(a)$ and $(b)$.
\end{proof}

\begin{remark}\label{r.unif}
The boundedness of the trajectories $(x_{\varepsilon}(\cdot), u_{\varepsilon}(\cdot))$ in Proposition~\ref{p.1} can be made independent of the specific initial datum if we consider initial data~$(x_{0}, u_{0})$ in a fixed compact subset $K^\circ$ of~$\R^{d}\times\R^{d}$. This assumption will be tacitly applied from now on.
\end{remark}

\subsection{Compactness of trajectories}\label{s.CT}

In this section we prove the compactness of trajectories~$(x_{\varepsilon}(\cdot),u_{\varepsilon}(\cdot))$ fulfilling~\eqref{e.gradflow} as~$\varepsilon\to0$. In particular, we show how to adapt the arguments of~\cite{MR3705699} in order to take into account also the control~$u$. In fact the second equation in~\eqref{e.gradflow} does not obey a gradient flow, hence does not allow a direct application of Theorem~\ref{AR-Thm}. For the sake of simplicity, from now on we set $\Lambda\coloneq C([0,T];\R^{d})$ and, for $p\in[1,+\infty)$, $\Gamma_{T}^{p}\coloneq L^{p}([0,T];\R^{d})$.

In order to describe the energetic behavior of a limit of~$(x_{\varepsilon}(\cdot),u_{\varepsilon}(\cdot))$ for $\varepsilon \to 0$, for every $t \in[0,T]$, every $x_{1}, x_{2}\in\R^{d}$, and every $u\in\R^{d}$ we define the cost function $c_{t}(x_{1}, x_{2}; u)$ as
\begin{equation}\label{e.cost}
c_{t}(x_{1},x_{2};u)\coloneq\left\{\begin{array}{ll}
\displaystyle \inf\Big\{\int_{0}^{1}|\nabla_{x}E(t,\theta(s))-u||\theta'(s)|\, ds:\,\theta\in\mathcal{A}^{t}_{x_{1},x_{2},u}\Big\}&\text{if $x_{1}\neq x_{2}$}\,,\\
0 & \text{otherwise,}
\end{array}\right.
\end{equation}
where~$\mathcal{A}^{t}_{x_{1},x_{2},u}$ is the set of admissible transitions from~$x_{1}$ to~$x_{2}$ at time~$t$
\begin{equation}\label{e.adm}
\begin{split}
\mathcal{A}^{t}_{x_{1},x_{2},u}\coloneq\{\theta & \in  C([0,1];\R^{d}):\text{$\theta(0)=x_{1}$, $\theta(1)=x_{2}$, there exists a partition $0=t_{0}<\ldots<t_{m}=1$}\\
 &\text{ such that $\theta\in \rm{Lip}_{\rm{loc}}((t_{i},t_{i+1});\R^{d})$, $\theta(t_{i})\in C(t,u)$ for $i=1,\ldots, m-1$,}\\
 & \text{ $\theta(s)\notin C(t,u)$ for $s\neq t_{i}$}\},
\end{split}
\end{equation}
where $C(t,u)$ is as in assumption~(E4).
Notice that the cost function~\eqref{e.cost} and the class of admissible transitions~\eqref{e.adm} are modified versions of corresponding of cost and admissible transitions in~\cite[Definition~2.2]{MR3705699}. As clearly stated in the following theorem, the cost~\eqref{e.cost} describes the energy dissipated by the limits of~$(x_{\varepsilon}(\cdot),u_{\varepsilon}(\cdot))$ at jump points.

With the notation introduced above, we can now state the main compactness result of this section. It generalizes Theorem \ref{AR-Thm} (see \cite[Theorem~1]{MR3705699}) for controlled systems of the type~\eqref{e.gradflow}.

\begin{theorem}\label{t.2}
Let~$(\varepsilon_{n})_{n \in \mathbb N}$ be a vanishing sequence, let $(x_{0}^{n},u_{0}^{n})\to (x_{0},u_{0})$ in~$\R^{d}\times\R^{d}$, and let $(x_{\varepsilon_{n}},u_{\varepsilon_{n}})$ be the solution of~\eqref{e.gradflow} associated to~$\varepsilon_{n}$ and to the initial condition~$(x_{0}^{n},u_{0}^{n})$. Then, there exists a pair $(x,u)\in\Gamma_{T}^{\infty}\times\Lambda$ and a positive Radon measure $\nu\in \mathcal{M}^{+}_{b}(0,T)$ such that the following properties hold:
\begin{itemize}
\item[(a)] up to a subsequence, $x_{\varepsilon_{n}}\to x$ in~$\Gamma_{T}^{p}$ for every $p\in[1,+\infty)$, and pointwise for all $t\in [0,T]$,
and $u_{\varepsilon_{n}}\to u$ uniformly in~$\Lambda$;

\item[(b)] $u$ satisfies $\dot{u}= f(u)$ with~$u(0)=u_{0}$;

 \item[(c)] for every $t\in[0,T]$ the pointwise limit function $x(\cdot)$ constructed in (a) admits left and right limits $x_{-}(t)$ and~$x_{+}(t)$, respectively, and $x_{\pm}(t) \in C(t,u(t))$;
 
 \item[(d)] the set $J\coloneq \{t\in[0,T]:\, \nu(\{t\}) >0\}$ is at most countable and coincides with the set of discontinuity points of~$x(\cdot)$;

\item[(e)] for every $s,t \in [0,T]$ it holds
\begin{equation}\label{e.enbalance}
\overline{E}(t,x_{+}(t), u(t)) + \nu([s,t]) = \overline{E}(s, x_{-}(s), u(s)) + \int_{s}^{t} [  \partial_{t} E(\tau, x(\tau)) - f(u(\tau))\cdot x(\tau)]\,\di \tau\,;
\end{equation}

\item[(e)] for every $t\in J$ we have $c_{t}(x_{-}(t), x_{+}(t); u(t)) = \nu(\{t\})$.
\end{itemize}
\end{theorem}

To prove Theorem~\ref{t.2} we follow the steps of~\cite{MR3705699}. For the reader convenience, we show the main changes in the proofs, but we may refer to \cite{MR3705699} for concluding details, which do not require modifications. We start with the analysis of useful properties of the cost function~$c_{t}$ (see Proposition~\ref{p.3} and Remarks~\ref{r.3} and~\ref{r.4.6} below).
\begin{lemma}\label{l.3}
Let~$K$ be a compact subset of~$\R^{d}$, $u\in\R^{d}$, and let~$t\in(0,T)$ be such that
\begin{displaymath}
\min_{x\in K}|\nabla_{x}E(t,x)-u|>0\,.
\end{displaymath} 
Then, there exists~$\alpha>0$ such that
\begin{displaymath}
\min_{\substack{ x\in K, v\in \overline{\B}(u,\alpha),\\ s\in[t-\alpha,t+\alpha]}}|\nabla_{x}E(s,x)-v|>0\,. 
\end{displaymath}
\end{lemma}

\begin{proof}
It follows from the continuity of the function 
\begin{displaymath}
(s,v)\mapsto\min_{x\in K}|\nabla_{x}E(s,x)-v|
\end{displaymath}
since~$E\in C^{1,1}_{\rm{loc}}$.
\end{proof}

We state and prove now a result, which generalizes~\cite[Proposition~4.1]{MR3705699}.

\begin{proposition}\label{p.3}
Let $t_{n}^{1},t_{n}^{2}\in[0,T]$ be such that $t_{n}^{i}\to t$  as $n \to \infty$, for $i=1,2$. Let $\theta_{n},u_{n}, u\in AC([t_{n}^{1},t_{n}^{2}];\R^{d})$ be such that $\theta_{n}(t^{i}_{n})\to x_i\in\R^{d}$ as $n \to \infty$ and $\dot u_{n}(s)=f(u_{n}(s))$ as  $s\in (t_{n}^{1},t_{n}^{2})$,
and $\lim_n \sup_{s\in [t_n^1, t_n^2]} |u_n(s)-u(s) |=0$. Then, the following facts hold:

\begin{itemize}

\item[$(a)$] if
\begin{equation}\label{e.010}
\liminf_{n}\int_{t^{1}_{n}}^{t^{2}_{n}}|\nabla_{x}E(s,\theta_{n}(s))-u_{n}(s)||\theta_{n}'(s)|\, \di s=0\,,
\end{equation}
then $x_{1}=x_{2}$;

\item[$(b)$] if $x_{1}\neq x_{2}$, then there exists $\theta\in\mathcal{A}^{t}_{x_{1},x_{2},u}$ such that
\begin{equation}\label{e.011}
\liminf_{n}\int_{t^{1}_{n}}^{t^{2}_{n}}|\nabla_{x}E(s,\theta_{n}(s))-u_{n}(s)||\theta_{n}'(s)|\, \di s\geq\int_{0}^{1}|\nabla_{x}E(t,\theta(s))-u||\theta'(s)|\, \di s\,.
\end{equation}
\end{itemize} 
\end{proposition}

\begin{proof}
Let us show~(a). First, we notice that, by~\eqref{e.1} and by an application of the chain rule,~$\theta_{n}(s)$ is uniformly bounded in~$\R^{d}$ for~$s\in[t^{1}_{n},t^{2}_{n}]$. Indeed, initial and final points $\theta_{n}(t^{i}_{n})$ converges to~$x_{i}$, and thus are bounded. Moreover, by chain rule, we have that for every $s\in[t^{1}_{n}, t^{2}_{n}]$
\begin{displaymath}
\begin{split}
\overline{E}(s , \theta_{n}(s), u_{n}(s) ) = & \ \overline{E}(t^{1}_{n} , \theta_{n}(t^{1}_{n}), u_{n}(t^{1}_{n})) + \int_{t^{1}_{n}}^{s} \partial_{t}E(\tau , \theta_{n} (\tau))\,\di \tau \\
& + \int_{t^{1}_{n}}^{s} (\nabla_{x}E(\tau,\theta_{n}(\tau))-u_{n}(\tau) ) \cdot \theta_{n}'(\tau) \, \di \tau - \int_{t^{1}_{n}}^{s} f (u_{n}(\tau))\cdot \theta_{n}(\tau) \, \di \tau \,.
\end{split}
\end{displaymath}
In view of Lemma~\ref{l.1} and of the assumption~\eqref{e.010}, we deduce that~$\theta_{n}(s)$ is uniformly bounded in~$\R^{d}$. Let us denote by~$\Theta$ the compact subset of~$\R^{d}$ containing~$\theta_{n}(s)$, $s\in[t_{n}^{1},t_{n}^{2}]$.

Let us assume by contradiction that~$x_{1}\neq x_{2}$. Thanks to condition~(E4), the set $C(t,u)\cap \Theta$ is finite. If this were not the case and $(x_i)_{i \in I}$ would be an infinite family of critical points, then we could extract from it a converging subsequence $x_k \to x \in C(t,u)\cap \Theta$, in view of the continuity of $\nabla_x E$, and $x$ would not be isolated, violating (E4).
Hence, there exists~$\delta>0$ such that
\begin{equation}\label{e.012}
\B(x,2\delta)\cap \B(y,2\delta)=\emptyset \qquad \text{for $x,y\in (C(t,u)\cap \Theta) \cup \{x_{1},x_{2}\}$ with $x\neq y$}\,.
\end{equation}
Let
\begin{equation}\label{e.013}
K_{\delta}\coloneq \Theta\setminus\bigcup_{x\in (\Theta\cap C(t,u))\cup\{x_{1},x_{2}\}}\B(x,\delta)\,.
\end{equation}
Then $K_{\delta}\subseteq \Theta$ is compact in~$\R^{d}$ and, by definition of~$C(t,u)$,
\begin{displaymath}
\min_{x\in K_{\delta}}|\nabla_{x}E(t,x)-u|>0\,.
\end{displaymath}
Applying Lemma~\ref{l.3}, we may assume, up to taking a smaller~$\delta>0$, that
\begin{equation}\label{e.014}
e_{\delta}\coloneq \min_{\substack{x\in K_{\delta}, v\in\overline{\B}(u,\delta),\\ s\in[t-\delta,t+\delta]}}|\nabla_{x}E(s,x)-v|>0\,.
\end{equation}
For~$n$ sufficiently large we have that $t_{n}^{i}\in[t-\delta,t+\delta]$ and $u_{n}(s)\in\B(u,\delta)$ for every $s\in[t^{1}_{n},t^{2}_{n}]$. By definition of~$K_{\delta}$ and by the previous properties, we have that the set $\{s\in[t_{n}^{1},t_{n}^{2}]:\,\theta_{n}(s)\in K_{\delta}\}\neq\emptyset$ and there exist $s_{1}, s_{2}\in\{s\in[t_{n}^{1},t_{n}^{2}]:\,\theta_{n}(s)\in K_{\delta}\}$ such that $s_{1}\neq s_{2}$ and $\theta_{n}(s_{i})\in\partial\B(x_{i},\delta)$ for $i=1,2$. Therefore, by~\eqref{e.014} we have
\begin{displaymath}
\begin{split}
\int_{t_{n}^{1}}^{t_{n}^{2}}|\nabla_{x}E(s,\theta_{n}(s))-u_{n}(s)||\theta'_{n}(s)|\, \di s& \geq\int_{\{s\in[t_{n}^{1},t_{n}^{2}]:\,\theta_{n}(s)\in K_{\delta}\}}|\nabla_{x}E(s,\theta_{n}(s))-u_{n}(s)||\theta_{n}'(s)|\, \di s\\
& \geq e_{\delta}\int_{\{s\in[t_{n}^{1},t_{n}^{2}]:\,\theta_{n}(s)\in K_{\delta}\}}|\theta'_{n}(s)|\, \di s\\
&\geq  e_{\delta} \min_{x,y\in C(t,u)\cap \Theta\cup\{x_{1},x_{2}\}}|x-y|-2\delta>0 \,. 
\end{split}
\end{displaymath}
This contradicts the hypothesis~\eqref{e.010}.

Let us now prove~$(b)$. Let $\delta$,~$K_{\delta}$, and~$e_{\delta}$ be as in~\eqref{e.012}-\eqref{e.014}. Up to extracting suitable subsequences, we can set
\begin{displaymath}
L\coloneq \lim_{n}\int_{t^{1}_{n}}^{t^{2}_{n}}|\nabla_{x}E(s,\theta_{n}(s))-u_{n}(s)||\theta'_{n}(s)|\, \di s>0\,.
\end{displaymath}
We reparametrize the time interval in the following way: we first define the strictly increasing function
\begin{displaymath}
s_{n}(r)\coloneq r+\int_{t_{n}^{1}}^{r}|\nabla_{x}E(\tau,\theta_{n}(\tau))-u_{n}(\tau)||\theta_{n}'(\tau)|\, \di \tau\,.
\end{displaymath}
Then we set $r_{n}=s_{n}^{-1}\colon [s_{n}^{1},s_{n}^{2}]\to[t_{n}^{1},t_{n}^{2}]$, where $s_{n}^{1}\coloneq s_{n}(t_{n}^{1})$ and $s_{n}^{2}\coloneq s_{n}(t^{2}_{n})$. We set $\tilde{\theta}_{n}(s)\coloneq\theta_{n}(r_{n}(s))$ and $\tilde{u}_{n}(s)\coloneq u_{n}(r_{n}(s))$. In particular, $\tilde{\theta}_{n}(s_{n}^{i})\to x_{i}$ and $\tilde{u}_{n}(\sigma)\to u$ uniformly for $\sigma\in[s^{1}_{n},s_{n}^{2}]$. 
Moreover, by change of variables,
\begin{equation}\label{e.015}
\int_{t_{n}^{1}}^{t_{n}^{2}}|\nabla_{x}E(s,\theta_{n}(s))-u_{n}(s)||\theta_{n}'(s)|\, \di s=\int_{s_{n}^{1}}^{s_{n}^{2}}|\nabla_{x}E(r_{n}(s),\tilde{\theta}_{n}(s))-\tilde{u}_{n}(s)||\tilde{\theta}_{n}'(s)|\, \di s\,.
\end{equation}
Notice now that
$$
r_n'(s)= (1 + |\nabla_x E(r_n(s),\tilde \theta_n(s))-  \tilde u_n(s)| |\theta_n'(r_n(s))|)^{-1}.
$$
hence
\begin{equation}
\begin{split}
|\nabla_{x}E(r_{n}(s),\tilde{\theta}_{n}(s))-\tilde{u}_{n}(s)||\tilde{\theta}_{n}'(s)| & = |\nabla_{x}E(r_{n}(s),\tilde{\theta}_{n}(s))-\tilde{u}_{n}(s)||\theta_{n}'(r_n(s))| |r_n'(s)| 
\\
&
= \frac{|\nabla_{x}E(r_{n}(s),\tilde{\theta}_{n}(s))-\tilde{u}_{n}(s)||\theta_{n}'(r_n(s))| }{1+|\nabla_{x}E(r_{n}(s),\tilde{\theta}_{n}(s))-\tilde{u}_{n}(s)||\theta_{n}'(r_n(s))| } \leq 1 \,,
\end{split}
\end{equation}
for $s\in[s_{n}^{1},s_{n}^{2}]$. Denote $A_\delta \coloneq \{s\in[s_{n}^{1},s_{n}^{2}]: \tilde \theta_n(s) \in K_\delta \}$. We notice further that if $s \in A_\delta$, then
\begin{equation}
0<e_\delta |\tilde \theta_{n}'(s)| \leq |\nabla_{x}E(r_{n}(s),\tilde{\theta}_{n}(s))-\tilde{u}_{n}(s)||\tilde \theta_{n}'(s)|  \leq 1 \,,
\end{equation}
and $\tilde \theta_n$ has finite speed on $A_\delta$.

The construction of the limiting function works from now exactly as in the proof of~\cite[Proposition~4.1]{MR3705699} taking into account that~$\tilde{u}_{n}(s)$ is uniformly close to~$u$. Let us explain informally how the construction work and we refer to the above mentioned reference for more details:
the sequence~$\tilde \theta_n$ is equibounded and, in view of the finite speed it is also equicontinuous. Up to a further linear time reparametrization, it admits by Ascoli-Arzel\`a Theorem a limit $\theta \in  C([0,1];\R^{d})$ such that $\theta(0)=x_{1}$, $\theta(1)=x_{2}$. Moreover, again in view of the finite speed, this limit cannot travel to an infinite number of critical points of minimal distance~$2 \delta$. Hence, it visits at most a finite number of them and $\theta \in \mathcal{A}^{t}_{x_{1},x_{2},u}$.
\end{proof}

\begin{remark}\label{r.3}
From Proposition~\ref{p.3} it follows that
\begin{displaymath}
\begin{split}
c_{t}(x_{1},x_{2};u)\leq &\inf\Big\{\int_{t^{1}_{n}}^{t_{n}^{2}}  |\nabla_{x}E(s,\theta_{n}(s))-u_{n}(s)||\theta'_{n}(s)|\,  \di s:\,\theta_{n},u_{n}\in AC([t_{n}^{1},t_{n}^{2}];\R^{d}) \text{ with $t_{n}^{i}\to t$,}\\
&\phantom{XXXXXXXXXXX}\text{ $\theta_{n}(t_{n}^{i})\to x_{i}$, $\dot u_{n}(s)=f(u_{n}(s))$, $u_{n}\rightrightarrows u$ uniformly in~$[t^{1}_{n}, t^{2}_{n}]$}\Big\}\,.
\end{split}
\end{displaymath}
\end{remark}

We now state an autonomous modification of Proposition~\ref{p.3}, in which the time parameter~$t$ is fixed. The result corresponds to~\cite[Proposition~4.5]{MR3705699}.

\begin{proposition}\label{p.3.5}
Let $u, x_{1}, x_{2}\in\R^{d}$, let $x_{n}^{i} \to x_{i}$ as $n\to\infty$ for $i=1,2$, and let $\theta_{n}\in\mathcal{A}^{t}_{x_{n}^{1},x_{n}^{2},u}$. Then, the following facts hold:
\begin{itemize}

\item[$(a)$] if $\liminf_{n}\int_{0}^{1}|\nabla_{x}E(t,\theta_{n}(s))-u||\theta_{n}'(s)|\, \di s=0$, then $x_{1}=x_{2}$;

\item[$(b)$] if $x_{1}\neq x_{2}$, then there exists $\theta\in\mathcal{A}_{x_{1},x_{2},u}^{t}$ such that
\begin{displaymath}
\liminf_{n}\int_{0}^{1}|\nabla_{x}E(t,\theta_{n}(s))-u||\theta'_{n}(s)|\, \di s\geq \int_{0}^{1}|\nabla_{x}E(t,\theta(s))-u||\theta'(s)|\, \di s\,.
\end{displaymath}
\end{itemize}
\end{proposition}

\begin{proof}
The proof can be carried out as in~\cite[Proposition~4.5]{MR3705699} working with the energy $\overline{E}(t, \cdot, u)$ with fixed parameters~$t$ and~$u$.
\end{proof}

\begin{remark}\label{r.4.6}
As a consequence of~(b) of Proposition~\ref{p.3.5}, whenever~$c_{t}(x_{1},x_{2};u)>0$ we have that the infimum in~\eqref{e.cost} is actually attained. Moreover, if~$c_{t}(x_{1},x_{2};u)=0$, then~$x_{1}=x_{2}$.
\end{remark}

We show now two results useful to describe the energetic behavior of the limits of sequences~$(x_{\varepsilon_{n}}, u_{\varepsilon_{n}})$.

\begin{proposition}\label{p.2.15}
Let us set
\begin{displaymath}
\nu_{n}\coloneq \Big(\frac{\varepsilon_{n}}{2}|\dot{x}_{\varepsilon_{n}}(\cdot)|^{2}+\frac{1}{2\varepsilon_{n}}|\nabla_{x}E(\cdot,x_{\varepsilon_{n}}(\cdot))-u_{\varepsilon_{n}}(\cdot)|^{2}\Big)\mathcal{L}^{1}_{|[0,T]}\,.
\end{displaymath}
Then, the following holds true:
\begin{itemize}
\item[$(a)$] there exists a positive Radon measure~$\nu\in\mathcal{M}^{+}_{b}(0,T)$ such that, up to a subsequence, $\nu_{n}\rightharpoonup\nu$ weakly$^{*}$ in~$\mathcal{M}^{+}_{b}(0,T)$;

\item[$(b)$]   there exists $\E\in BV(0,T)$ such that, up to a further subsequence, $E(\cdot,x_{\varepsilon_{n}}(\cdot))-u_{\varepsilon_{n}}(\cdot)\cdot x_{\varepsilon_{n}}(\cdot)$ converges to~$\E$ pointwise in~$[0,T]$;

\item[$(c)$] there exists $\mathcal{G}\in L^{\infty}(0,T)$ such that for every $s, t\in[0,T]$, $s\leq t$ we have
\begin{displaymath}
\E_{+}(t) +\nu([s,t]) = \E_{-}(s) +\int_{s}^{t}\mathcal{G}(\tau)\,\di \tau \,;
\end{displaymath}

\item[$(d)$] the set $J= \{t\in[0,T]: \nu(\{t\})>0\}$ is at most countable and coincides with the jump set of~$\E$.
\end{itemize}
\end{proposition}

\begin{proof}
In view of Proposition~\ref{p.1} and estimate~\eqref{finitemass},~$\nu_{n}$ is bounded in mass, uniformly with respect to~$n$, and up to a subsequence, $\nu_{n}\rightharpoonup\nu$ weakly$^{*}$ in~$\mathcal{M}^{+}_{b}(0,T)$. If we consider the function~$\F_{n}(t)$ given by
\begin{equation}
\begin{split}
\F_{n}(t) & \coloneq \overline{E}(t,x_{\varepsilon_{n}}(t),u_{\varepsilon_{n}}(t))-\int_{0}^{t}\partial_t E(\tau,x_{\varepsilon_{n}}(\tau))\, \di \tau + \int_{0}^{t}f(u_{\varepsilon_{n}}(\tau))\cdot x_{\varepsilon_{n}}(\tau)\, \di \tau
\\
&
= \overline{E}(0,x_{\varepsilon_{n}}(0),u_{\varepsilon_{n}}(0)) - \frac{1}{\varepsilon} \int_{0}^{t} |\nabla_x \overline{E}(\tau,x_{\varepsilon_{n}}(\tau),u_{\varepsilon_{n}}(\tau))|^2\, \di \tau \,,
\end{split}
\end{equation}
then~$\F_{n}$ is a sequence of bounded non-increasing functions. Hence, by Helly Theorem, it admits, up to subsequence, a pointwise limit~$\F\in BV(0,T)$. From (E2) and Proposition \ref{p.1} (b) the function~$\partial_{t}E(t,x_{\varepsilon_{n}}(t)) - f(u_{\varepsilon_{n}}(t))\cdot x_{\varepsilon_{n}}(t)$ is uniformly bounded with respect to~$t\in[0,T]$ and it admits a weak$^*$ limit $\mathcal{G}\in L^{\infty}(0,T)$. Therefore, the function $\E(t)\coloneq\F(t)+\int_{0}^{t}\mathcal{G}(\tau)\, d\tau$ is the  pointwise limit of~$\overline{E}(t,x_{\varepsilon_{n}}(t),u_{\varepsilon_{n}}(t))$. The rest of the proof goes precisely as in~\cite[Proposition~5.2]{MR3705699}. In particular, $J=\supp(\mathrm{D}\E)_{\rm{jump}}$ and $\nu(\{t\})=\E_{-}(t)-\E_{+}(t)$.
\end{proof}

\begin{lemma}\label{l.4}
Under the assumptions and notations of Proposition \ref{p.2.15}, let $t_{n}^{i}\to t$, $i=1,2$, let $\dot u_{n}(s)=f(u_{n}(s))$ and $u_{n}\rightrightarrows  u$ uniformly for $s\in[t_{n}^{1},t_{n}^{2}]$, and let $x_{\varepsilon_{n}}(t_{n}^{i})\to x_{i}\in\R^{d}$. Then, $\nu(\{t\})\geq c_{t}(x_{1},x_{2};u)$. In particular, $x_{1}=x_{2}$ if $t\notin J$, where  $J$ is as in Proposition \ref{p.2.15} (d).
\end{lemma}

\begin{proof}
For every $\tau>0$
\begin{displaymath}
\begin{split}
\nu([t-\tau,t+\tau])&\geq\limsup_{n}\nu_{\varepsilon_{n}}([t^{1}_{n},t^{2}_{n}])=\limsup_{n}\int_{t_{n}^{1}}^{t_{n}^{2}} \left ( \frac{\varepsilon_{n}}{2}|\dot{x}_{n}(s)|^{2}+\frac{1}{2\varepsilon_{n}}|\nabla_{x}E(s,x_{\varepsilon_{n}}(s))-u_{n}(s)|^{2}\right )\, \di s\\
&\geq \liminf_{n} \int_{t^{1}_{n}}^{t_{n}^{2}}|\nabla_{x}E(s,x_{\varepsilon_{n}}(s))-u_{\varepsilon_{n}}(s)||\dot{x}_{\varepsilon_{n}}(s)|\, \di s\geq c_{t}(x_{1},x_{2};u)\,.
\end{split}
\end{displaymath}
The thesis follows from Proposition~\ref{p.3}.
\end{proof}

We are now ready to prove Theorem~\ref{t.2}.

\begin{proof}[Proof of Theorem~\ref{t.2}]
Let us denote 
\begin{displaymath}
B\coloneq\{t\in[0,T]:\,|\nabla_{x}E(t,x_{\varepsilon_{n}}(t))-u_{\varepsilon_{n}}(t)|\to0\}\,.
\end{displaymath}
In view of~(c) of Proposition~\ref{p.1}, $\mathcal{L}^{1}([0,T]\setminus B)=0$.

Let $A\subset B\setminus J$ be a countable dense subset of~$[0,T]$, and let $I\coloneq A\cup J\cup\{0\}$. Since~$I$ is at most countable, we may fix a suitable subsequence such that $x_{\varepsilon_{n}}(t)\to x(t)$ for every $t\in I$, for some limit $x(t)\in\R^{d}$. Clearly, we already have that $u_{\varepsilon_{n}}\rightrightarrows u$ uniformly in~$[0,T]$, due to the continuity with respect to~initial data of the equation $\dot{u}=f(u)$.

For $t\in[0,T]\setminus I$ we define
\begin{equation}\label{e.deflim}
\tilde{x}(t)\coloneq \lim_{k} x(s_{k})\qquad\text{ for every sequence $(s_{k})_{k\in\mathbb{N}}, s_k \in  A$ converging to~$t$}\,.
\end{equation}
Let us show that $\tilde{x}(t)$ is well defined for every $t\in [0,T]\setminus I$. It is clear that, at least along a subsequence, the limit in~\eqref{e.deflim} exists for every sequence $s_{k}$ in~$A$ converging to~$t$. We have to prove that it is unique. Let~$t_{k}^{i}\to t$, $i=1,2$ in~$A$ and assume that $x(t_{k}^{i})\to x_{i}$. By construction, $x_{\varepsilon_{n}}(t_{k}^{i})\to x(t_{k}^{i})$ for every~$k$, so that we may find a suitable subsequence~$\varepsilon_{n_{k}}$ such that $x_{\varepsilon_{n_{k}}}(t_{k}^{i})\to x_{i}$ as $k\to \infty$. Applying Lemma~\ref{l.4} and recalling that $t\notin J$, we get that $x_{1}=x_{2}$, so that $\tilde{x}$ is well-defined in~$[0,T]\setminus I$.

We set
\begin{displaymath}
x(t)\coloneq\left\{\begin{array}{ll}
x(t)&\text{if $t\in I$}\,,\\
\tilde{x}(t)&\text{if $t\notin I$}\,.
\end{array}\right.
\end{displaymath}
We notice that, by construction and by continuity of~$u(\cdot)$,  $x(t) \in C(t, u(t))$ for every~$t\in(0,T]\setminus J$. It remains to show that $x_{\varepsilon_{n}}(t)\to x(t)$ for~$t\in[0,T]$. It suffices to show it for~$t\in[0,T]\setminus I$, since, by construction, the convergence is already satisfied in~$I$.

Assume that for some~$t\notin I$ we have $x_{\varepsilon_{n}}(t)\to\bar{x}\in\R^{d}$. Fix a sequence $t_{k}\in A$ converging to~$t$. In particular, $x(t_{k})\to x(t)$ by definition. Again, we may fix a subsequence~$\varepsilon_{n_{k}}$ such that $x_{\varepsilon_{n_{k}}}(t_{k})\to x(t)$. Applying Lemma~\ref{l.4} and recalling that $t\notin J$, we get that $\bar{x}=x(t)$. Hence, $x_{\varepsilon_{n}}(t)\to x(t)$ for every $t\in[0,T]$. Being~$x_{\varepsilon_{n}}$ uniformly bounded in~$\R^{d}$, $x_{\varepsilon_{n}}\to x$ in $\Gamma_{T}^{p}$ for every $p\in[1,+\infty)$. Moreover, this implies that the function~$\mathcal{G}$ determined in~(c) of Proposition~\ref{p.4} actually coincides with $\partial_{t} E(t, x(t)) - f(u(t))\cdot x(t)$ and that $\E(t) = \overline{E} (t, x(t), u(t))$ for every $t\in[0,T]$.

Let us now show that~$x$ admits left and right limits for every~$t\in [0,T]$. Let us focus on the existence of~$x_{+}(t)$. Let $t_{k}^{1}, t_{k}^{2} \searrow t$, and let us assume, without loss of generality, that $t^{1}_{k}< t^{2}_{k}$. Up to a subsequence, we may further assume that $x(t^{i}_{k}) \to x_{i} \in\R^{d}$, $i=1,2$. Being $\E\in BV(0,T)$, it is clear that $\E(t^{i}_{k}) \to \E_{+}(t)$. Furthermore, by the convergence of $x_{\varepsilon_{n}}(t^{i}_{k})$ to $x(t^{i}_{k})$ and of $u_{\varepsilon_{n}}(t^{i}_{k})$ to~$u(t^{i}_{k})$ as $n\to\infty$ for every $k\in\mathbb{N}$ and for $i=1,2$, we have that we can construct a suitable subsequence~$\varepsilon_{n_{k}}$ such that
\begin{displaymath}
\lim_{k}\, \overline{E}(t^{1}_{k}, x_{\varepsilon_{n_{k}}}(t^{1}_{k}), u_{\varepsilon_{n_{k}}}(t^{1}_{k})) = \lim_{k}\, \overline{E}(t^{2}_{k}, x_{\varepsilon_{n_{k}}}(t^{2}_{k}), u_{\varepsilon_{n_{k}}}(t^{2}_{k})) = \E_{+}(t) \,. 
\end{displaymath}
Hence, rewriting the energy balance~\eqref{e.2} for every~$k$ and passing to the limit as~$k\to\infty$ we get that
\begin{displaymath}
0 = \lim_{k}\,  \overline{E}(t^{2}_{k}, x_{\varepsilon_{n_{k}}}(t^{2}_{k}), u_{\varepsilon_{n_{k}}}(t^{2}_{k}))  -  \overline{E}(t^{1}_{k}, x_{\varepsilon_{n_{k}}}(t^{1}_{k}), u_{\varepsilon_{n_{k}}}(t^{1}_{k})) = \limsup_{k} \, \nu_{\varepsilon_{n_{k}}} ([t^{1}_{k}, t^{2}_{k}]) \geq c_{t}(x_{1},x_{2};u(t)) \,.
\end{displaymath}
This implies that $x_{1}=x_{2}$, and the right limit $x_{+}(t)$ is well defined for every $t\in[0,T]$. In a similar way we can show that~$x_{-}(t)$ is well defined. Moreover,~$x_{\pm}(t) \in C(t, u(t))$ for every~$t\in[0,T]$.

It is now straightforward to see from~(c) of Proposition~\ref{p.2.15} that the energy balance~\eqref{e.enbalance} is satisfied. Finally, we show that $c_{t}(x_{-}(t), x_{+}(t); u(t)) = \nu(\{t\})$ for every $t\in J$. In Lemma~\ref{l.4} we have already proved the inequality~$c_{t}(x_{-}(t), x_{+}(t); u(t)) \leq \nu(\{t\})$. For the opposite inequality, in view of Proposition~\ref{p.3.5} and Remark~\ref{r.4.6} we assume that $c_{t}(x_{-}(t), x_{+}(t); u(t)) > 0$. Let us denote with $\theta\in \mathcal{A}^{t}_{x_{-}(t), x_{+}(t), u(t)}$ the optimal transition between~$x_{-}(t)$ and~$x_{+}(t)$. By chain rule we have that
\begin{displaymath}
\begin{split}
c_{t}(x_{-}(t), x_{+}(t); u(t)) & = \int_{0}^{1} | \nabla_{x} E(t, \theta(s)) - u(t)| |\theta'(s)|\,\di s \geq - \int_{0}^{1} (\nabla_{x} E(t, \theta(s)) - u(t)) \cdot \theta'(s)\,\di s \\
& \vphantom{\int} = \overline{E}(t, x_{-}(t), u(t)) - \overline{E}(t, x_{+}(t), u(t))\,.
\end{split}
\end{displaymath}
This concludes the proof of the theorem.
\end{proof}


\section{Mean-field limit of evolutions of critical points}\label{s.CE}

\subsection{Mean-field limit for $\varepsilon>0$}

In this section we deduce the mean-field limit of the ODE system~\eqref{e.gradflow}. Although this is by now a standard procedure, see, e.g., \cite[Chapter 8]{MR2401600}, we show here the details of the passage to the mean-field limit in order to stress the dependence on the auxiliary control variable~$u$ introduced in~\eqref{e.gradflow}.

Let $\eta_{0}\in\sP_{c}(\R^{d}\times\R^{d})$ represent the distribution of initial conditions~$(x_{0},u_{0})$, where~$\sP_{c}$ denotes the space of probability measures with compact support. Given $(x_{0}^{1},u_{0}^{1}),\ldots,(x_{0}^{N},u_{0}^{N})\in\R^{d}\times\R^{d}$ distributed as~$\eta_{0}$, we denote with  $(x_{\varepsilon}^{i},u_{\varepsilon}^{i})$ the solution of~\eqref{e.gradflow} with initial condition~$(x_{0}^{i},u_{0}^{i})$, $i\in\{1,\ldots,N\}$. We define the empirical measure~$\eta_{\varepsilon, t}^{N} \in\sP(\R^{d}\times\R^{d})$ by
\begin{displaymath}
\eta_{\varepsilon, t}^{N} \coloneq\frac{1}{N}\sum_{i=1}^{N}\delta_{(x_{\varepsilon}^{i}(t),u_{\varepsilon}^{i}(t))}\,.
\end{displaymath}

\begin{remark}\label{r.3.1}
In agreement with the initial value problem~\eqref{e.gf2}, we could, for instance, imagine that the measure $\eta_{0}\in \sP_{c}(\R^{d}\times\R^{d})$ takes the form $\eta_{0} \coloneq (\mathrm{id} , \nabla_{x}E(0,\cdot))_{\#} \mu_{0}$ for some $\mu_{0}\in  \sP_{c}(\R^{d})$. This represents exactly the case where the initial control~$u_{0}$ is $\nabla_{x}E(0,x_{0})$.
\end{remark}

For every $\varphi\in C^{1}_{b}(\R^{d}\times\R^{d})$ we have that
\begin{equation}\label{e.6}
\begin{split}
\frac{\di}{\di t}\left\langle \varphi,\eta_{\varepsilon}^{N}(t)\right\rangle&=\frac{1}{N}\sum_{i=1}^{N}\frac{\di}{\di t}\varphi(x_{\varepsilon}^{i}(t),u_{\varepsilon}^{i}(t))\\
&=\frac{1}{N}\sum_{i=1}^{N}\nabla_{x}\varphi(x_{\varepsilon}^{i}(t),u_{\varepsilon}^{i}(t)){\,\cdot\,}\dot{x}^{i}_{\varepsilon}(t)+\nabla_{u}\varphi(x_{\varepsilon}^{i}(t),u_{\varepsilon}^{i}(t)){\,\cdot\,}\dot{u}^{i}_{\varepsilon}(t)\\
&=\frac{1}{N}\sum_{i=1}^{N}\nabla_{x}\varphi(x_{\varepsilon}^{i}(t),u_{\varepsilon}^{i}(t)){\,\cdot\,}\Big(\frac{-\nabla_{x}E(t,x_{\varepsilon}^{i}(t))+u_{\varepsilon}^{i}(t)}{\varepsilon}\Big)+\nabla_{u}\varphi(x_{\varepsilon}^{i}(t),u_{\varepsilon}^{i}(t)){\,\cdot\,}f(u_{\varepsilon}^{i}(t))\\
&=\int_{\R^{d}\times\R^{d}}\nabla_{x}\varphi(x,u){\,\cdot\,}\Big(\frac{-\nabla_{x}E(t,x)+u}{\varepsilon}\Big)+\nabla_{u}\varphi(x,u){\,\cdot\,}f(u)\, \di \eta_{\varepsilon, t}^{N} (x,u)\,,
\end{split}
\end{equation}
which corresponds to the weak form of
\begin{equation}\label{e.3}
\partial_{t}\eta_{\varepsilon}^{N}=-\dive_{(x,u)}\Bigg[\left(\begin{array}{cc}
\displaystyle \frac{-\nabla_{x}E(t,x)+u}{\varepsilon}\\[2mm]
\displaystyle f(u)
\end{array}\right)\eta_{\varepsilon}^{N}\Bigg]\,.
\end{equation}

\begin{remark}\label{r.3.2}
Continuing the discussion of Remark~\ref{r.traj}, we want here to further justify the choice of working in the space~$\sP(\R^{d}\times\R^{d})$, where also the control parameter~$u$ has its own distribution, and not in~$\sP(\R^{d})$, where only the space variable~$x$ would be described by a probability measure. Let us indeed consider the simpler setting of~\eqref{e.gf2} with the initial data~$x_0^{1},\ldots,x_{0}^{N}$ being distributed according to a measure~$\mu_{0}\in\sP(\R^{d})$. Given $x_{1}, \dots, x_{N}\colon [0,T]\to\R^{d}$ the corresponding trajectories, we could define the empirical measure~$\mu_{\varepsilon, t}^{N}:=\tfrac{1}{N}\sum_{i=1}^{N} \delta_{x_{i}(t)}$ and, repeating the computation in~\eqref{e.6}, for every $\varphi\in C^{1}_{b}(\R^{d})$ we would get
\begin{equation}\label{e.6.5}
\begin{split}
\frac{\di}{\di t}\left\langle \varphi, \mu_{\varepsilon, t}^{N} \right\rangle&= \frac{1}{N}\sum_{i=1}^{N}\nabla_{x}\varphi(x_{\varepsilon}^{i}(t)){\,\cdot\,}\dot{x}^{i}_{\varepsilon}(t)\\
&=\frac{1}{N}\sum_{i=1}^{N}\nabla_{x}\varphi(x_{\varepsilon}^{i}(t)){\,\cdot\,}\Big(\frac{-\nabla_{x}E(t,x_{\varepsilon}^{i}(t))+\nabla_{x} E(0,x^{i}_{0})}{\varepsilon}\Big) \\
&= - \int_{\R^{d}}\nabla_{x}\varphi(x){\,\cdot\,}\frac{\nabla_{x}E(t,x)}{\varepsilon}\,\di \mu^{N}_{\varepsilon, t}(x) + \frac{1}{N\varepsilon}\sum_{i=1}^{N}\nabla_{x}\varphi(x_{\varepsilon}^{i}(t)){\,\cdot\,}\frac{\nabla_{x} E(0,x^{i}_{0})}{\varepsilon} \,.
\end{split}
\end{equation}
In order to obtain an integral formula as in~\eqref{e.6}, one could try to define a flow $X_{\varepsilon, t}(x_{0})$ that associates to each $x_{0}$ the value at time~$t$ of the solution of~\eqref{e.gf2} starting at~$x_{0}$ at time $t=0$, and plug its inverse into the last term in~\eqref{e.6.5}. However, as noticed already in Remark~\ref{r.traj}, this is not possible, since for distinct initial data~$x_{0}, \tilde x_{0}$ it could happen that the two trajectories cross each other at time~$t$. Hence, we can not deduce a continuity equation for the sole distribution of~$x$'s.
\end{remark}

We want to pass to the limit in~\eqref{e.3} as~$N\to \infty$. We notice that, in view of $(b)$ of Proposition~\ref{p.1}, the support of~$\eta_{\varepsilon,t}^{N}$ is bounded in~$\R^{d}\times\R^{d}$ uniformly with respect to~$t \in [0,T]$, $N\in\mathbb{N}$, and~$\varepsilon>0$. In order to identify a continuous in time limit $\eta_{\varepsilon}\in C([0,T];\sP(\R^{d}\times\R^{d}))$ of $\eta_{\varepsilon}^{N}$, in the following lemma we estimate the Wasserstein $1$-distance between $\eta_{\varepsilon, t_{i}}^{N}$ and~$\eta_{\varepsilon, t_{2}}^{N}$ for $t_{1},t_{2}\in[0,T]$ (equicontinuity).

\begin{lemma}\label{l.2}
Let $t_{1},t_{2}\in[0,T]$. Then
\begin{equation}\label{e.4}
W_{1}(\eta_{\varepsilon}^{N}(t_{1}),\eta_{\varepsilon}^{N}(t_{2}))\leq C_{\varepsilon}|t_{2}-t_{1}|
\end{equation}
for some positive constant~$C_{\varepsilon}$ depending on $\varepsilon>0$ but not on $t_{1},t_{2}$, and~$N$.
\end{lemma}

\begin{proof}
Since $\eta_{\varepsilon}^{N}$ is an empirical measure, we simply have that
\begin{displaymath}
\begin{split}
W_{1}(\eta_{\varepsilon}^{N}(t_{1}),\eta_{\varepsilon}^{N}(t_{2}))&\leq\frac{1}{N}\sum_{i=1}^{N}|x_{\varepsilon}^{i}(t_{1})-x_{\varepsilon}^{i}(t_{2})|+|u_{\varepsilon}^{i}(t_{1})-u_{\varepsilon}^{i}(t_{2})|\\
&\leq\frac{1}{N}\sum_{i=1}^{N}\int_{t_{1}}^{t_{2}}|\dot{x}_{\varepsilon}^{i}(\tau)|\, \di \tau+\int_{t_{1}}^{t_{2}}|\dot{u}_{\varepsilon}^{i}(\tau)|\, \di \tau \leq C_{\varepsilon}|t_{1}-t_{2}|\,,
\end{split}
\end{displaymath}
where we have used the system~\eqref{e.gradflow}, the boundedness of~$x_{\varepsilon}^{i},u_{\varepsilon}^{i}$, and the hypotheses (E1)-(E5) on $E$ and on~$f$. 
\end{proof}

\begin{proposition}\label{p.2}
There exists $\eta_{\varepsilon}\in C([0,T];\sP(\R^{d}\times\R^{d}))$ such that $W_{1}(\eta_{\varepsilon}^{N}(t),\eta_{\varepsilon}(t))\to0$ as $N\to \infty$, uniformly with respect to~$t\in[0,T]$. Moreover, $\eta_{\varepsilon}$ satisfies
\begin{equation}\label{e.5}
\partial_{t}\eta_{\varepsilon}=-\dive_{(x,u)}\Bigg[\left(\begin{array}{cc}
\displaystyle \frac{-\nabla_{x}E(t,x)+u}{\varepsilon}\\[2mm]
\displaystyle f(u)
\end{array}\right)\eta_{\varepsilon}\Bigg]
\end{equation}
in the sense of distributions.
\end{proposition}

\begin{remark}\label{r.1}
According to~\cite[Section~8.1]{MR2401600}, we can pick the measure~$\eta_{\varepsilon,t} =(Y_{\varepsilon,t})_{\#}\eta_{0}$, where~$Y_{\varepsilon,t}$ is the flow associated to system~\eqref{e.gradflow} defined in Remark~\ref{r.2}.
\end{remark}

\begin{proof}
In view of Lemma~\ref{l.2}, the sequence~$\eta_{\varepsilon}^{N}$ is equicontinuous in~$\sP(\R^{d}\times\R^{d})$. As already mentioned, the support of~$\eta_{\varepsilon, t}^{N}$ is equibounded in~$\R^{d}\times\R^{d}$ for every $t\in[0,T]$, so that~$\eta_{\varepsilon, t}^{N}$ is compact for every~$t\in[0,T]$. Therefore, by Ascoli-Arzel\`a Theorem there exists $\eta_{\varepsilon}\in C([0,T];\sP(\R^{d}\times\R^{d}))$ such that (up to a subsequence)~$\eta_{\varepsilon}^{N}$ converges to~$\eta_{\varepsilon}$ in~$W_{1}$-distance, uniformly with respect to~$t\in[0,T]$.

For every $\varphi\in C^{1}_{b}(\R^{d}\times\R^{d})$ we can pass to the limit in~\eqref{e.6} as $N \to \infty$, obtaining
\begin{displaymath}
\begin{split}
\int_{\R^{d}\times\R^{d}}&\varphi(x,u)\, \di \eta_{\varepsilon,t} (x,u)-\int_{\R^{d}\times\R^{d}}\varphi(x,u)\, \di \eta_{\varepsilon,0} (x,u)\\
&=\int_{0}^{t}\int_{\R^{d}\times\R^{d}}\nabla_{x}\varphi(x,u){\,\cdot\,}\Big(\frac{-\nabla_{x}E(\tau,x)+u}{\varepsilon}\Big)+\nabla_{u}\varphi(x,u){\,\cdot\,}f(u)\, \di \eta_{\varepsilon, \tau} (x,u)\, \di \tau\,.
\end{split}
\end{displaymath}
Therefore, $\eta_{\varepsilon}$ solves~\eqref{e.5} in the sense of distributions. By~\cite[Section~8.1]{MR2401600}, the solution of~\eqref{e.5} is unique. Thus, the whole sequence~$\eta^{N}_{\varepsilon}$ converges to~$\eta_{\varepsilon}$ uniformly with respect to~$W_{1}$.
\end{proof}


\subsection{Mean-field limit for $\varepsilon\to0$}\label{s.L}

In order to derive a mean-field limit of evolutions of critical points, we wish to take advantage of the superposition principle introduced in \cite[Theorem 8.2.1]{MR2401600}. Accordingly, let us define the probability measure $\Pi_{\varepsilon}\in\sP(\R^{d}\times\R^{d}\times\Gamma_{T}^{p}\times\Lambda)$ by $\Pi_{\varepsilon}\coloneq (\mathrm{id} \times \mathrm{id} \times Y_{\varepsilon})_{\#}\eta_{\varepsilon}$, where we consider the flow~$Y_{\varepsilon}$ also as a function of time. By definition of~$\Pi_{\varepsilon}$, for every $\phi \in C_{b}(\R^{d}\times\R^{d})$ we have
\begin{equation}\label{e.Pi}
\int_{\R^{d}\times\R^{d}}\phi (x_{0},u_{0})\, \di \eta_{\varepsilon,t} (x_{0},u_{0})=\int_{\R^{d}\times\R^{d}\times\Gamma_{T}^{p}\times\Lambda}\phi (x(t),u(t))\, \di \Pi_{\varepsilon}(x_{0},u_{0},x(\cdot),u(\cdot))\,.
\end{equation}
For every bounded  continuous function $\varphi\colon [0,T]\times\R^{d}\times\R^{d}\to\R$, we may consider~\eqref{e.Pi} also integrated in time. In particular, by Fubini,
\begin{equation}\label{e.Pi2}
\int_{0}^{T}\int_{\R^{d}\times\R^{d}}\varphi(t,x_{0},u_{0})\, \di \eta_{\varepsilon,t}(x_{0},u_{0})\, \di t=\int_{\R^{d}\times\R^{d}\times\Gamma_{T}^{p}\times\Lambda}\int_{0}^{T}\varphi(t,x(t),u(t))\, \di t \, \di \Pi_{\varepsilon}(x_{0},u_{0},x(\cdot),u(\cdot))\,.
\end{equation}
We notice in particular that the function from $\Gamma_{T}^{p}\times\Lambda$ to $\R$
\begin{displaymath}
(x(\cdot),u(\cdot))\mapsto\int_{0}^{T}\varphi(t,x(t),u(t))\, \di t
\end{displaymath}
is continuous.

\begin{theorem}\label{t.1}
There exists $\Pi\in\sP(\R^{d}\times\R^{d}\times\Gamma_{T}^{p}\times\Lambda)$ such that, up to a subsequence, $\Pi_{\varepsilon}$ narrowly converges to~$\Pi$. Moreover, every accumulation point~$\Pi$ of~$\Pi_{\varepsilon}$ satisfies
\begin{equation}\label{e.7}
\int_{\R^{d}\times\R^{d}\times\Gamma_{T}^{p}\times\Lambda}\int_{0}^{T}|\nabla_{x}E(t,x(t))-u(t)|^{2}\, \di t\, \di \Pi(x_{0},u_{0},x(\cdot),u(\cdot))=0\,.
\end{equation}
\end{theorem}

\begin{proof}
By definition, for every $\varepsilon\in(0,1]$ the support of~$\Pi_{\varepsilon}$ is contained in the subset of $\R^{d}\times\R^{d}\times\Gamma^{p}_{T}\times\Lambda$
\begin{displaymath}
\mathcal K \coloneq \mathrm{cl}\Bigg(\bigcup_{\varepsilon\in(0,1]}\Big\{(x_{0},u_{0},x,u)\in\R^{d}\times\R^{d}\times\Gamma_{T}^{p}\times\Lambda: (x,u) \text{ solves~\eqref{e.gradflow} with initial datum $(x_{0},u_{0})\in \supp(\eta_{0})$}\Big\}\Bigg)\,,
\end{displaymath}
where we denote with $\mathrm{cl}(A)$ the closure of~$A$. In view of Theorem~\ref{t.2} the above set is compact in~$\Gamma^{p}_{T}\times\Lambda$ with respect to the strong topology of~$L^{p}$ and the uniform convergence in~$\Lambda$. Therefore, by Prokhorov theorem there exists a measure~$\Pi\in\sP(\R^{d}\times\R^{d}\times\Gamma_{T}^{p}\times\Lambda)$ such that, up to a subsequence, $\Pi_{\varepsilon}$ narrowly converges to~$\Pi$. Moreover, for every $\varphi\in C_{b}([0,T]\times\R^{d}\times\R^{d})$ we have that
\begin{equation}\label{e.1.27}
\lim_{\varepsilon\to0}\int_{\R^{d}\times\R^{d}\times\Gamma_{T}^{p}\times\Lambda}\int_{0}^{T}\varphi(t,x(t),u(t))\, \di t\, \di \Pi_{\varepsilon}=\int_{\R^{d}\times\R^{d}\times\Gamma_{T}^{p}\times\Lambda}\int_{0}^{T}\varphi(t,x(t),u(t))\, \di t\, \di \Pi\,.
\end{equation}
Taking $\varphi(t,x,u)$ any bounded continuous extension of $|\nabla_{x}E(t,x)-u|^{2}$ outside $K^\circ$ (see Remark \ref{r.unif}) as a test function and recalling  Proposition \ref{p.1} (c), we get
\begin{equation}\label{e.1.28}
\begin{split}
\lim_{\varepsilon\to0}\int_{\R^{d}\times\R^{d}\times\Gamma_{T}^{p}\times\Lambda}&\int_{0}^{T}|\nabla_{x}E(t,x(t))-u(t)|^{2}\, \di t\, \di \Pi_{\varepsilon}\\
&=\lim_{\varepsilon\to0}\int_{\R^{d}\times\R^{d}}\int_{0}^{T}|\nabla_{x}E(t,Y^{1}_{\varepsilon,t}(x_{0}))-Y^{2}_{t}(u_{0})|^{2}\, \di t\, \di \eta_{0}(x_{0},u_{0})=0\,,
\end{split}
\end{equation}
where $Y_{\varepsilon}^{1}$ and $Y^{2}$ denote the two components of the flow $Y_{\varepsilon}$ associated to system \eqref{e.gradflow}. Notice that $Y^{2}$ does in fact not depend on $\varepsilon$. Combining \eqref{e.1.27} and \eqref{e.1.28} we deduce~\eqref{e.7}.
\end{proof}

\begin{proposition}\label{p.3.1}
There exists a Borel family $\{\eta_{t}: t\in[0,T]\}$ in $\sP(\R^{d}\times\R^{d})$ such that, up to subsequences, $\eta_{\varepsilon,t}\otimes\mathcal{L}^{1}_{|[0,T]}$ narrowly converges to $\eta_{t}\otimes\mathcal{L}^{1}_{|[0,T]}$. Moreover, for every $\psi\in C_{c}(\R^{d})$, every~$\varphi\in C^{1}_{b}(\R^{d})$, every $\theta\in C_{c}([0,T])$, and every~$\phi\in C^{1}_{b}([0,T])$,
\begin{eqnarray}
&&\displaystyle \int_{0}^{T}\int_{\R^{d}\times\R^{d}}\theta(t)(\nabla_{x}E(t,x)-u)\cdot\psi(x)\, \di \eta_{t}(x,u)\, \di t=0\,,\label{e.8}\\[1mm]
&& \displaystyle \int_{0}^{T}\int_{\R^{d}\times\R^{d}} |\nabla_{x}E(t,x)-u|^{2}  \di \eta_{t}(x,u)\, \di t=0 \,, \label{e.9.5}\\[1mm]
&&\displaystyle \int_{0}^{T}\int_{\R^{d}\times\R^{d}}\partial_{t}\phi(t)\varphi(u)\, \di \eta_{t}(x,u)\, \di t+\int_{0}^{T}\int_{\R^{d}\times\R^{d}}\phi(t)\nabla_{u}\varphi(u){\,\cdot\,}f(u)\, \di \eta_{t}(x,u)\, \di t=0\,.\label{e.9}
\end{eqnarray}
\end{proposition}

\begin{remark}
Equality~\eqref{e.8} implies that for a.e.~$t\in[0,T]$, for $\eta(t)$-a.e.~$(x,u)$ in~$\R^{d}\times\R^{d}$ we have $\nabla_{x}E(t,x)-u=0$.
\end{remark}

\begin{proof}
The measure $\eta_{\varepsilon,t}\otimes\mathcal{L}^{1}_{|[0,T]}$ is (up to a rescaling) a probability measure  with compact support in $\R^{d}\times\R^{d}\times[0,T]$ independent of~$\varepsilon>0$. In view of the structure of $\eta_{\varepsilon,t}\otimes\mathcal{L}^{1}_{|[0,T]}$, we have that there exists a Borel family $\{\eta_{t}:t\in [0,T]\}$ of measures in~$\sP(\R^{d}\times\R^{d})$ such that, up to subsequences,~$\eta_{\varepsilon,t} \otimes\mathcal{L}^{1}_{|[0,T]}$ narrowly converges to~$\eta_{t}\otimes\mathcal{L}_{|[0,T]}$. From equality~\eqref{e.Pi} and from Theorem~\ref{t.1} we deduce formula~\eqref{e.8}. In a similar way we get~\eqref{e.9.5}.

As for~\eqref{e.9}, given~$\phi$ and~$\varphi$ as in the statement of the theorem, we simply test the continuity equation~\eqref{e.5} against $\phi(t)\varphi(u)$, obtaining
\begin{displaymath}
\int_{0}^{T}\int_{\R^{d}\times\R^{d}}\big(\partial_{t}\phi(t)\varphi(u)+\phi(t)\nabla_{u}\varphi(u){\,\cdot\,}f(u)\big)\, \di \eta_{\varepsilon,t} (x,u)\, \di t=0\,.
\end{displaymath}
Passing to the limit as~$\varepsilon\to0$ in the previous equality we obtain~\eqref{e.9}.
\end{proof}

\subsection{Mean-field energy balance}\label{s.EB}

While Theorem~\ref{t.1} explains that evolutions of critical points are distributed essentially as $\eta_t$, we would like to clarify in this section how the energy~$E(t,x(t))$ is distributed at time $t \in [0,T]$ given that $t \mapsto x(t)$ is an evolution of critical points as derived in Theorem~\ref{t.2} with random initial data $(x_0,u_0)$ distributed as~$\eta_0$. 
Ideally we would expect that~$\overline{E}(t,x(t),u(t))$ is distributed as $\overline{E}(t,\cdot,\cdot)_\# \eta_t$. Unfortunately, the lack of smoothness of the trajectories $t \mapsto x(t)$ does not allow to obtain such a mean-field description of the energy, but we are able to derive below a slightly weaker form of it, which leverages again the superposition principle and Lebesgue point description of left and right limits.

Let us start introducing some useful notation. For any $(x(\cdot),u(\cdot)) \in \Gamma^p_T\times \Lambda$ and $0 \leq s \leq t \leq T$, it is convenient to denote
\begin{equation}\label{Vst}
\begin{split}
\mathcal V^{s,t}(x(\cdot),u(\cdot))\coloneq & \ - \limsup_{r \to 0^+} \overline{E} \left ( t,\frac{1}{r} \int_{t}^{t+r} x(\tau) \, \di \tau, u(t) \right) + \liminf_{r \to 0^-} \overline{E}\left (s,\frac{1}{r} \int_{s-r}^{s} x(\tau) \, \di \tau, u(s) \right ) 
 \\
& \phantom{XXXXX} + \int_s^t [\partial_t E(\tau,x(\tau)) - f(u(\tau))\cdot x(\tau) ] \, \di \tau \,. 
\end{split}
\end{equation}
Notice that if $t \mapsto x(t)$ admits left and right limits $x_{\pm}(t)$ at every $t$ then
\begin{eqnarray*}
\mathcal V^{s,t}(x(\cdot),u(\cdot))= -\overline{E} \left ( t,x_+(t), u(t) \right) + \overline{E}\left (s,x_-(s), u(s) \right ) + \int_s^t [\partial_t E(\tau,x(\tau)) - f(u(\tau))\cdot x(\tau) ] \, \di \tau \,.
\end{eqnarray*}
In the following it will be useful to consider the projection 
$$
\pi_{\Gamma_T^p\times\Lambda} \colon \mathbb R^d \times \mathbb R^d  \times \Gamma^{p}_{T}\times\Lambda\to  \Gamma^{p}_{T}\times\Lambda, \quad (x_0,u_0,x(\cdot), u(\cdot)) \mapsto(x(\cdot), u(\cdot))\,,
$$ 
and the marginals $\Pi_\varepsilon^*\coloneq {\pi_{\Gamma_T^p\times\Lambda}}_\# \Pi_\varepsilon$ and $\Pi^* \coloneq {\pi_{\Gamma_T^p\times\Lambda}}_\# \Pi$  in $\mathcal P( \Gamma^{p}_{T}\times\Lambda)$.
\begin{lemma}\label{l.meas} Let us fix $0 \leq s \leq t \leq T$. Then the map 
$$
(x(\cdot),u(\cdot)) \mapsto  \mathcal V^{s,t}(x(\cdot),u(\cdot)),
$$
 is measurable from $\Gamma^p_T\times \Lambda$ endowed with the measure~$\Pi^*$ to~$\R$. Moreover, it is integrable with respect to~$\Pi^*$, and  it coincides pointwise on the support of~$\Pi^*$ with the measure~$\nu([s,t])= \nu_{x(\cdot),u(\cdot))}([s,t])$ found in Proposition~\ref{p.2.15} (a) or Theorem~\ref{t.2} (d), for any pair~$(x(\cdot), u(\cdot))$. 
\end{lemma}
\begin{proof}
The map $(x, u)\mapsto \mathcal{V}^{s,t}(x(\cdot), u(\cdot))$  is defined in \eqref{Vst} through $\limsup$ and $\liminf$ operations of continuous functions on $\Gamma^{p}_{T}\times\Lambda$. Since $\Pi$ is a probability measure (so a Borel measure), $\mathcal{V}^{s,t}$ is $\Pi$-measurable. 
From Theorem \ref{t.1}, the support of $\Pi^*$ is contained in the Kuratowski $\liminf$ of the support of~$\Pi_{\varepsilon}^*$, so that for any $(x(\cdot), u(\cdot)) \in \supp( \Pi^*)$  there exists a sequence $((x_{\varepsilon_n}(\cdot), u_{\varepsilon_n}(\cdot)))_n$ of solutions of~\eqref{e.gradflow} converging to  $(x(\cdot), u(\cdot))$ in $\Gamma_{T}^{p}\times\Lambda$ for any vanishing sequence~$(\varepsilon_n)_n$. In particular,~$x_{\varepsilon_n}(t)$ converges up to subsequences to~$x(t)$ for almost every~$t$ in $[0,T]$.
Hence,~$\Pi^*$ is supported on essentially bounded curves, so that $\mathcal{V}^{s, t}$ is $\Pi^*$-integrable. Moreover, on the support of~$\Pi^*$ we have that~$\mathcal{V}^{s,t}(x(\cdot), u(\cdot))$ actually coincides with the measure~$\nu([s,t])= \nu_{x(\cdot),u(\cdot))}([s,t])$ found in Proposition~\ref{p.2.15} (a) or Theorem~\ref{t.2} (d), for the pair~$(x(\cdot), u(\cdot))$. 
\end{proof}

The following result explains how the energy $ \overline{E}(t,x(t),u(t))$ is distributed at time $t \in [0,T]$ for any $(x(\cdot), u(\cdot)) \in \supp  \, \Pi^*$.

\begin{theorem}\label{MF_EB_thm}
For $\Pi\in\sP(\R^{d}\times\R^{d}\times\Gamma_{T}^{p}\times\Lambda)$ as in Theorem \ref{t.1} and $\Pi^*={\pi_{\Gamma_T^p\times\Lambda}}_\# \Pi$  in $\mathcal P( \Gamma^{p}_{T}\times\Lambda)$, the following mean-field energy balance principle holds: for every $0 \leq s \leq t \leq T$ 
\begin{equation}\label{meas_eq}
\begin{split}
\int_{\Gamma_{T}^{p}\times\Lambda} & \overline{E} \left ( t,\tilde x_+(t), u(t) \right) \, \di \Pi^*(x(\cdot),u(\cdot)) + \mathcal V([s,t])   \\
&
= \int_{\Gamma_{T}^{p}\times\Lambda} \Big \{ \overline{E}\left (s, \tilde x_-(s), u(s) \right ) + \int_s^t [\partial_t E(\tau,x(\tau)) - f(u(\tau))\cdot x(\tau) ] \, \di \tau]  \Big \} \, \di \Pi^*(x(\cdot),u(\cdot)) \,,
\end{split}
\end{equation}
where $\tilde x(\cdot)$ is a representative of $x(\cdot)$ with left and right limits $\tilde x_\pm(t)$ at any $t \in [0,T]$,   
$$
\mathcal V \coloneq \left [ \int_{\R^d \times\R^d \times\Gamma_{T}^{p}\times\Lambda}  \nu_{(\tilde x(\cdot), u(\cdot))} \, \di \Pi(x_0,u_0,x(\cdot),u(\cdot)) \right ] \in \mathcal M_b^+(0,T),
$$
and $\nu_{(\tilde x(\cdot), u(\cdot))}$ is the positive measure of Proposition \ref{p.2.15} (a),
\end{theorem}
\begin{proof}
As mentioned in the proof of Lemma \ref{l.meas}, from Theorem \ref{t.1}, for any $(x(\cdot), u(\cdot)) \in \supp( \Pi^*)$  there exists a sequence $((x_{\varepsilon_n}(\cdot), u_{\varepsilon_n}(\cdot)))_n$ of solutions of \eqref{e.gradflow} converging to  $(x(\cdot), u(\cdot))$ in $\Gamma_{T}^{p}\times\Lambda$ for any vanishing sequence  $(\varepsilon_n)_n$. In particular, $x_{\varepsilon_n}(t)$ converges up to subsequences to $x(t)$ for almost every~$t$ in~$[0,T]$. However, from Theorem~\ref{t.2} (a)-(c) there exists yet one more subsequence (not relabeled), which converges pointwise to a trajectory $\tilde x \in \Gamma_{T}^{p}$, possessing left and right limits $\tilde x_\pm(t)$ at any $t \in [0,T]$. Hence, $\tilde x$ and $x$ coincide almost everywhere. In particular, the following integrals must coincide
$$
\int_{t}^{t+r} x(\tau) \, \di \tau= \int_{t}^{t+r} \tilde x(\tau) \, \di \tau, \quad \int_{s-r}^{s} x(\tau) \, \di \tau=\int_{s-r}^{s} \tilde x(\tau) \, \di \tau \,.
$$
Therefore, by Theorem \ref{t.2} (d), we have
\begin{displaymath}
\begin{split}
\mathcal V^{s,t}(x(\cdot),u(\cdot))&= -\overline{E} \left ( t,\tilde x_+(t), u(t) \right) +  \overline{E}\left (s, \tilde x_-(s), u(s) \right ) + \int_s^t [\partial_t E(\tau,x(\tau)) - f(u(\tau))\cdot x(\tau) ] \,  \di \tau
\\
&
=  \nu_{(\tilde x(\cdot), u(\cdot))}([s,t]) \,,
\end{split}
\end{displaymath}
where $\nu_{(\tilde x(\cdot), u(\cdot))}$ is the positive measure of Proposition \ref{p.2.15} (a).
From Lemma \ref{l.meas} we are allowed to consider the integration of these identities with respect to $\Pi^*$ (or $\Pi$), to eventually obtain \eqref{meas_eq}. We conclude noticing that, by Carath\'eodory extension theorem, the integrated measures define a positive Radon measure, which we denote
$
\mathcal V =\left [ \int_{\R^d \times\R^d \times\Gamma_{T}^{p}\times\Lambda}  \nu_{(\tilde x(\cdot), u(\cdot))} \di \Pi(x_0,u_0,x(\cdot),u(\cdot)) \right ] \in \mathcal M_b^+(0,T)$.
\end{proof}

\begin{remark}
It would be very tempting to write 
\begin{eqnarray*}
\int_{\Gamma_{T}^{p}\times\Lambda} \overline{E} \left ( t,\tilde x_\pm(t), u(t) \right) \, \di \Pi^*(x(\cdot),u(\cdot))&=& \int_{\R^{d}\times\R^{d}\times\Gamma_{T}^{p}\times\Lambda} \overline{E} \left ( t,\tilde x_\pm(t), u(t) \right) \, \di \Pi(x_0,u_0,x(\cdot),u(\cdot)) \\&=& \int_{\R^{d}\times\R^{d}} \overline{E} \left ( t,\tilde x_\pm, u \right) \, \di \eta_t(\tilde x,u),
\end{eqnarray*}
but, unfortunately, the function $t \mapsto \overline{E} \left ( t,\tilde x_\pm(t), u(t) \right)$ is only measurable and it would not be possible to obtain such a pointwise identity; a further integration in time may be needed in order to express the identities in terms of integrations with respect to $\eta$. Nevertheless, the identities \eqref{meas_eq} hold true pointwise for all  $0 \leq s \leq t \leq T$; it is perhaps a bit more abstract energy balance principle as one may have expected, but  it is also a quite concise description of the distribution of the energy.
\end{remark}

\section{Learning of energies and data-driven evolutions}\label{s.reconstruction}

In this section we focus on the problem of the reconstruction of the energy function~$E$, assuming that we have observed a certain large number $N$ of evolutions $x^{i}\colon [0,T]\to \R^{d}$, $i=1,\dots, N$, obtained as limit of solutions~$x^{i}_{\varepsilon}$ of the singularly perturbed gradient flow~\eqref{e.gradflow} as $\varepsilon\to 0$. The energy reconstruction will be  recast in terms of a minimum problem of a suitable discrepancy functional~$\J_{\eta}$, very much alike the left-hand-side of \eqref{e.9.5}. The functional depends explicitly on a measure~$\eta_{t} \otimes \mathcal{L}^{1}_{|[0,T]}$, which is  limit, along a subsequence, of the measures~$\eta^{N}_{\varepsilon,t}\otimes\mathcal{L}^{1}_{|[0,T]}$ as $N\to \infty$ and $\varepsilon \to0$, where $\eta^{N}_{\varepsilon, t} \in \sP(\R^{d}\times\R^{d})$ is the empirical measure centered at the $\varepsilon$-evolutions. In what follows, we propose and analyze a constructive and numerically implementable procedure which allows us to approximate~$E$ in a finite dimensional setting up to an arbitrarily small error, in a suitable sense.

\subsection{Learning as a variational problem}

Following the lines of~\cite{BFHM16}, we fix two constants $M,R>0$ and we consider the function class
\begin{equation}\label{e.XMR}
X_{M,R} \coloneq \{ \hat{E}\in W^{2,\infty}_{loc}([0,T]\times\R^{d}):\, \| \hat{E} \|_{W^{2,\infty}([0,T]\times \overline{ \mathrm{B}}_{R})}\leq M\}\,.
\end{equation}
Our particular choice of~$R,M$ is the following: let~$\eta_{0}\in\mathcal{P}_{c}(\R^{d}\times\R^{d})$ be a distribution of initial conditions~$(x_{0}, u_{0})$ with compact support in~$\R^{d}\times\R^{d}$. By Proposition~\ref{p.1} there exists $R>0$ such that the flow map $Y_{\varepsilon,t}(x_{0}, u_{0})\in K^\circ \subset \overline{\mathrm{B}}_{R}\times \overline{\mathrm{B}}_{R}$ for every $t\in[0,T]$, every~$(x_{0},u_{0})\in\supp(\eta_{0})$, and every~$\varepsilon>0$. In view of hypothesis (E1) of Section~\ref{s.ODE}, we have that $E\in C^{1,1}_{\rm{loc}}([0,T]\times \R^{d})$, so that to the given~$R$ corresponds $M=M(R)$ such that $E\in X_{M,R}$. We notice that~$R$ and~$M$ do not depend on~$\varepsilon$.

Having in mind the numerical implementation, we are interested in computing good approximations~$\hat{E}_{N}$ of~$E$ belonging to suitable finite dimensional subset~$V_{N}$ of~$X_{M,R}$, $N\in\mathbb{N}$. In particular $V_N$ is a suitable ball of a finite dimensional subspace of $W^{2,\infty}_{\rm{loc}}$, e.g., a suitable finite element subspace.
For the sequence $(V_{N})_{N\in\mathbb{N}}$ we make the following uniform approximation assumption.

\begin{definition}\label{d.UA}
Let $\eta\in \mathcal{M}_{b}^{+}([0,T]\times\R^{d}\times\R^{d})$, let~$(V_{N})_{N\in\mathbb{N}}$ be a sequence of finite dimensional subsets of~$X_{M,R}$, and let $\bar{\eta} \in \mathcal{M}_{b}^{+}([0,T]\times\R^{d})$ be defined as
\begin{equation}\label{e.bareta}
\bar{\eta} (B) \coloneq \eta(B\times \R^{d}) \qquad\text{for every Borel set $B\subseteq [0,T]\times\R^{d}$}\,.
\end{equation}
We say that~$(V_{N})_{N\in\mathbb{N}}$ has the \emph{uniform approximation property} with respect to~$\eta$ if for every $\hat{E}\in X_{M,R}$ there exists a sequence $\hat{E}_{N}\in V_{N}$ such that $\hat{E}_{N}\to \hat{E}$ in $W^{1,\infty}(\supp ( \bar \eta ) )$ as $N\to\infty$.
\end{definition}

\begin{remark}\label{r.4}
The role played by the measure~$\eta$ will be better clarified in the following discussions. We anticipate here that, roughly speaking, the support of~$\bar \eta$ represents the region of~$\R^{d}\times\R^{d}$ that has been explored by the observed evolutions. For this reasons, it is natural to assume the above  uniform approximation property only on~$\supp(\bar \eta)$. In the numerical simulations we will make an extensive use of this property, since we will employ suitable space refinements only on the regions of~$\R^{d}\times\R^{d}$ visited by some evolution. We refer to Section~\ref{s.numerics} for further details.
\end{remark}

We now introduce the key functionals for our reconstruction procedure. For every $N\in\mathbb{N}$ and every $\varepsilon>0$, let us fix~$N$ pairs~$(x^{i}_{0}, u^{i}_{0})\in\supp(\eta_{0})$ distributed according to~$\eta_{0}$ and let us consider the corresponding solutions~$(x^{i}_{\varepsilon}, u^{i}_{\varepsilon})\colon[0,T]\to \R^{d}\times\R^{d}$ of the ODE system~\eqref{e.gradflow}. As in Section~\ref{s.CE}, we define the empirical measure $\eta^{N}_{\varepsilon}\in C([0,T];\mathcal{P}(\R^{d}\times\R^{d}))$ by $\eta^{N}_{\varepsilon,t} \coloneq \tfrac{1}{N}\sum_{i=1}^{N}\delta_{(x^{i}_{\varepsilon}(t), u^{i}_{\varepsilon}(t))}$ and we set
\begin{equation}\label{e.10} 
\J_{N, \varepsilon}(\hat{E}) \coloneq \frac{1}{T}\int_{0}^{T}\int_{\R^{d}\times\R^{d}}|\nabla_{x} \hat{E}(t,x) - \nabla_{x} E(t,x)|^{2}\,\di\eta^{N}_{\varepsilon,t}(x,u) \,\di t \qquad\text{for every $\hat{E}\in V_{N}$}.
\end{equation}
In Proposition~\ref{p.2} we have shown that in the limit as $N\to \infty$ the sequence~$\eta^{N}_{\varepsilon}$ converges uniformly with respect to~$W_{1}$ to a curve $\eta_{\varepsilon}\in C([0,T];\mathcal{P}(\R^{d}\times\R^{d}))$ solution of the continuity equation~\eqref{e.5}. Accordingly, we define the functional
\begin{equation}\label{e.11}
\J_{\eta_\varepsilon}(\hat{E}) \coloneq \frac{1}{T}\int_{0}^{T}\int_{\R^{d}\times\R^{d}}|\nabla_{x} \hat{E}(t,x) - \nabla_{x} E(t,x)|^{2}\,\di\eta_{\varepsilon,t}(x,u) \,\di t \qquad\text{for every $\hat{E}\in X_{M,R}$}.
\end{equation}
We notice that, with our choice of~$R$ and~$M$,~$\supp(\eta^{N}_{\varepsilon, t}) \cup \supp(\eta_{\varepsilon,t}) \subseteq  \overline{\B}_{R}\times  \overline{\B}_{R}$ for every~$\varepsilon$, every~$N$, and every~$t$.

Finally, for a Borel family $\{\eta_{t}: t\in[0,T]\}\subseteq\mathcal{P}(\R^{d}\times\R^{d})$ we set
\begin{equation}\label{e.12}
\J_{\eta}(\hat{E}) \coloneq \frac{1}{T}\int_{0}^{T}\int_{\R^{d}\times\R^{d}}|\nabla_{x} \hat{E}(t,x)-u|^{2}\,\di\eta_{t}(x,u) \,\di t \qquad\text{for every $\hat{E}\in X_{M,R}$}.
\end{equation}
Notice that this functional is simply designed to naturally measure the discrepancy occurring in equation~\eqref{e.9.5}.

\begin{remark}
We notice that when~$\{\eta_{t}: t\in[0,T]\}\subseteq \mathcal{P}(\R^{d}\times\R^{d})$, see, e.g., \eqref{e.9.5}, is such that
\begin{equation}\label{e.concentration}
\int_{0}^{T}\int_{\R^{d}\times\R^{d}}|\nabla_{x} E(t,x) - u|^{2}\,\di\eta_{t}(x,u) \,\di t =0\,,
\end{equation}
the functional~$\J_{\eta}$ can be rewritten as
\begin{displaymath}
\J_{\eta}(\hat{E})=\frac{1}{T}\int_{0}^{T}\int_{\R^{d}\times\R^{d}} \!\!\!\!\!\!\!\! |\nabla_{x} \hat{E}(t,x)-\nabla_{x} E(t,x) |^{2}\di \eta_{t}(x,u)\,\di t = \frac{1}{T}\int_{[0,T] \times \R^{d}} \!\!\!\!\!\!\!\! |\nabla_{x} \hat{E}(t,x)-\nabla_{x} E(t,x) |^{2} \di \bar\eta(t,x)\,,
\end{displaymath}
where~$\eta\coloneq \eta_{t}\otimes \mathcal{L}^{1}_{|[0,T]}$ and~$\bar{\eta}$ is as in~\eqref{e.bareta}. From now on, we will tacitly use this transformation.
\end{remark}
Let $(x_{0}^{1}, u^{1}_{0}) \ldots, (x_{0}^{N}, u_{0}^{N}) \in \supp(\eta_0)$ and let~$(x^{i}_{\varepsilon}, u^{i}_{\varepsilon})\in \Gamma^{p}_{T}\times\Lambda$, for $i=1,\ldots, N$, be the solutions of the system~\eqref{e.gradflow}. Assume that for $i=1,\ldots,N$~$x^{i}_{\varepsilon}$ converges to a quasi-static evolution~$x_{i}$ in $\Gamma^{p}_{T}$ for every $1\leq p<\infty$ and $u^{i}_{\varepsilon}$ converges to~$u_{i}$ in~$\Lambda$. Let us consider the empirical measure $\eta^{N}_{t} \coloneq \tfrac{1}{N}\sum_{i=1}^{N} \delta_{(x_{i}(t), u_{i}(t))}\in \mathcal{P}(\R^{d}\times\R^{d})$ and set
\begin{equation}\label{e.XX}
 \J_{N}(\hat{E}) \coloneq \int_{0}^{T}\int_{\R^{d}\times\R^{d}} |\nabla_{x} \hat{E}(t,x) - u|^{2}\,\di\eta^{N}_{t}(x,u)\,\di t \qquad\text{for~$\hat{E}\in X_{M,R}$.}
\end{equation}
The numerically implementable algorithm to approximate the energy $E$ is based on the following finite dimensional optimization problem:
\begin{equation}\label{basic_alg}
\hat E_N:=\arg\min_{\hat{E}\in V_{N}}\, \J_{N}(\hat{E})
\end{equation}
As \eqref{basic_alg} defines a sequence of variational problems, we wishes to show that $\hat E_N \to E$ in a suitable sense, by using variational convergence arguments, such as $\Gamma$-convergence. As it is a standard notion, we refer to \cite{Braides02,MR1201152} for more details. 

\subsection{Approximation by $\Gamma$-convergence}

Our construction is guided by the following (essentially commutative) diagram of limits:
\begin{equation}\label{limit_dia}
{\large
\begin{array}{ccc}
\mathcal J_{N,\varepsilon}&\xrightarrow{\varepsilon}& \mathcal J_N\\
N  \downarrow & \Gamma   \searrow N,\varepsilon_N& \downarrow {\small N}\\
\J_{\eta_\varepsilon} &\xrightarrow{\varepsilon}& \J_\eta
\end{array}
}
\end{equation}
The following results clarify the limits appearing in the diagram. We start from the bottom of the diagram, showing the uniform convergence of $\J_{\eta_\varepsilon}$ to $\J_\eta$ for $\varepsilon \to 0$.

\begin{proposition}\label{p.4}
Let $\eta_{\varepsilon}\in C([0,T];\mathcal{P}(\R^{d}\times\R^{d}))$ be as in Proposition~\ref{p.2} and assume that, along a not relabeled subsequence,~$\eta_{\varepsilon,t}\otimes\mathcal{L}^{1}_{| [0,T]}$ converges narrowly to $\eta_{t}\otimes\mathcal{L}^{1}_{| [0,T]}$. Then,~$\J_{\eta_\varepsilon}$ converges to~$\J_{\eta}$  uniformly in~$X_{M,R}$ for $\varepsilon \to 0$.
\end{proposition}

\begin{proof}
In view of Proposition~\ref{p.3.1} we have that~$\{\eta_{t}:t\in[0,T]\}$ satisfies~\eqref{e.concentration}.
 By construction of~$\eta_{\varepsilon}$, we have that the support of $\eta_{\varepsilon, t}$ is contained in~$\overline{\B}_{R}\times \overline{\B}_{R}$ for every $t\in[0,T]$. This implies that $\supp (\eta_{t}\otimes\mathcal{L}^{1}_{|[0,T]})\subseteq [0,T]\times\overline{\B}_{R}\times \overline{\B}_{R}$ and
\begin{displaymath}
|\J_{\eta_\varepsilon}(\hat{E}) - \J_{\eta}(\hat{E})| \leq \Lip_{ [0,T]\times\overline{\B}_{R}} ( | \nabla_{x} \hat{E} - \nabla_{x} E |^{2} ) W_{1}(\eta_{\varepsilon,t}\otimes\mathcal{L}^{1}_{|[0,T]}, \eta_{t}\otimes\mathcal{L}^{1}_{|[0,T]})\,.
\end{displaymath}
Since $\hat{E}$ and $E$ belong to $X_{M,R}$, we have that $\Lip_{ [0,T]\times\overline{\B}_{R}} ( | \nabla_{x} \hat{E} - \nabla_{x} E |^{2}) \leq 8M^{2}$, so that
\begin{displaymath}
|\J_{\eta_\varepsilon}(\hat{E}) - \J_{\eta}(\hat{E})| \leq 8M^{2} \, W_{1}(\eta_{\varepsilon,t}\otimes\mathcal{L}^{1}_{|[0,T]}, \eta_{t}\otimes\mathcal{L}^{1}_{|[0,T]})\,.
\end{displaymath}
Since $\eta_{\varepsilon,t}\otimes\mathcal{L}^{1}_{| [0,T]}$ converges narrowly to $\eta_{t}\otimes\mathcal{L}^{1}_{| [0,T]}$ and their supports are uniformly bounded, we also have that $W_{1}(\eta_{\varepsilon,t}\otimes\mathcal{L}^{1}_{|[0,T]}, \eta_{t}\otimes\mathcal{L}^{1}_{|[0,T]}) \to 0$ as $\varepsilon\to0$, and the proof is thus concluded.
\end{proof}

%

Let us now continue by describing the approximation provided by the upper limits of the diagram~\eqref{limit_dia}.

\begin{proposition}\label{p.6}
Let $\delta>0$ and $N\in\mathbb{N}$ be given. Let $(x_{0}^{1}, u^{1}_{0}) \ldots, (x_{0}^{N}, u_{0}^{N}) \in \supp(\eta_0)$ and let~$(x^{i}_{\varepsilon}, u^{i}_{\varepsilon})\in \Gamma^{p}_{T}\times\Lambda$, for $i=1,\ldots, N$, be the solutions of the system~\eqref{e.gradflow}. Assume that for $i=1,\ldots,N$~$x^{i}_{\varepsilon}$ converges to a quasi-static evolution~$x^{i}$ in $\Gamma^{p}_{T}$ for every $1\leq p<\infty$ and $u^{i}_{\varepsilon}$ converges to~$u^{i}$ in~$\Lambda$. Let us consider the empirical measure $\eta^{N}_{t} \coloneq \tfrac{1}{N}\sum_{i=1}^{N} \delta_{(x^{i}(t), u^{i}(t))}\in \mathcal{P}(\R^{d}\times\R^{d})$ and set
\begin{equation}\label{e.15}
 \J_{N}(\hat{E}) \coloneq \int_{0}^{T}\int_{\R^{d}\times\R^{d}} |\nabla_{x} \hat{E}(t,x) - u|^{2}\,\di\eta^{N}_{t}(x,u)\,\di t \qquad\text{for~$\hat{E}\in X_{M,R}$.}
\end{equation}
Then, there exist~$\varepsilon_{N}>0$ and two positive constants $\overline{C}_{1}, \overline{C}_{2}$ (independent of~$\delta$ and~$N$) such that for every $\hat{E}\in X_{M,R}$ and every $0<\varepsilon\leq \varepsilon_{N}$
\begin{equation}\label{e.14}
 \J_{N,\varepsilon}(\hat{E})\leq \overline{C}_{1} ( \J_{N}(\hat{E})  +\delta + \varepsilon) \qquad \text{and} \qquad  \J_{N}(\hat{E})\leq \overline{C}_{2}( \J_{N,\varepsilon}(\hat{E})  + \delta + \varepsilon)\,.
\end{equation}
\end{proposition}

\begin{remark}
We notice that the hypothesis on the strong convergence of~$x^{i}_{\varepsilon}$ to~$x_{i}$ is not too restrictive in view of the compactness result shown in Theorem~\ref{t.2}. In fact, as a modeling assumption, we prescribe here that the observed quasi-static evolutions~$x_{i}$ are limit of the singularly perturbed dynamic described by~\eqref{e.gradflow}.
\end{remark}

\begin{proof}[Proof of Proposition~\ref{p.6}]
In view of the convergence hypothesis on $x^{i}_{\varepsilon}$, there exists~$\varepsilon_{N}>0$ such that
\begin{displaymath}
\| x^{i}_{\varepsilon} - x^{i} \|_{L^{1}(0,T)}, \, \| u^{i}_{\varepsilon} - u^{i}\|_{L^{\infty}(0,T)}  <\delta \qquad\text{for every $i=1,\ldots, N$ and every $0< \varepsilon \leq \varepsilon_{N}$}.
\end{displaymath}
In view of~(c) of Proposition~\ref{p.1}, we have
\begin{displaymath}
\int_{0}^{T}\int_{\R^{d}\times\R^{d}}|\nabla_{x} E(t,x) - u|^{2} \di \eta_{\varepsilon,t}^{N}(x, u)\,\di t\leq \overline{C}\varepsilon\,,
\end{displaymath}
for some positive constant~$\overline{C}$ independent of~$N$ and~$\varepsilon$. Thus, for every~$\hat{E}\in X_{M,R}$ we have that
\begin{align}
\J_{N,\varepsilon} & (\hat{E})  \leq \frac{2}{T}\int_{0}^{T}\int_{\R^{d}\times\R^{d}}\!\!\!\!\! |\nabla_{x} \hat{E}(t,x) - u|^{2}\,\di\eta^{N}_{\varepsilon,t}(x,u)\,\di t +\frac{2}{T}\int_{0}^{T}\int_{\R^{d}\times\R^{d}} \!\!\!\!\! |\nabla_{x} E(t,x) - u|^{2}\,\di\eta^{N}_{\varepsilon,t}(x,u)\,\di t \nonumber \\
& \leq \frac{2}{NT}\sum_{i=1}^{N}\int_{0}^{T} |\nabla_{x} \hat{E}(t,x^{i}_{\varepsilon}(t)) - u^{i}_{\varepsilon}(t)|^{2}\,\di t +\frac{2\overline{C}}{T}\varepsilon \nonumber \\
& = 2\J_{N}(\hat{E})+ \frac{2}{NT}\sum_{i=1}^{N}\int_{0}^{T}\Big( |\nabla_{x} \hat{E}(t,x^{i}_{\varepsilon}(t)) - u^{i}_{\varepsilon}(t)|^{2} - |\nabla_{x} \hat{E}(t,x^{i}(t)) - u^{i}(t)|^{2}\Big)\,\di t +\frac{2\overline{C}}{T}\varepsilon \label{e.16} \\
& = 2\J_{N}(\hat{E})+ \frac{2}{NT}\sum_{i=1}^{N}\int_{0}^{T}\big ( \nabla_{x} \hat{E}(t,x^{i}_{\varepsilon}(t)) - \nabla_{x} \hat{E}(t,x^{i}(t)) - u^{i}_{\varepsilon}(t) + u^{i}(t) \big ) \nonumber \\
&\qquad {\,\cdot\,} \big ( \nabla_{x} \hat{E}(t,x^{i}_{\varepsilon}(t)) + \nabla_{x} \hat{E}(t,x^{i}(t)) - u^{i}_{\varepsilon}(t) - 
u^i(t) \big)\,\di t +\frac{2\overline{C}}{T}\varepsilon\,, \nonumber
\end{align}
Thanks to our choice of~$R$, we have that~$\|x^{i}_{\varepsilon}\|_{\infty}$, $\|x^{i}\|_{\infty}$, $\|u^{i}\|_{\infty} \leq R$ uniformly in~$\varepsilon$. Hence, for every $t\in[0,T]$ and every $i=1,\ldots, N$,
\begin{eqnarray*}
&&  | \nabla_{x} \hat{E}(t,x^{i}_{\varepsilon}(t)) - \nabla_{x} \hat{E}(t,x^{i}(t))|\leq  \Lip_{[0,T]\times\overline{\B}_{R}}(\nabla_{x} \hat{E}) |x^{i}_{\varepsilon}(t) - x^{i}(t)|\,,\\[1mm]
&&  |\nabla_{x} \hat{E}(t,x^{i}_{\varepsilon}(t)) + \nabla_{x} \hat{E}(t,x^{i}(t)) - u^{i}_{\varepsilon}(t) - u^{i}(t)| \leq 2(M+R)\,.
\end{eqnarray*}

In view of the previous estimates, assuming that $0< \varepsilon\leq \varepsilon_{N}$ we continue in~\eqref{e.16} with
\begin{displaymath}
\begin{split}
\J_{N,\varepsilon}(\hat{E}) & \leq 2\J_{N}(\hat{E})+ \frac{4(M + R)}{NT}\sum_{i=1}^{N}\int_{0}^{T} \Lip_{[0,T]\times\overline{\B}_{R}}(\nabla_{x} \hat{E}) |x^{i}_{\varepsilon}(t) - x^{i}(t)|\,\di t +2\delta(M+R) + \frac{2\overline{C}}{T}\varepsilon\\
& \leq  2\J_{N}(\hat{E}) + \frac{4(M + R)M}{T}\,\delta + 2\delta(M+R) + \frac{2\overline{C}}{T}\varepsilon \leq \overline{C}_{1}(\J_{N}(\hat{E}) + \delta + \varepsilon)\,.
\end{split}
\end{displaymath}
In a similar way we can show the second inequality in~\eqref{e.14}.
\end{proof}

In the next proposition we show a uniform estimate of the distance between~$\J_{N, \varepsilon}$ and~$\J_{\eta}$, which explains the central diagonal $\Gamma$-limit of
\eqref{limit_dia}. 

\begin{proposition}\label{p.7}
Let $\{\eta_{t}: t\in[0,T]\}$ be a Borel family in~$\mathcal{P}(\R^{d}\times\R^{d})$ with uniformly compact support and such that~\eqref{e.concentration} is satisfied. Let $\eta \coloneq \eta_{t}\otimes\mathcal{L}^{1}_{|[0,T]}$,~$\bar{\eta}\in \mathcal{M}^{+}_{b}([0,T]\times\R^{d})$ as in~\eqref{e.bareta}, and $\hat{E}_{1}, \hat{E}_{2}\in X_{M,R}$. Assume that $K\coloneq \supp(\bar\eta)$ is contained in $[0,T]\times \overline{\mathrm{B}}_{R}$. Then,
\begin{equation}\label{e.17}
| \J_{N, \varepsilon}(\hat{E}_{1}) - \J_{\eta}(\hat{E}_{2})| \leq  \frac{8M^{2}}{T}\,  W_{1}( \eta^{N}_{\varepsilon, t} \otimes \mathcal{L}^{1}_{|[0,T]}, \eta_{t}\otimes\mathcal{L}^{1}_{|[0,T]}) + 4M \|\hat{E}_{1} - \hat{E}_{2}\|_{W^{1,\infty}(K)} \,.
\end{equation}
\end{proposition}

\begin{proof}
For $\hat{E}_{1}, \hat{E}_{2} \in X_{M,R}$ we have
\begin{equation}\label{e.19}
| \J_{N, \varepsilon}(\hat{E}_{1}) - \J_{\eta}(\hat{E}_{2})| \leq  | \J_{N, \varepsilon}(\hat{E}_{1}) -  \J_{\eta}(\hat{E}_{1}) | + | \J_{\eta}(\hat{E}_{1}) - \J_{\eta}(\hat{E}_{2})| =: I_{1}+ I_{2}\,.
\end{equation}

 For~$I_{1}$, we have that
 \begin{equation}\label{e.20}
\begin{split}
 I_{1}& = \Big| \frac{1}{T}\int_{0}^{T} \int_{\R^{d}\times\R^{d}} |\nabla_{x} \hat{E}_{1}(t,x) - \nabla_{x} E(t,x)|^{2}\,\di (\eta^{N}_{\varepsilon,t} - \eta_{t})(x,u)\, \di t \Big| \\
& \leq \frac{1}{T} \, \Lip_{[0,T]\times \overline{\B}_{R}} (|\nabla_{x}\hat{E}_{1} - \nabla_{x} E|^{2})\, W_{1}( \eta^{N}_{\varepsilon, t} \otimes \mathcal{L}^{1}_{|[0,T]}, \eta_{t}\otimes\mathcal{L}^{1}_{|[0,T]}) \\
&\leq   \frac{8M^{2}}{T}\,  W_{1}( \eta^{N}_{\varepsilon, t} \otimes \mathcal{L}^{1}_{|[0,T]}, \eta_{t}\otimes\mathcal{L}^{1}_{|[0,T]})\,.
 \end{split}
 \end{equation}
 
 As for~$I_{2}$ we write
 \begin{equation}\label{e.22}
 \begin{split}
 I_{2} & =  \Big| \frac{1}{T}\int_{0}^{T} \int_{\R^{d}\times\R^{d}} |\nabla_{x} \hat{E}_{1}(t,x) - \nabla_{x} E(t,x)|^{2} - |\nabla_{x} \hat{E}_{2}(t,x) - \nabla_{x} E(t,x)|^{2}\,\di \eta_{t}(x,u)\, \di t \Big|\\
 & \leq  \frac{1}{T}\int_{K} |\nabla_{x} \hat{E}_{1}(t,x) - \nabla_{x} \hat E_{2}(t,x)| \Big( |\nabla_{x} \hat{E}_{1}(t,x) - \nabla_{x} E(t,x)| \\
 &\qquad \qquad+  |\nabla_{x} \hat{E}_{2}(t,x) - \nabla_{x} E(t,x)| \Big) \,\di \bar \eta(t, x) \leq 4M \|\hat{E}_{1} - \hat{E}_{2}\|_{W^{1,\infty}(K)} \,,
 \end{split}
 \end{equation}
 where we have used the inequality~$\bar{\eta}(K)\leq T$.
 \end{proof}

\begin{corollary}\label{c.1}
Let $(\varepsilon_{N})_{N}$ be a null sequence. Assume that the sequence of measures~$(\eta^{N}_{\varepsilon_{N}, t} \otimes \mathcal{L}^{1}_{| [0,T]})_N$ converges narrowly (and hence with respect to~$W_{1}$) to~$\eta \coloneq \eta_{t}\otimes \mathcal{L}^{1}_{| [0,T]}$, for a Borel family~$\{\eta_{t}:t\in[0,T]\} \subseteq \mathcal{P}(\R^{d}\times\R^{d})$. Let~$(V_{N})_N$ be a sequence of closed subspaces of~$X_{M,R}$ having the uniform approximation property with respect to~$\eta$. Then, the sequence of functionals~$(\J_{N,\varepsilon_{N}})_N$ defined on~$V_{N}$ $\Gamma$-converges to~$\J_{\eta}$ in~$X_{M,R}$ with respect to~the topology of~$W^{1,\infty}(\supp(\bar{\eta}))$.
\end{corollary}

\begin{proof} 
The $\Gamma$-liminf inequality follows directly from Proposition~\ref{p.7}. In a similar way, the $\Gamma$-limsup inequality is a consequence of Proposition~\ref{p.7} and of Definition~\ref{d.UA}, which ensures that for every~$\hat{E}\in X_{M,R}$ there exists a sequence $\hat{E}_{N}\in V_{N}$ such that $\hat{E}_{N}\to \hat{E}$ in $W^{1,\infty}(\supp( \bar \eta))$, where~$\bar\eta$ is as in~\eqref{e.bareta}.
\end{proof}

In the next two propositions we discuss the convergence of minimizers of the functionals~$\J_{N, \varepsilon}$,~$\J_{N}$ to minimizers of~$\J_{\eta}$.

\begin{proposition}\label{p.8}
Let $N\in\mathbb{N}$, let~$V_{N}$ be a closed subspace of~$X_{M,R}$, and let~$\eta^{N}_{\varepsilon, t},\, \eta^{N}_{t}\in\sP(\R^{d}\times\R^{d})$ be as in Proposition~\ref{p.6}. Then, the minimum problems
\begin{equation}\label{e.25}
\min_{\hat{E}\in V_{N}}\, \J_{N, \varepsilon}(\hat{E})\qquad\text{and}\qquad \min_{\hat{E}\in V_{N}}\, \J_{N}(\hat{E})
\end{equation}
admit a solution.
\end{proposition}

\begin{proof}
Let $\hat{E}_{m} \in V_{N}$ be a minimizing sequence for~$\J_{N,\varepsilon}$. By compactness of~$X_{M,R}$, the sequence $(\hat{E}_{m})_m$ converges, up to subsequence, to some function~$\hat{E}\in V_{N}$ in $W^{1,\infty}([0,T]\times\overline{\B}_{R})$. Hence, by dominated convergence we deduce that $\J_{N,\varepsilon}(\hat{E}_{m}) \to \J_{N,\varepsilon}(\hat{E})$. In a similar way we get the existence of minimizers for~$\J_{N}$ on~$V_{N}$.
\end{proof}

\begin{proposition}\label{p.9}
Let~$\delta>0$ and let~$(\varepsilon_{N})_N$ be a null sequence. Assume that the hypotheses of Corollary~\ref{c.1} are satisfied. Then, every sequence~$(\hat{E}_{N})_N$  of minimizers of~$\J_{N,\varepsilon_{N}}$ in~$V_N$ admits a subsequence (not relabeled) converging to~$\hat{E}\in X_{M,R}$ in~$W^{1,\infty}([0,T]\times\overline{\B}_{R})$. In particular,~$\hat{E}$ is a minimizer of the functional~$\J_{\eta}$.
\end{proposition}

\begin{proof} 
By definition of~$X_{M,R}$ a sequence of minimizers~$\hat{E}_{N}$ of~$\J_{N,\varepsilon_{N}}$ is weak$^*$-compact in~$X_{M,R}$, and we denote with~$\hat{E}$ the weak$^*$-limit of a suitable subsequence of~$\hat{E}_{N}$. In view of Proposition~\ref{p.7}, $\J_{N,\varepsilon_{N}} (\hat{E}_{N})$ converges to~$\J_{\eta}(\hat{E})$. From the minimality of~$\hat{E}_{N}$ and the uniform approximation property satisfied by the subspaces~$V_{N}$, we can easily prove that~$\hat{E}$ is a minimizer of~$\J_{\eta}$ in~$X_{M,R}$.
\end{proof}

\begin{theorem}\label{p.10}
Let $\delta>0$,~$\varepsilon_{N}>0$, and~$\eta^{N}_{\varepsilon_{N},t}, \, \eta^{N}_{t}\in\mathcal{P}(\R^{d}\times\R^{d})$ be as in Proposition~\ref{p.6}. Assume that there exists a Borel family~$\{\eta^{\delta}_{t}:t\in[0,T]\}\subseteq\mathcal{P}(\R^{d}\times\R^{d})$ such that $\eta^{N}_{\varepsilon_{N},t}\otimes\mathcal{L}^{1}_{|[0,T]}$ converges narrowly to~$\eta^{\delta} \coloneq\eta^{\delta}_{t}\otimes\mathcal{L}^{1}_{|[0,T]}$. Suppose that~$(V_{N})_{N\in\mathbb{N}}$ satisfies the uniform approximation property with respect to~$\eta^{\delta}$ and denote with~$\hat{E}_{N}\in V_{N}$ a solution of
\begin{equation}\label{e.22.5}
\min_{\hat{E}\in V_{N}}\, \J_{N}(\hat{E})\,.
\end{equation}
Then,~$(\hat{E}_{N})_N$ converges, up to a subsequence, to some~$\hat{E}_{\delta}\in X_{M,R}$ satisfying
\begin{equation}\label{e.23}
\J_{\eta^{\delta}}(\hat{E}_{\delta}) \leq C \delta
\end{equation}
for a positive constant~$C$ independent of~$\delta$.

Moreover, 
 there exist a further Borel family~$\{\eta_{t}: t\in [0,T]\} \subseteq \mathcal{P}(\R^{d}\times\R^{d})$ and a further~$\hat{E}\in X_{M,R}$ such that, up to a subsequence,~$\eta^{\delta}$ converges narrowly to $\eta \coloneq \eta_{t}\otimes \mathcal{L}^{1}_{|[0,T]}$,~$\hat{E}_{\delta}\to \hat{E}$ in~$W^{1,\infty}([0,T]\times \overline{B}_{R})$ for $\delta \to 0$, and 
\begin{equation}\label{e.23.5}
\J_{\eta}(\hat{E}) =0\,.
\end{equation}
\end{theorem}

\begin{proof}
Let $\hat{E}_{N,\varepsilon_{N}}$ be a solution of
\begin{displaymath}
\min_{\hat{E}\in V_{N}}\, \J_{N, \varepsilon_{N}}(\hat{E}) \,.
\end{displaymath}
Pairing the inequalities~\eqref{e.14} of Proposition~\ref{p.6} for~$\J_{N,\varepsilon_{N}}$ and~$\J_{N}$ we get that
\begin{equation}\label{e.24}
\begin{split}
\J_{N,\varepsilon_{N}}(\hat{E}_{N}) & \leq \overline{C}_{1} ( \J_{N}(\hat{E}_{N}) + \delta + ,\varepsilon_{N} ) \leq \overline{C}_{1}  ( \J_{N}(\hat{E}_{N,\varepsilon_{N}}) + \delta + \varepsilon_N)\\
& \leq \overline{C}_{1}\overline{C}_{2} ( \J_{N,\varepsilon_{N}}(\hat{E}_{N,\varepsilon_{N}}) + \delta +  \varepsilon_N)\,.
\end{split}
\end{equation}
In particular, the constants~$\overline{C}_{1}$ and~$\overline{C}_{2}$ do not depend on~$\delta$ and~$N$. By definition of~$X_{M,R}$, the sequence~$\hat{E}_{N}$ converges, up to a not relabeled subsequence, to some~$\hat{E}\in W^{2,\infty}([0,T]\times\overline{\B}_{R})$ in the~$W^{1,\infty}$-norm. Up to an extension, we may assume~$\hat{E}\in X_{M,R}$. By Proposition~\ref{p.7},~$\J_{N,\varepsilon_{N}}(\hat{E}_{N})$ converges to~$\J_{\eta^{\delta}}(\hat{E})$ as $N\to\infty$. In view of Proposition~\ref{p.9} we have that, up to a further subsequence, $\hat{E}_{N,\varepsilon_{N}}$ converges in~$W^{1,\infty}([0,T]\times\overline{\B}_{R})$ to a minimizer of~$J_{\eta^{\delta}}$. Hence, applying again Proposition~\ref{p.7}, we deduce that $\J_{N,\varepsilon_{N}}(\hat{E}_{N,\varepsilon_{N}}) \to0$. Thus, applying Corollary~\ref{c.1} to~\eqref{e.24} and taking into account the convergence of~$\hat{E}_{N}$ to~$\hat{E}$, we deduce~\eqref{e.23}.

The second part of the proposition follows immediately from~\eqref{e.23}, noticing that the measures~$\eta^{\delta}$ have uniform compact support in~$[0,T]\times \R^{d}\times\R^{d}$ and are therefore compact with respect to narrow convergence.
\end{proof}

\begin{remark}\label{r.10}
Let us briefly comment on the result of Theorem~\ref{p.10}. As we noticed in~\eqref{e.12}, by Proposition~\ref{p.3.1} we have that the measure~$\eta \coloneq \eta_{t}\otimes \mathcal{L}^{1}_{|[0,T]}$ is concentrated on the set
\begin{displaymath}
\{(t,x,u)\in [0,T]\times\R^{d}\times\R^{d}:\, \nabla_{x} E(t,x)=u\}\,.
\end{displaymath}
Therefore, for every $\hat{E}\in X_{M,R}$ we can write
\begin{displaymath}
\J_{\eta}(\hat{E}) = \frac{1}{T} \int_{0}^{T}\int_{\R^{d}\times\R^{d}} | \nabla_{x} \hat{E}(t,x) - \nabla_{x} E(t,x)|^{2} \,\di\eta_{t}(x,u)\,\di t = \frac{1}{T} \| \nabla_{x} \hat{E}(t,x) - \nabla_{x} E(t,x) \|_{L^{2}([0,T] \times\R^{d}, \bar{\eta} )}^{2}\,.
\end{displaymath}
Hence, equality~\eqref{e.23.5} in Theorem~\ref{p.10} can be reformulated as
\begin{displaymath}
\| \nabla_{x} \hat{E}(t,x) - \nabla_{x} E(t,x) \|_{L^{2}([0,T] \times\R^{d}, \bar\eta )}^{2} = 0 \,.
\end{displaymath}
This of course implies that $\nabla_{x} \hat{E}(t,x) = \nabla_{x} E(t,x)$~$\bar\eta$-a.e. in $[0,T] \times\R^{d}$, that is, we are able to reconstruct the spatial gradient of the energy function~$E$ in the region of~$[0,T]\times\R^{d}$ that has been explored by the quasi-static evolutions~$x_{i}(\cdot)$ which, when the number of observed evolutions~$N$ is very large, is well approximated by the support of the measure~$\bar \eta$. This is indeed a natural constraint, since we have no information about the region of~$[0,T]\times\R^{d}$ that has not been explored by a quasi-static evolution.

Even if the measure~$\bar\eta$ is pretty much resulting from an abstract construction, since it has been obtained by applying a compactness argument to the sequence of measures~$\eta^{\delta}$ constructed in the first part of Theorem~\ref{p.10}, we anyway claim that our approach to energy reconstruction is entirely constructive and numerically implementable. Let us briefly explain why. In the last part of Theorem~\ref{p.10} we have shown that~$\eta$ is the limit as $\delta\to 0$ of the sequence $\eta^{\delta} \coloneq \eta^{\delta}_{t}\otimes \mathcal{L}^{1}_{|[0,T]}$. The measure~$\eta^{\delta}$ satisfies
\begin{displaymath}
C\delta\geq \J_{\eta^{\delta}}(\hat{E}_{\delta}) = \frac{1}{T} \|\nabla_{x} \hat{E}_{\delta}(t,x) - \nabla_{x} E(t,x) \|_{L^{2}([0,T] \times \R^{d}, \bar{\eta}^{\delta} )}^{2}\,.
\end{displaymath}
This implies that~$\nabla_{x}{\hat{E}}_{\delta}$ is itself a good approximation of~$\nabla_{x}{E}$ in the space $L^{2}([0,T]\times\R^{d}, \bar{\eta}^{\delta} )$. Moreover, the first part of Theorem~\ref{p.10} gives us another important piece of information:~$\eta^{\delta}$ and~$\hat{E}_{\delta}$ can be approximated in a ``finite dimensional-finite number of evolutions'' setting in which, indeed, we work only with a finite number~$N$ of observed evolutions, which completely determine the empirical measure $\eta^{N}_{\varepsilon_{N}, t}$, and we solve the minimum problem~\eqref{e.22.5} on a suitable finite dimensional subspace~$V_N$ of~$X_{M,R}$.
\end{remark}

\begin{remark}
We note that we could also obtain a result similar to Theorem~\ref{p.10} in nature, in which the roles of~$\J_{N,\varepsilon_{N}}$ and~$\J_{N}$ are exchanged. Let~$\delta>0$ and~$\varepsilon_{N}$ be as in Proposition~\ref{p.6}, and assume that the sequence of measures~$\eta_{N}\coloneq \eta^{N}_{t}\otimes \mathcal{L}^{1}_{|[0,T]}$ converges narrowly to~$\eta^{\delta} \coloneq \eta^{\delta}_{t}\otimes \mathcal{L}^{1}_{|[0,T]}$ and that~$(V_{N})_{N\in\mathbb{N}}$ satisfies the uniform approximation property with respect to~$\eta_{\delta}$. Then, denoted with~$\hat{E}_{N}\in V_{N}$ a solution of
\begin{displaymath}
\min_{\hat{E}\in V_{N}}\, \J_{N,\varepsilon_{N}} (\hat{E})\,,
\end{displaymath}
we have that, up to a subsequence,~$\hat{E}_{N}$ converges to some~$\hat{E}_{\delta}\in X_{M,R}$ satisfying
\begin{displaymath}
\J_{\eta^{\delta}}(\hat{E}_{\delta}) \leq C \delta
\end{displaymath}
for a positive constant~$C$ independent of~$\delta$. Moreover, in the limit as~$\delta\to 0$,$\eta^{\delta}$ converges narrowly to some~$\eta\coloneq \eta_{t}\otimes\mathcal{L}^{1}_{|[0,T]}$,~$\hat{E}_{\delta}$ converges to some~$\hat{E}$ in~$W^{1,\infty}([0,T]\times\overline{\B}_{R})$ and $\J_{\eta}(\hat{E})=0$. The proof of this result is still based on the arguments of Propositions~\ref{p.6} and~\ref{p.7}.

Such an approximation result will be used as a practical proxy for \eqref{e.22.5} in the numerical experiments in Section~\ref{s.numerics}.
\end{remark}

\subsection{Data-driven evolutions of critical points}\label{sec:ddecp}

In the previous section we obtained compactness results, which explain the approximation of the energy $E$ by data-driven energies $\hat{E} \in W^{2,\infty}([0,T]\times \overline{\B}_{R})$ and, for $\delta>0$, $ \hat{E}_{\delta} \in W^{2,\infty}([0,T]\times \overline{\B}_{R})$ constructed in Theorem~\ref{p.10}. In this section we show  pointwise error estimates on the singularly perturbed evolutions generated using the data-driven energies $\hat E_N$, $\hat{E}_{\delta}$, and $\hat{E}$, with respect to evolutions of critical points of the original energy $E$.

\begin{corollary}\label{thm:DDE}
Let $\hat{E}$, $\hat{E}_{N}$, and $\eta$ be as in Theorem~\ref{p.10}, and let~$\bar\eta\in \mathcal{M}_{b}^{+}([0,T]\times\R^{d})$ be defined as in~\eqref{e.bareta}. Then, for every $\varepsilon>0$ there exists $N =N(\varepsilon) \in \mathbb N$ large enough such that the solution~$(\hat{x}_{\varepsilon}^N, u_{\varepsilon})$ of the system \eqref{e.gradflow} with initial condition $(x_0, u_0)$ and energy $\hat{E}_{N}$ fulfills the error estimate
\begin{equation}\label{quantest}
\max_{t \in [0,T]} |\hat{x}_{\varepsilon}^{N} (t) - {x}_{\varepsilon}(t)| \leq \varepsilon + \frac{2M}{\varepsilon} e^{\frac{M}{\varepsilon}} \| \mathrm{dist}(x_{\varepsilon}(\cdot), \supp(\bar\eta))\|_{L^{1}(0,T)} \,,
\end{equation}
where $(x_{\varepsilon}, u_{\varepsilon})$ denotes the solution of~\eqref{e.gradflow} with initial condition $(x_0, u_0)$ and energy~$E$ and $\mathrm{dist}(\cdot, K)$ is the usual distance from a set~$K\subseteq\R^{d}$.
\end{corollary}

\begin{remark}
Let us comment on formula~\eqref{quantest}. Although the error estimate does not a priori guarantee the data-driven evolution~$\hat{x}^{N}_{\varepsilon}$ to be close to~$x_{\varepsilon}$, we anyway expect it to happen in most of the  applications.
Indeed, as shown in the numerical experiments in Section~\ref{s.numerics}, increasing the number of observed evolutions results in the enlargement of~$\supp(\bar\eta)$. This means that even if, according to Theorem~\ref{t.2}, the distance of~$x_{\varepsilon}$ from a quasi-static evolution of critical points~$t\mapsto x(t)$ has no clear rate of convergence, the distance of~$x_{\varepsilon}$ from the whole $\supp(\bar\eta)$, union of the orbits of all the observed quasi-static evolutions, can be expected to satisfy the condition
\begin{equation}\label{e.rmkDDE}
\lim_{\varepsilon\to 0} \, \frac{2M}{\varepsilon} e^{\frac{M}{\varepsilon}} \| \mathrm{dist}(x_{\varepsilon}(\cdot), \supp(\bar\eta))\|_{L^{1}(0,T)}  = 0\,.
\end{equation}
In particular, we refer to Section~\ref{s.numerics} for some numerical examples of data-driven evolutions. Here, we conclude by noticing that~\eqref{quantest}-\eqref{e.rmkDDE} imply that $\hat{x}_{\varepsilon} - x_{\varepsilon}$ tends to zero uniformly in~$[0,T]$. Therefore, we deduce from Theorem~\ref{t.2} that, along a suitable subsequence~$\varepsilon_{n} \to 0$, $\hat{x}^{N_{n}}_{\varepsilon_{n}} \to x$ in $L^{p}(0,T)$ for $p<+\infty$, where $N_{n} = N(\varepsilon_{n})$. Finally, in the particular case $ \| \mathrm{dist}(x_{\varepsilon}(\cdot), \supp(\bar\eta))\|_{L^{1}(0,T)}  = 0$ we even have that $\hat{x}_{\varepsilon} = x_{\varepsilon}$, since $\nabla_{x} \hat{E} = \nabla_{x} E$ on~$\supp(\bar\eta)$.
\end{remark}

%

\begin{proof}[Proof of Corollary~\ref{thm:DDE}]
Let~$\hat{E}_{\delta} \in W^{1,\infty}([0,T]\times\overline{\B}_{R})$ be as in~\eqref{e.23} of Theorem~\ref{p.10}, so that $\hat{E}_{\delta} \to \hat{E}$ as $\delta \to 0$ and $\hat{E}_{N} \to \hat{E}_{\delta}$ for $N\to \infty$ in $W^{1,\infty}([0,T]\times \overline{\B}_{R})$. In particular, for every $\varepsilon>0$ there exist $\delta>0$ small enough and~$N$ large enough  such that 
\begin{equation}\label{c:0:1}
\max\Big  \{ \| \nabla_x \hat{E}_{N} - \nabla_x \hat{E}_{\delta} \|_{L^{\infty}([0,T]\times \overline{\B}_{R})}, \| \nabla_x \hat{E}_{\delta} - \nabla_{x} \hat E\|_{L^{\infty}([0,T]\times \overline{\B}_{R})} \Big \} \leq \frac{\varepsilon^2}{2 T} e^{-\frac{M}{\varepsilon}}.
\end{equation}
For $\varepsilon >0$, let $(x_\varepsilon, u_\varepsilon)$ and $(\hat{x}_{\varepsilon}, u_{\varepsilon})$ be solutions of system \eqref{e.gradflow} with initial condition $(x_0, u_0)$ and energies $E$ and $\hat{E}$, respectively (notice that $u_{\varepsilon}$ is always the same). 

We can estimate
\begin{displaymath}
\begin{split}
| \hat{x}^{N}_{\varepsilon} (t) & - \hat{x}_{\varepsilon}(t) | = \frac{1}{\varepsilon} \left| \int_{0}^{t} \nabla_x \hat E_{N} (\tau, \hat{x}_{\varepsilon}^{N}(\tau)) - \nabla_x \hat E (\tau, \hat x_{\varepsilon}(\tau))\, \di\tau \right|
\\
&
 \leq \frac{1}{\varepsilon} \left| \int_{0}^{t} \nabla_x \hat E_{N} (\tau, \hat{x}_{\varepsilon}^{N}(\tau)) - \nabla_x \hat E_{N} (\tau, \hat x_{\varepsilon}(\tau))\, \di\tau \right|
\\
&
\qquad + \frac{1}{\varepsilon} \left| \int_{0}^{t} \nabla_x \hat E_{N} (\tau, \hat{x}_{\varepsilon} (\tau)) - \nabla_x \hat E_{\delta} (\tau, \hat x_{\varepsilon}(\tau))\, \di\tau \right| + \frac{1}{\varepsilon} \left| \int_{0}^{t} \nabla_x \hat E_{\delta} (\tau, \hat{x}_{\varepsilon} (\tau)) - \nabla_x \hat E (\tau, \hat x_{\varepsilon}(\tau))\, \di\tau \right|
\\
&
\leq \frac{M}{\varepsilon} \int_{0}^{t} |\hat{x}^{N}_{\varepsilon} (\tau) - \hat{x}_{\varepsilon}(\tau)| \, \di\tau + \varepsilon e^{-\frac{M}{\varepsilon}} \,.
\end{split}
\end{displaymath}
Hence, by Gronwall we deduce
\begin{displaymath}
\max_{t \in [0,T]}  |\hat{x}_{\varepsilon}^{N} (t) - \hat{x}_{\varepsilon} (t)| \leq \varepsilon \,.
\end{displaymath}

It remains to estimate $| \hat{x}_{\varepsilon}(t) - x_{\varepsilon}(t)|$. Denoting with $x^{\bar\eta}_{\varepsilon}(t)$ a point in~$\supp(\bar\eta)$ such that $|x_{\varepsilon}(t) - x^{\bar\eta}_{\varepsilon}(t)| = \mathrm{dist}(x_{\varepsilon}(t), \supp(\bar\eta))$, we estimate
\begin{equation}\label{e.bound}
\begin{split}
| \hat{x}_{\varepsilon}(t) & - x_{\varepsilon}(t)|  = \frac{1}{\varepsilon} \left| \int_{0}^{t} \nabla_{x} \hat{E} (\tau, \hat{x}_{\varepsilon}(\tau)) - \nabla_{x} E (\tau, x_{\varepsilon}(\tau)) \, \di \tau \right|
\\
&
\leq \frac{1}{\varepsilon} \left| \int_{0}^{t} \nabla_{x} \hat{E} (\tau, \hat{x}_{\varepsilon}(\tau)) - \nabla_{x} \hat E (\tau, x_{\varepsilon}(\tau)) \, \di \tau \right| + \frac{1}{\varepsilon} \left| \int_{0}^{t} \nabla_{x} \hat{E} (\tau, x_{\varepsilon}(\tau)) - \nabla_{x} E (\tau, x_{\varepsilon}(\tau)) \, \di \tau \right|
\\
&
\leq \frac{1}{\varepsilon}  \int_{0}^{t} |\nabla_{x} \hat{E} (\tau, \hat{x}_{\varepsilon}(\tau)) - \nabla_{x} \hat E (\tau, x_{\varepsilon}(\tau))| \, \di \tau + \frac{1}{\varepsilon}  \int_{0}^{t} |\nabla_{x} \hat{E} (\tau, x_{\varepsilon}(\tau)) - \nabla_{x} \hat E (\tau, x_{\varepsilon}^{\bar \eta}(\tau))| \, \di \tau 
\\
&
\qquad  + \frac{1}{\varepsilon} \int_{0}^{t} |\nabla_{x} E (\tau, x_{\varepsilon}^{\bar \eta}(\tau))) - \nabla_{x}  E (\tau, x_{\varepsilon}(\tau))| \, \di \tau \,.
\end{split}
\end{equation}
By the Lipschitz continuity of $E$ and $\hat E$ we deduce
\begin{displaymath}
| \hat{x}_{\varepsilon}(t) - x_{\varepsilon}(t)| \leq \frac{M}{\varepsilon} \int_{0}^{t} |\hat{x}_{\varepsilon}(\tau) - x_{\varepsilon}(\tau)|\,\di \tau + \frac{2 M}{\varepsilon} \|{\rm dist}(x_{\varepsilon}(\cdot),{\rm supp}(\bar \eta))\|_{L^{1}(0,T)}\,,
\end{displaymath}
which, again by Gronwall inequality, implies~\eqref{quantest}.
\end{proof}

\section{One-dimensional model of thin nonlinear-elastic rod}\label{s.example}

In this section we would like to provide a concerte example of learning of a simple mechanical energy and its quasi-static evolutions. In the following we consider the  one-dimensional model of thin nonlinear-elastic rod
\begin{equation}\label{e.ner}
E_0(y) \coloneq \int_{0}^{1} a( y'(r))\,\di r\,,
\end{equation}
where $y \colon [0,1] \to \R$ is a displacement map and $E_0$ is a suitable nonlinear elastic energy,
fully determined by a potential function $a \colon \R \to \R$. We consider here quasi-static evolutions $y \colon [0,T]\times [0,1] \to \R$ of critical points $y(t)$ of $E_0$ subjected to time-dependent boundary conditions $y(t,i)=f_i(t)$, for $i \in\{0,1\}$, their mean-field descriptions, and the learning of the potential function $a$. As the theory we developed in this paper applies to finite dimensional states, see Section~\ref{s.ecp}, we approach the problem in a space-discrete setting.
 
\subsection{Analytic results of the space-discrete model}
For every~$y\in\R^{d+2}$, we define the discrete energy function
\begin{equation}\label{e.eaux}
E_{0}(y ) \coloneq \sum_{i=1}^{d+1} a ( y_{i+1} - y_{i} )\,,
\end{equation}
which represents the space-discrete version of \eqref{e.ner} (the $y_i$ are intended to be discrete evaluation of a displacement map $y(\xi_i)$ at suitable space nodes $0=\xi_{1}< \ldots < \xi_{d+2}=1$),
where $a\colon \R\to [0,+\infty)$ satisfies the following assumptions:
\begin{itemize}
\item[(a1)] $a\in C^{1,1}(\R)$;

\item[(a2)] there exists a positive constant~$C$ such that for every $s\in\R$:
\begin{displaymath}
|a'(s)|\leq C(|a(s)|+1)\,;
\end{displaymath}

\item[(a3)] there exist $p\in(1,+\infty)$ and two positive constants $C_{1},\, C_{2}$ such that for every~$s\in\R$
\begin{displaymath}
C_{1} |s|^{p}- C_{2} \leq a(s)\,.
\end{displaymath}

\end{itemize}

In order to introduce an explicit time dependence, we fix $f_{1}, f_{2}\in C^{1,1}( [0,T])$ and restrict the admissible trajectories $y\in\R^{d+2}$ to those of the form $(f_{1}(t), y_{1}, \ldots, y_{d}, f_{2}(t))$, so that the energy~$E_{0}$ in~\eqref{e.eaux} can be rewritten as
\begin{equation}\label{e.ena}
E(t,x)\coloneq E_{0}( (f_{1}(t), x_{1}, \ldots, x_{d}, f_{2}(t))) = a(x_{1} - f_{1}(t)) + a( f_{2}(t) - x_{d}) + \sum_{i=1}^{d-1} a(x_{i+1}- x_{i})
\end{equation}
for every $x=(x_{1}, \ldots, x_{d})\in\R^{d}$ and every $t\in[0,T]$.

To simplify the notation, we introduce the \emph{discrete gradient} operator $\mathbf{D}\colon \R^{d+2} \to \R^{d+1}$ defined by
\begin{displaymath}
\mathbf{D} = (\mathbf{D}_{i,j})_{\{i=1,\ldots, d+1\}}^{\{ j=1,\ldots, d+2\}}\,,\qquad 
\mathbf{D} y \coloneq \left( \begin{array}{ccc}
y_{2} - y_{1} \\ \vdots \\ y_{d+2} - y_{d+1}
\end{array}\right) .
\end{displaymath}
It is clear that $\mathbf{D}$ is surjective on~$\R^{d+1}$. Moreover, for~$t\in[0,T]$ we define $e_{t}\colon \R^{d}\to \R^{d+2}$ as
\begin{displaymath}
e_{t}(x) \coloneq (f_{1}(t), x_{1}, \ldots, x_{d}, f_{2}(t)) \qquad\text{for every $x= (x_{1},\ldots, x_{d})\in\R^{d}$.}
\end{displaymath}
With this notation at hand, we can rewrite~$E(t,x)$ as
\begin{displaymath}
E(t,x) = \sum_{j=1}^{d+1} a\Big( \big( \mathbf{D} \, e_{t}(x) \big)_{j} \Big)\qquad\text{for every $x \in\R^{d}$}.
\end{displaymath}
We also write explicitly the expression of~$\nabla_{x}E(t,x)$:
\begin{equation}\label{e.gradient}
\big ( \nabla_{x} E(t,x) \big)_{i} = \partial_{i}\Bigg(\sum_{j=1}^{d+1} a\Big(\big(\mathbf{D}( f_{1}(t), x_{1}, \ldots, x_{d}, f_{2}(t)) \big)_{j} \Big) \Bigg) =  \overline{\mathbf{D}}^{T} \left( \begin{array}{ccc}
 a'((\mathbf{D} \, e_{t}(x)_{1} )\\
 \vdots\\
  a'((\mathbf{D} \, e_{t}(x))_{d+1} )
\end{array}\right)_i
\end{equation}
where $\overline{\mathbf{D}} \in\mathbb{M}^{(d+1)\times d}$ is the minor $\overline{\mathbf{D}} \coloneq (\mathbf{D}_{i,j})_{\{i=1,\ldots, d+1\}}^{\{j=2,\ldots, d+1\}}$. We notice that $\ker(\overline{\mathbf{D}}^{T}) = \left\langle (1, \ldots, 1)\right\rangle \subseteq \R^{d+1}$.

We finally point out that for the control parameter~$u$ in~\eqref{e.gradflow} we will not consider any dynamics, that is, we fix $f\equiv0$ in~\eqref{e.gf3} and~$u(t)$ will be constantly equal to its initial value~$u_{0}\coloneq \nabla_{x} E(0,x_{0})$.

In order to apply the abstract scheme developed in the previous sections, we have first to check that the energy function~\eqref{e.ena} satisfies properties~(E1)-(E4). In the following lemma we rigorously show that~$E$ fulfills~(E1)-(E3), while in Remark~\ref{r.E4} we discuss the generic validity of condition~(E4).

\begin{lemma}\label{l.ena}
The energy function $E\colon [0,T]\times\R^d \to \R$ defined in~\eqref{e.ena} satisfies conditions~(E1)-(E3).
\end{lemma}

\begin{proof}
Property~(E1) is clearly satisfied in view of~(a1) and of the regularity of~$f_{1}$ and~$f_{2}$.

By~(a2) and by regularity of~$f_{1}$ and~$f_{2}$, we have that
\begin{displaymath}
\begin{split}
|\partial_{t} E(t,x)| & = |a'(x_{2} - f_{1}(t)) f_{1}'(t) - a'(f_{2}(t) - x_{d}) f'_{2}(t)| \leq C(|a'(x_{2} - f_{1}(t)) | +  |a'(f_{2}(t) - x_{d})|) \\
&\leq C (|a(x_{2} - f_{1}(t)) | +  |a(f_{2}(t) - x_{d})| +2) \,.
\end{split}
\end{displaymath}
Since $a\geq 0$, we can continue the previous estimate with
\begin{displaymath}
\begin{split}
|\partial_{t} E(t,x)| & \leq C\Big( a(x_{2} - f_{1}(t))  +  a(f_{2}(t) - x_{d}) + \sum_{i=1}^{d-1} a(x_{i+1} - x_{i}) +1 \Big) = {C} ( E(t,x) + 1),
\end{split}
\end{displaymath}
for some positive constant~$C$ independent of~$t$ and~$x$. Thus,~(E2) holds.

As for~(E3), by~(a3) we have that
\begin{equation}\label{e.e3}
E(t,x) \geq C\Big( |x_{1} - f_{1}(t)|^{p} + |f_{2}(t) - x_{d}|^{p} + \sum_{i=1}^{d-1} |x_{i+1}- x_{i}|^{p}\Big)\,. 
\end{equation}
By convexity of the function $s\mapsto |s|^{p}$ for $p>1$ and by Young inequality we get
\begin{displaymath}
|x_{1} - f_{1}(t)|^{p} - |x_{1}|^{p} \geq -p|x_{1}|^{p-2}x_{1} f_{1}(t) \geq -p|x_{1}|^{p-1}|f_{1}(t)| \geq -C|f_{1}(t)|^{p} - \tfrac{1}{2}|x_{1}|^{p}\,.
\end{displaymath}
A similar inequality holds for $|f_{2}(t) - x_{d}|^{p}$. Hence,~\eqref{e.e3} becomes
\begin{displaymath}
E(t,x) \geq \tfrac{1}{2}\Big( |x_{1}|^{p} + |x_{d}|^{p} + \sum_{i=1}^{d-1} |x_{i+1} - x_{i} |^{p}\Big) - C( |f_{1}(t)|^{p} + |f_{2}(t)|^{p})\,.
\end{displaymath}
At this point, it is easy to see that there exists a positive constant~$c$ such that
\begin{displaymath}
|x_{1}|^{p} + |x_{d}|^{p} + \sum_{i=1}^{d-1} |x_{i+1} - x_{i} |^{p} \geq c|x_{j}|^{p} \qquad\text{for every $j\in\{1,\ldots, d\}$}.
\end{displaymath}
In fact, for $j=1$ and $j=d$ the inequality is obvious. For $1<j<d$ we notice that
\begin{displaymath}
\begin{split}
|x_{1}|^{p} + |x_{d}|^{p} & + \sum_{i=1}^{d-1} |x_{i+1} - x_{i} |^{p} \geq |x_{1}|^{p}  + \sum_{i=1}^{j-1} |x_{i+1} - x_{i} |^{p} \geq \tfrac{1}{(d+1)^{p-1}} \Big| |x_{1}| +  \sum_{i=1}^{j-1} |x_{i+1} - x_{i} |\Big|^{p}\\
& \geq \tfrac{1}{(d+1)^{p-1}} |x_{j}|^{p}\,. 
\end{split}
\end{displaymath}
This concludes the proof of~(E3).
\end{proof}

\begin{remark}\label{r.E4}
Let us comment on the validity of property~(E4). In the framework described above, as we are assuming~$f=0$ in~\eqref{e.gradflow}, it is actually enough to have that
\begin{equation}\label{e.E4}
\text{$C(t,u)$ contains only isolated points for every $u\in \R^{d}$ with~$u=\nabla_{x} E(0,x_{0})$, $x_{0}\in\R^{d}$}.
\end{equation}
The validity of~\eqref{e.E4} is related to the so called \emph{transversality conditions} (see, e.g.,~\cite{MR3324377, MR2318261}). Indeed, in~\cite{MR3324377}  the authors first show that the transversality conditions for an energy~$E$ imply that the set of critical points~$C(t)\coloneq\{ x\in\R^{d}:\, \nabla_{x}E(t, x)=0\}$ contains only isolated points. In~\cite[Theorem~1.3]{MR3324377} (see also~\cite[Corollary~3.7]{MR3324377}) they also prove the genericity of the transversality conditions. In our setting, the latter result states that, assuming~$a\in C^{4}(\R)$ and~$f_{1}, f_{2}\in C^{4}(0,T)$, there exists a set $N\subseteq \R^{d}\times \mathbb{R}^{d\times d}$ of Lebesgue measure zero such that for every~$(v,\mathbf{A})\in (\R^{d}\times\mathbb{R}^{d\times d}) \setminus N$ the function~$\widetilde{E}\colon [0,T]\times\R^{d} \to \R$ defined as $\widetilde{E}(t,x) \coloneq E(t,x) + v\cdot x + \mathbf{A}x\cdot x$ satisfies the transversality conditions, so that the set~$\{x\in\R^{d}:\, \nabla_{x}\widetilde{E}(t,x)= 0\}$ contains only isolated points.

In the present work we have been considering the energy $\overline{E}(t,x,u)\coloneq E(t,x)-u{\,\cdot\,} x$, which already modifies $E$ by the additive linear term~$-u\cdot x$, where $u\coloneq \nabla_{x} E(0,x_{0})$,~$x_{0}$ being the initial condition of~$x(\cdot)$. Hence, assuming that the distribution of~$x_{0}$ has a non-degenerate (say, for instance, of positive Lebesgue measure) support and~$\nabla_{x} E(0,\cdot)$ is non-degenerate, we deduce that condition~\eqref{e.E4} is in general satisfied, up to a further generic quadratic perturbation of~$\overline{E}$. 
\end{remark}

In view of Lemma~\ref{l.ena} and of Remark~\ref{r.E4}, from now on we will assume that~$E$ in~\eqref{e.ena} satisfies properties~(E1)-(E4).
Hence, we can apply the theoretical results of Section~\ref{s.reconstruction} to our energy~$E$ in~\eqref{e.ena}. Since~$E$ is completely determined by the monovariate function~$a$, we slightly modify the notation of Section~\ref{s.reconstruction} to this new setting, rewriting the identification problem in terms of~$a$. In fact, while approximating a high-dimensional  (multivariate) function $E$ directly incurs in the curse of dimensionality \cite{novakwoz09} in general, our model is actually parametrized by a lower dimensional function $a$, making the learnability/approximation problem computationally tractable. Of course, this imposes a further modeling constraint. Accordingly, for fixed $M,R>0$, instead of the space~$X_{M,R}$ in~\eqref{e.XMR}, we consider
\begin{displaymath}
A_{M,R} \coloneq \{ \hat{a}\in W^{2,\infty}_{loc}(\R):\, \| \hat{a} \|_{W^{2,\infty}( \I_{R})}\leq M\}\,,
\end{displaymath}
where $\I_{R}\coloneq [-R, R] \subset \R$ is a suitable interval. The choice of~$M,R>0$ can be performed similarly to Section~\ref{s.reconstruction}, simply noticing that the boundedness of~$x$ implies the boundedness of~$\mathbf{D}\, e_{t}(x)$, with a bound that depends on the boundary data~$f_{1}(t)$ and~$f_{2}(t)$. 

We consider a sequence~$(V_{N})_{N\in\mathbb{N}}$ of finite dimensional subspaces of~$A_{M,R}$, for which the uniform approximation property of Definition~\ref{d.UA} reads now as follows.

\begin{definition}\label{d.UAena}
Let $\eta\in \mathcal{M}_{b}^{+}([0,T]\times\R^{d}\times\R^{d})$, let~$\bar{\eta}$ be as in~\eqref{e.bareta}, let~$(V_{N})_{N\in\mathbb{N}}$ be a sequence of finite dimensional subspaces of~$A_{M,R}$, let $\tilde \eta \in \mathcal{M}_{b}^{+}(\R^{d + 1})$ be defined by
\begin{equation}\label{e.tildeeta}
\tilde{\eta} (B) \coloneq (\mathbf{D})_{\#} (e_{t})_{\#} \bar \eta ([0,T]\times B ) \qquad\text{for every Borel subset $B\subseteq\R^{d+1}$}\,,
\end{equation}
and let 
\begin{equation}\label{e.Oeta}
O_{\eta}\coloneq \bigcup_{i=1}^{d+1} \supp \big( (\pi_{i})_{\#} \tilde{\eta}\big),
\end{equation}
where~$\pi_{i}\colon \R^{d+1}\to \R$ stands for the projection on the~$i$-th component. We say that~$(V_{N})_{N\in\mathbb{N}}$ has the \emph{uniform approximation property} with respect to~$\eta$ if for every $\hat{a}\in A_{M,R}$ there exists a sequence $\hat{a}_{N}\in V_{N}$ such that $\hat{a}_{N}\to \hat{a}$ in $W^{1,\infty}( O_{\eta})$ as $N\to\infty$.
\end{definition}

We now rewrite the functionals~\eqref{e.10}-\eqref{e.12} in terms of~$\hat{a}, a\in A_{M,R}$ making use of formula~\eqref{e.gradient}. As already mentioned, here we consider time independent controls~$u = \nabla_{x} E(0,x_{0})$. Hence, given a distribution~$\mu_{0}\in\sP_{c}(\R^{d})$ of initial conditions~$x_{0}$, the corresponding distribution of~$(x_{0}, u_{0})$ reads as $\eta_{0}\coloneq (\mathrm{id}\times \nabla_{x}E(0,\cdot))_{\#} \mu_0 \in \mathcal{P}_{c}(\R^{d}\times\R^{d})$. For every $N\in\mathbb{N}$ and every $\varepsilon>0$ we fix~$N$ pairs~$(x^{i}_{0}, u^{i}_{0})\in\supp(\eta_{0})$ distributed according to~$\eta_{0}$ and we consider the corresponding solutions~$(x^{i}_{\varepsilon}, u^{i}_{\varepsilon})\colon[0,T]\to \R^{d}\times\R^{d}$ of the ODE system~\eqref{e.gradflow}. Given the empirical measure $\eta^{N}_{\varepsilon,t} \coloneq \tfrac{1}{N}\sum_{i=1}^{N}\delta_{(x^{i}_{\varepsilon}(t), u^{i}_{\varepsilon}(t))}$, for every~$\hat{a}\in V_{N}$ we define
\begin{equation}\label{e.210} 
\J_{N, \varepsilon}(\hat{a}) \coloneq \frac{1}{T} \int_{0}^{T} \int_{\R^{d}\times\R^{d}} \left |  \overline{\mathbf{D}}^{T} \left( \begin{array}{ccc}
 \hat a'(\mathbf{D} \, e_{t}(x)_{1} ) -  a'(\mathbf{D} \, e_{t}(x)_{1} )\\
 \vdots\\
  \hat a'(\mathbf{D} \, e_{t}(x))_{d+1} ) -   a'(\mathbf{D} \, e_{t}(x))_{d+1} )
\end{array}\right) 
\right |^{2} \di\eta^{N}_{\varepsilon,t}(x,u) \,\di t \,.
\end{equation}
In the limit as $N\to \infty$ the sequence $\eta^{N}_{\varepsilon}$ converges uniformly with respect to~$W_{1}$ to a curve $\eta_{\varepsilon}\in C([0,T];\mathcal{P}(\R^{d}\times\R^{d}))$. Therefore, for every~$\hat{a}\in A_{M,R}$ we set
\begin{displaymath}\label{e.211}
\J_{\varepsilon}(\hat{a}) \coloneq \frac{1}{T}\int_{0}^{T}\int_{\R^{d}\times\R^{d}} \left |  \overline{\mathbf{D}}^{T} \left( \begin{array}{ccc}
 \hat a'(\mathbf{D} \, e_{t}(x)_{1} ) -  a'(\mathbf{D} \, e_{t}(x)_{1} )\\
 \vdots\\
  \hat a'(\mathbf{D} \, e_{t}(x))_{d+1} ) -   a'( \mathbf{D} \, e_{t}(x))_{d+1} )
\end{array}\right) 
\right |^{2} \di\eta_{\varepsilon,t}(x,u) \,\di t \,.
\end{displaymath}
For a Borel family $\{\eta_{t}: t\in[0,T]\}\subseteq\mathcal{P}(\R^{d}\times\R^{d})$ and for every~$\hat{a}\in A_{M,R}$ we set
\begin{displaymath} \label{e.212}
\J_{\eta}(\hat{a}) \coloneq \frac{1}{T}\int_{0}^{T}\int_{\R^{d}\times\R^{d}}  \left |  \overline{\mathbf{D}}^{T} \left( \begin{array}{ccc}
 \hat a'(\mathbf{D} \, e_{t}(x)_{1} ) \\
 \vdots\\
  \hat a'(\mathbf{D} \, e_{t}(x))_{d+1} ) 
  \end{array}\right) -u \right |^{2} \di\eta_{t}(x,u) \,\di t \,.
\end{displaymath}
Also in this case, we notice that if~$\{\eta_{t}:t\in[0,T]\}\subseteq\mathcal{P}(\R^{d}\times\R^{d})$ is such that
\begin{equation}\label{e.concentration2}
\int_{0}^{T} \int_{\R^{d}\times \R^{d}}  \left |  \overline{\mathbf{D}}^{T} \left( \begin{array}{ccc}
  a'(\mathbf{D} \, e_{t}(x)_{1} ) \\
 \vdots\\
   a'(\mathbf{D} \, e_{t}(x))_{d+1} ) 
  \end{array}\right) -u \right |^{2} \di\eta_{t}(x,u) \,\di t = 0\,,
\end{equation}
then we can also express~$\J_{\eta}$ in the equivalent form
\begin{displaymath}
\mathcal{J}_{\eta}(\hat{a})\coloneq \int_{0}^{T}\int_{\R^{d}\times\R^{d}} \left | \overline{\mathbf{D}}^{T}  \left(\begin{array}{ccc}
\hat a'(\mathbf{D} \, e_{t}(x)_{1} ) - a'( \mathbf{D} \, e_{t}(x)_{1} )\\
\vdots\\
\hat a'(\mathbf{D} \, e_{t}(x)_{d+1}) - a'( \mathbf{D} \, e_{t}(x)_{d+1})
\end{array} \right) \right |^{2} \di \eta_{t}(x,u)\,\di t\,.
\end{displaymath}

We now adapt the main results of Section~\ref{s.reconstruction}, namely, Proposition~\ref{p.7} and Theorem~\ref{p.10}.

\begin{proposition}\label{p.7.2}
Let $\{\eta_{t}: t\in[0,T]\}$ be a Borel family in~$\mathcal{P}(\R^{d}\times\R^{d})$ with uniformly compact support and such that~\eqref{e.concentration2} is satisfied. Let $\eta \coloneq \eta_{t}\otimes\mathcal{L}^{1}_{|[0,T]}$,~$\tilde{\eta}\in \mathcal{M}^{+}_{b}([0,T]\times\R^{d})$ and~$O_{\eta}$ be as in~\eqref{e.tildeeta}-\eqref{e.Oeta}, and let $\hat{a}_{1}, \hat{a}_{2}\in A_{M,R}$. Assume that~$O_{\eta}\subseteq \I_{R}$. Then,
\begin{displaymath}
| \J_{N, \varepsilon}(\hat{a}_{1}) - \J_{\eta}(\hat{a}_{2})| \leq D_{1}\,  W_{1}( \eta^{N}_{\varepsilon, t} \otimes \mathcal{L}^{1}_{|[0,T]}, \eta_{t}\otimes\mathcal{L}^{1}_{|[0,T]}) + D_{2} \|\hat{a}_{1} - \hat{a}_{2}\|_{W^{1,\infty}(O_{\eta})} \,,
\end{displaymath}
for some positive constants $D_{1}, D_{2}$ depending on~$\|\mathbf{D}\|$,~$d$,~$M$,~$T$,~$f_{1}$, and~$f_{2}$.
\end{proposition}

\begin{proof}
For $\hat{a}_{1}, \hat{a}_{2} \in A_{M,R}$ we have
\begin{displaymath}
| \J_{N, \varepsilon}(\hat{a}_{1}) - \J_{\eta}(\hat{a}_{2})| \leq  | \J_{N, \varepsilon}(\hat{a}_{1}) -  \J_{\eta}(\hat{a}_{1}) | + | \J_{\eta}(\hat{a}_{1}) - \J_{\eta}(\hat{a}_{2})| =: I_{1}+ I_{2}\,.
\end{displaymath}

 Following the lines of the proof of Proposition~\ref{p.7} and using~\eqref{e.gradient}, we get
 \begin{displaymath}
 I_{1} \leq   D_{1}\,  W_{1}( \eta^{N}_{\varepsilon, t} \otimes \mathcal{L}^{1}_{|[0,T]}, \eta_{t}\otimes\mathcal{L}^{1}_{|[0,T]})\,,
 \end{displaymath}
 where~$D_{1}= D_{1}(\|\mathbf{D}\|, d, M, T )>0$.
 
 As for~$I_{2}$ we write
 \begin{displaymath}
 \begin{split}
 I_{2} & \leq  \frac{1}{T}\int_{0}^{T}\int_{\R^{d}\times\R^{d}}  \left | \overline{\mathbf{D}}^{T}  \left(\begin{array}{ccc}
\hat a_{1}'(\mathbf{D} \, e_{t}(x)_{1} ) - \hat a_{2}'( \mathbf{D} \, e_{t}(x)_{1} )\\
\vdots\\
\hat a_{1}'(\mathbf{D} \, e_{t}(x)_{d+1}) - \hat a_{2}'( \mathbf{D} \, e_{t}(x)_{d+1})
\end{array} \right) \right | \\
&\qquad \qquad\qquad \left(  \left | \overline{\mathbf{D}}^{T}  \left(\begin{array}{ccc}
\hat a_{1}'(\mathbf{D} \, e_{t}(x)_{1} ) - a'( \mathbf{D} \, e_{t}(x)_{1} )\\
\vdots\\
\hat a_{1}'(\mathbf{D} \, e_{t}(x)_{d+1}) - a'( \mathbf{D} \, e_{t}(x)_{d+1})
\end{array} \right) \right | \right.\\
& \qquad\qquad\qquad + \left.  \left | \overline{\mathbf{D}}^{T}  \left(\begin{array}{ccc}
\hat a_{2}'(\mathbf{D} \, e_{t}(x)_{1} ) - a'( \mathbf{D} \, e_{t}(x)_{1} )\\
\vdots\\
\hat a_{2}'(\mathbf{D} \, e_{t}(x)_{d+1}) - a'( \mathbf{D} \, e_{t}(x)_{d+1})
\end{array} \right) \right| \right) \di  \eta_{t} (x, u)\,\di t \\
&=  \frac{1}{T}\int_{\R^{d+1}}  \left | \overline{\mathbf{D}}^{T}  \left(\begin{array}{ccc}
\hat a_{1}'(y_{1} ) - \hat a_{2}'( y_{1} )\\
\vdots\\
\hat a_{1}'(y_{d+1}) -\hat a_{2}'(y_{d+1})
\end{array} \right) \right | \left(  \left | \overline{\mathbf{D}}^{T}  \left(\begin{array}{ccc}
\hat a_{1}'(y_{1} ) - a'(y_{1} )\\
\vdots\\
\hat a_{1}'(y_{d+1}) - a'( y_{d+1})
\end{array} \right) \right | \right.\\
& \qquad\qquad\qquad + \left.  \left | \overline{\mathbf{D}}^{T}  \left(\begin{array}{ccc}
\hat a_{2}'(y_{1} ) - a'(y_{1} )\\
\vdots\\
\hat a_{2}'(y_{d+1}) - a'( y_{d+1})
\end{array} \right) \right| \right) \di (\pi_{1}, \ldots, \pi_{d+1})_{\#} \tilde\eta (y) \\
& \leq D_{2} \|\hat{a}_{1} - \hat{a}_{2}\|_{W^{1,\infty}(O_{\eta})} \,,
 \end{split}
 \end{displaymath}
 where~$D_{2}= D_{2}(\|\mathbf{D}\|, d, M, T, f_{1}, f_{2})>0$. This concludes the proof of the proposition.
 \end{proof}
 
 As an immediate consequence of Proposition~\ref{p.7.2} we deduce the following.
 
 \begin{corollary}\label{c.1.2}
Let $\varepsilon_{N}$ be a null sequence. Assume that the sequence~$\eta^{N}_{\varepsilon_{N}, t} \otimes \mathcal{L}^{1}_{| [0,T]}$ converges narrowly (and hence with respect to~$W_{1}$) to~$\eta \coloneq \eta_{t}\otimes \mathcal{L}^{1}_{| [0,T]}$ for a Borel family~$\{\eta_{t}: t\in[0,T]\} \subseteq \mathcal{P}(\R^{d}\times\R^{d})$. Let~$(V_{N})_{N \in \mathbb N}$ be a sequence of closed subspaces of~$A_{M,R}$ having the uniform approximation property of Definition~\ref{d.UAena} with respect to~$\eta$. Then, the sequence of functionals~$\J_{N,\varepsilon_{N}}$ defined on~$V_{N}$ $\Gamma$-converges to~$\J_{\eta}$ in~$A_{M,R}$ with respect to the topology of~$W^{1,\infty}(O_{\eta})$.
\end{corollary}

\begin{proof}
As Corollary~\ref{c.1}, this result follows from the application of Proposition~\ref{p.7.2}.
\end{proof}

We now adapt Theorem~\ref{p.10}.

\begin{theorem}\label{p.10.2}
Let $\delta>0$,~$\varepsilon_{N}>0$ and~$\eta^{N}_{\varepsilon_{N},t}\in\mathcal{P}(\R^{d}\times\R^{d})$ be as in Proposition~\ref{p.6}. Assume that there exists a Borel family~$\{\eta^{\delta}_{t}\}_{t\in[0,T]}\subseteq\mathcal{P}(\R^{d}\times\R^{d})$ such that $\eta^{N}_{\varepsilon_{N},t}\otimes\mathcal{L}^{1}_{|[0,T]}$ converges narrowly to $\eta^{\delta} \coloneq\eta^{\delta}_{t}\otimes\mathcal{L}^{1}_{|[0,T]}$. Suppose that~$(V_{N})_{N\in\mathbb{N}}$ satisfies the uniform approximation property of Definition~\ref{d.UAena} with respect to~$\eta^{\delta}$ and denote with~$\hat{a}_{N}\in V_{N}$ a solution of
\begin{equation}\label{e.optN}
\min_{\hat{a}\in V_{N}}\, \J_{N}(\hat{a})\,.
\end{equation}
Then,~$\hat{a}_{N}$ converges, up to a subsequence, to some~$\hat{a}_{\delta}\in A_{M,R}$ satisfying ~$\J_{\eta^{\delta}}(\hat{a}_{\delta}) \leq C \delta$ for a positive constant~$C$ independent of~$\delta$.

Moreover, there exist a further Borel family~$\{\eta_{t}:t\in [0,T]\} \subseteq \mathcal{P}(\R^{d}\times\R^{d})$ and a further~$\hat{a}\in A_{M,R}$ such that, up to a subsequence,~$\eta^{\delta}$ converges narrowly to $\eta \coloneq \eta_{t}\otimes \mathcal{L}^{1}_{|[0,T]}$,~$\hat{a}_{\delta}\to \hat{a}$ in~$W^{1,\infty}(\I_{R})$ as $\delta \to 0$, and~$\J_{\eta}(\hat{a}) =0$.
\end{theorem}

\begin{proof}
It is enough to follow step by step the proof of Theorem~\ref{p.10} taking into account that the results of Proposition~\ref{p.6} still hold in the present framework.
\end{proof}

\begin{remark}
Let us comment on the results of Theorem~\ref{p.10.2}. Since $\ker(\overline{\mathbf{D}}^{T})= \left\langle (1,\ldots, 1)\right\rangle$, we can write~$\J_{\eta}( \hat a)$ as
\begin{displaymath}
\begin{split}
\J_{\eta}(\hat a) = \int_{0}^{T} \int_{\R^{d} \times \R^{d} }  & \left | \overline{\mathbf{D}}^{T}  \left[  \left(\begin{array}{ccc}
\hat a'(\mathbf{D} \, e_{t}(x)_{1} ) - a'( \mathbf{D} \, e_{t}(x)_{1} )\\
\vdots\\
\hat a'(\mathbf{D} \, e_{t}(x)_{d+1}) - a'( \mathbf{D} \, e_{t}(x)_{d+1}) 
\end{array}  \right)  \right.  \right.  \\
& \qquad\qquad\qquad\qquad \left. \left. - \frac{1}{d+1} \left( \begin{array}{ccc}
 \sum_{i=1}^{d+1}( \hat{a}' - a') (\mathbf{D}\, e_{t} (x)_{i} )\\
 \vdots\\
 \sum_{i=1}^{d+1}( \hat{a}' - a') (\mathbf{D}\, e_{t} (x)_{i})
 \end{array}\right) \right] \right |^{2} \di \eta_{t}(x,u)\,\di t\,.
 \end{split}
\end{displaymath}
By definition of~$\tilde{\eta}$ in~\eqref{e.tildeeta}, we can further recast the expression as  
\begin{equation}\label{e.estimateJeta}
\J_{\eta}(\hat{a}) = \int_{\R^{d+1}} \left | \overline{\mathbf{D}}^{T}  \left[ \left(\begin{array}{ccc}
\hat a'(z_{1} ) - a'( z_{1} )\\
\vdots\\
\hat a'(z_{d+1}) - a'( z_{d+1})
\end{array}  \right) - \frac{1}{d+1} \left( \begin{array}{ccc}
 \sum_{i=1}^{d+1}( \hat{a}' - a') (z_{i} )\\
 \vdots\\
 \sum_{i=1}^{d+1}( \hat{a}' - a') (z_{i})
 \end{array}\right) \right ] \right |^{2} \di \tilde \eta ( z )  \,.
\end{equation}
Again since $\ker(\overline{\mathbf{D}}^{T}) = \left\langle (1,\ldots, 1)\right\rangle$, we deduce that there exists a positive constant~$c>0$ such that $|\overline{\mathbf{D}}^{T} w| \geq c|w|$ for every~$w\in\R^{d+1}$ with $\sum_{i=1}^{d+1} w_{i}=0$. Therefore, we can estimate~\eqref{e.estimateJeta} from below with
\begin{equation}\label{e.estimateeta}
\J_{\eta}(\hat{a}) \geq c \int_{\R^{d+1}} \left | \left(\begin{array}{ccc}
\hat a'(z_{1} ) - a'( z_{1} )\\
\vdots\\
\hat a'(z_{d+1}) - a'( z_{d+1})
\end{array}  \right) - \frac{1}{d+1} \left( \begin{array}{ccc}
 \sum_{i=1}^{d+1}( \hat{a}' - a') (z_{i} )\\
 \vdots\\
 \sum_{i=1}^{d+1}( \hat{a}' - a') (z_{i})
 \end{array}\right) \right |^{2} \di \tilde \eta ( z )  \,.
\end{equation}

If we assume that~$\hat{a}\in A_{M,R}$ is a minimizer of~$\J_{\eta}$, i.e.,~$\J_{\eta}(\hat{a}) =0$, inequality~\eqref{e.estimateeta} implies that for~$\tilde{\eta}$-a.e.~$z\in\R^{d+1}$, for every~$i=1,\ldots, d+1$ it holds
\begin{equation}\label{e.consequence2}
\hat{a}'(z_{i}) - a'(z_{i}) = \frac{1}{d+1} \sum_{j=1}^{d+1} \hat{a}'(z_{j}) - a'(z_{j})\,.
\end{equation}
If, for instance, the support of~$\tilde{\eta}$ is connected, we deduce from~\eqref{e.consequence2} that the function~$\hat{a}$, constructed following the approach of Theorem~\ref{p.10.2}, satisfies 
\begin{equation}\label{e.consequence}
\hat{a}'(s) - a'(s) = \rm{const}, \qquad \text{for every~$s\in O_{\eta}$}\,,
\end{equation}
where~$O_{\eta}$ has been defined in~\eqref{e.Oeta}. Clearly, the constant in~\eqref{e.consequence} remains unknown. However, as we will do in our numerical experiments in Section~\ref{s.numerics}, if we further restrict our attention to the class of functions~$\hat{a}, a \in A_{M,R}$ with~$\hat{a}'(0) = a'(0)=0$ and, in addition, we have~$0\in O_{\eta}$, then $\hat{a}'(s) - a'(s) \equiv 0$ in~$O_{\eta}$, and we obtain exact identification of the gradient of the energy, since clearly the energy~$E$ is always defined up to a constant.

We finally notice that the above considerations can still hold even if we drop the connectedness of~$\supp ( \tilde\eta)$. Assume that $\supp(\tilde \eta) = U_{1} \cup U_{2}$, where~$U_{i}$ are two connected components of~$\supp(\tilde \eta)$. If, for instance, there exists~$s\in O_{\eta}$,~$z^{1}\in U_{1}$, and~$z^{2}\in U_{2}$ such that
\begin{displaymath}
z^{1}= (z^{1}_{1}, \ldots, z^{1}_{i-1}, s, z^{1}_{i+1}, \ldots, z^{1}_{d+1}) \qquad\text{and}\qquad z^{2} = (z^{2}_{1},\ldots, z^{2}_{j-1}, s, z^{2}_{j+1}, \ldots, z^{2}_{d+1})\,,
\end{displaymath} 
then we can still deduce~\eqref{e.consequence}. Indeed, by~\eqref{e.consequence2} we have that
\begin{displaymath}
\hat{a}'(z_{i}) - a'(z_{i}) = \hat{a}'(s) - a'(s) \qquad\text{for~$z\in U_{1}$, for every $i=1, \ldots, d+1$}\,,
\end{displaymath}
and the same holds true in~$U_{2}$. Hence,~\eqref{e.consequence} holds true.
\end{remark}


\subsection{Numerical results}\label{s.numerics}

In this section we present numerical experiments, which show the practical efficiency of the optimization~\eqref{e.optN} in recovering the potential function~$a$ from observation of a finite number of evolutions of critical points. In particular, we highlight some practical issues and the impact of various parameters of the problem on the reconstructions. First of all, we recast the problem in a discrete and numerically efficient implementation. Afterwards, we focus on how the available information -- corresponding to the number of experiments or measurements per experiment -- impacts the quality of reconstructions. Finally, we show that the choice of the constant~$M$ as in~$A_{M,R}$ is in a sense generic, as sufficiently large~$M$ (for other parameters fixed) allows for appropriate reconstructions. Finally, we compare simulations of data-driven evolutions generated by the empirical $\hat a$ with those generated by the true potential $a$. We show the remarkable accordance of the results.

\subsubsection{Efficient numerically implementable formulation}
The following experiments are realized by a common numerical implementation and are applied to the toy mechanical example of Section \ref{s.example}.  As space of competitors, we consider 
\begin{equation}
V_\Lambda \coloneq \big  \{ \hat a \in  A_{M,R} | \   \hat a  \text{ is piecewise quadratic on a given grid }\Lambda \big  \} \,.\end{equation} 
We observe measurements at times $0=t_0< \dots<t_{N_e}=T$ with stepsizes $\Delta_m=(t_{m+1}-t_{m-1})/2$ and gridpoints $p_1<\dots<p_K$ of $\Lambda$ with stepsizes $\tilde \Delta_k=p_{k+1}-p_{k}$.
For an appropriate increasing sequence of grids $\Lambda \coloneq \Lambda(N)$, the corresponding sequence of spaces $V_N \coloneq V_{\Lambda(N)}$ has the uniform approximation property on compact sets.
We consider an initial data distribution~$\mu_0^N$ drawn from a $d$-dimensional normal distributions with uniform standard deviation. For any initial data~$x_0^i$ in the support of~$\mu_0^N$, we solve the system~\eqref{e.gf2} for trajectories~$x_\varepsilon^i$ for fixed $\varepsilon>0$.

As time-discrete approximation of the energy functional $\mathcal J_{N,\varepsilon}$  in \eqref{e.210} we consider 
\begin{equation} \label{equ:discrete_energy}
 \mathcal{ \tilde  J}_{N,\varepsilon}(\hat{a})\coloneq  \frac{1}{N} \sum_{i=1}^N \sum_{m=1}^{N_e} \Delta_m \left |  \overline{\mathbf{D}}^{T}  \left( \begin{array}{ccc}
 \hat a'(\mathbf{D} \, e_{t_m}(x_\varepsilon^i)_{1} ) -  a'(\mathbf{D} \, e_{t_m}(x_\varepsilon^i)_{1} )\\
 \vdots\\
  \hat a'(\mathbf{D} \, e_{t_m}(x_\varepsilon^i)_{d+1} ) -   a'(\mathbf{D} \, e_{t_m}(x_\varepsilon^i)_{d+1} )
\end{array}\right) \right |^2,
\end{equation}
obtained by replacing the integral in time with a sum of point evaluations, which would correspond to assuming solutions, control, and boundary conditions to be piecewise constant in time.  We assume that the arguments of~$\hat a'$ as  in~\eqref{equ:discrete_energy} are distributed according to 
a discrete version of~$\tilde{\eta}$ in Definition~\ref{d.UAena}, which encodes  the available information to recover~$a'$.

As previously stated,~$V_N$ needs to be designed in order to approximate~$A_{M,R}$. In order to  provide additionally a form of numerical stabilization and preconditioning, we choose the grid~$\Lambda$ adaptively with respect to the distribution~$\tilde{\eta}$. In particular we consider denser meshes in regions of the support where~$\tilde{\eta}$ has large density and coarser grids in regions of low density, thereby exploring the entire support of~$\tilde{\eta}$. 
 
As functional defined in  \eqref{equ:discrete_energy} solely depends on  derivatives $\hat a'$, we seek for $\hat a' \in V_N'$ which consists of piecewise linear functions such that $\hat a \in V_N$. In particular we consider the expansion 
\begin{equation}
\hat a'(r)= \sum_{\lambda=1}^{D(N)} \hat{a}_\lambda' \phi_\lambda(r) \,,
\end{equation}
where $\{\phi_\lambda: \lambda=1, \dots, D(N)\}$ is a set of suitable basis functions of $V_N'$, and ${\mathbf a'} \coloneq (\hat a_1',\dots,\hat {a}_{D(N)}')$ denotes the corresponding coefficient vector. From this information it is immediate by integration to identify~$a$  up to additive constants  on connected components of the support of~$\tilde \eta$ . Note, however, that it is not possible to relate additive constants at different connected components of the support of $\tilde \eta$. In case the initial distribution has connected support, since the forward processes are continuous, also the support of $\tilde \eta$ consists of connected components (at most one for each time $t_m$). Thus, one can assume that for a sufficient amount of experiments and connected support of the initial distribution there are only few connected components.

It should be noted that $\mathcal{\tilde J}_{N,\varepsilon}(\hat a)$ with $\hat{ a}'$ written as above can be written as a quadratic functional
\begin{equation} \label{equ:discrete_quadratic_problem}
\begin{split}
\mathcal{\tilde J}_{N,\varepsilon}(\hat a) & = \frac{1}{N}\sum_{i=1}^N \sum_{m=1}^{N_e} \Delta_m \left | \sum_{\lambda=1}^{D(N)} \hat a'_\lambda \overline{\mathbf{D}}^T
 \left( \begin{array}{ccc}
  \phi_\lambda(\mathbf{D} \, e_{t_m}(x_\varepsilon^i)_{1} ) )\\
 \vdots\\
  \hat \phi_\lambda(\mathbf{D} \, e_{t_m}(x_\varepsilon^i)_{d+1} )
  \end{array}\right)
 -  \overline{\mathbf{D}}^T
 \left( \begin{array}{ccc}   
         a'(\mathbf{D} \, e_{t_m}(x_\varepsilon^i)_{1} )\\
 \vdots\\
 a'(\mathbf{D} \, e_{t_m}(x_\varepsilon^i)_{d+1} )
\end{array}\right) \right |^2
\\
&
\vphantom{\int} = \| \mathbf{M}\mathbf{a'}- \mathbf{Y}\|^2_2
\end{split}
\end{equation}
with $\mathbf{a'}=(\hat a'_\lambda)_{\lambda=1}^{D(N)} \in \mathbb{R}^{D(N)}$, $\mathbf{M}\in \mathbf{R}^{dN{N_e} \times D(N)}$ and $\mathbf{Y}\in \mathbb{R}^{dN{N_e}}$. Here the data vector $\mathbf Y$ corresponds to \begin{equation*}
\mathbf Y= \left [ \overline{\mathbf{D}}^T\left( \begin{array}{ccc}   
         a'(\mathbf{D} \, e_{t_m}(x_\varepsilon^i)_{1} )\\
 \vdots\\
 a'(\mathbf{D} \, e_{t_m}(x_\varepsilon^i)_{d+1} )
\end{array}\right) \right ]_{i,m} 
\end{equation*} 
and the tensor $\mathbf{M}$   is defined blockwise via
\begin{equation}
\mathbf{M}_{im,\lambda} \coloneq \overline{\mathbf{D}}^T
 \underbrace{\left( \begin{array}{ccc}
 \hat \phi_\lambda(\mathbf{D} \, e_{t_m}(x_\varepsilon^i)_{1} ) )\\
 \vdots\\
  \hat \phi_\lambda(\mathbf{D} \, e_{t_m}(x_\varepsilon^i)_{d+1} )
  \end{array}\right)}_{B_{i,m}}
\end{equation}
Therefore the assembly of $\mathbf{M}$ corresponds to formulation of the interpolation matrix $B_{i,m}$ and the ``componentwise'' application of $\overline{\mathbf{D}}^T$ and can be done iteratively. In particular, the system is sparse with at most~$4$ entries per row, and thus  the approach can be applied even with a large number of measurements and experiments.

Minimizing the function $\mathcal{\tilde J}_{N,\varepsilon}(\hat a)$ over~$V_N$ is a quadratic optimization problem. However, we require the ansatz to be conform, i.e., ensure the inclusion $V_N \subset A_{M,R}$. This constraint requires  $\|\hat a\|_{W^{2,\infty}(I_R)}\leq M$ for $\hat a  \in V_N$.
To enforce this constraint numerically, note that $a'$ and $a''$ are bounded by the maximal and minimal values of the corresponding coefficient vectors $\mathbf{a'}$ and $\mathbf{a''}$ as follows: One considers the gradient operator 
\begin{equation} \label{equ:def_DG}
\mathbf{D}_{\Lambda} \coloneq \begin{pmatrix}
\frac 1 {\tilde \Delta_1} & - \frac 1 {\tilde \Delta_1} & 0& \dots &0 &0
\\
0 & \frac 1 {\tilde \Delta_2} & - \frac 1 {\tilde \Delta_2} & \dots &0 &0
\\
\vdots&\vdots &\ddots & \ddots & \vdots & \vdots
\\
0 & 0&0& \dots & \dots &\frac{1}{\tilde \Delta_{D(N)-1}} & -\frac{1}{\tilde \Delta_{D(N)-1}}
\end{pmatrix} \in \mathbb{R}^{(D(N)-1) \times D(N)}
\end{equation}
corresponding to the grid $\Lambda$, so that $\mathbf{a''}=\mathbf{D}_\Lambda \mathbf{a'}$ is the coefficient vector of the piecewise constant function corresponding to $a''$. Combining \eqref{equ:discrete_quadratic_problem} and \eqref{equ:def_DG}, we can consider as a discrete version of the reconstruction problem, 
\begin{equation}
\min_{\hat a' \in V_N'} \|\mathbf{M} \mathbf{a'}- \mathbf{Y}\|^2 \quad \text{subject to} \begin{cases} \|\mathbf{a'}\|_\infty \leq M_1,
\\
\|\mathbf{D}_{\Lambda}\mathbf{a'}\|_\infty \leq M_2.
\end{cases} 
\label{equ:discrete_minimization_problem}\end{equation}
Note that allowing two different bounds $M_1$ and $M_2$ offers more flexibility, while serving the same purpose in the theoretical setting of creating compactness. In particular this allows to target $a''$ more specifically by stricter bounds in order to avoid oscillatory behavior.

Moreover, note that due to the structure of~$\mathbf{M}$, $\mathbf{a'}$ can only be reconstructed up to a constant vector, meaning~$\hat a'$ can only be determined up to constant analogously to the considerations leading to~\eqref{e.consequence}.  For the sake of simplicity we further restrict the optimization to competitors with $\hat a'(0)=0$ and we assume that $a'(0)=0$ as well.

Moreover, we notice that, because of discretizations in time in~$\mathcal{\tilde{J}}_{N,\varepsilon}$ and the use of non equivalent constraints, the minimizer of~\eqref{equ:discrete_minimization_problem} does not precisely coincide with the minimizer of the original minimization problem. It is however reasonable to think that it indeed approximates the true solution of~$\mathcal{J}_{N,\varepsilon}$, which in turn approximates the true energy due to $\Gamma$-approximation.

Being \eqref{equ:discrete_minimization_problem} a least squares problem with norm  constraints, a variety of optimization algorithms are applicable. For the results presented in this work we used the CVX toolbox \cite{cvx,gb08}, which is well suited as all functions and operations can be written as convex functions and constraints.

\begin{figure}[h!]
\begin{center}
\begin{tabular}{l r}
\includegraphics[width=0.48\textwidth]{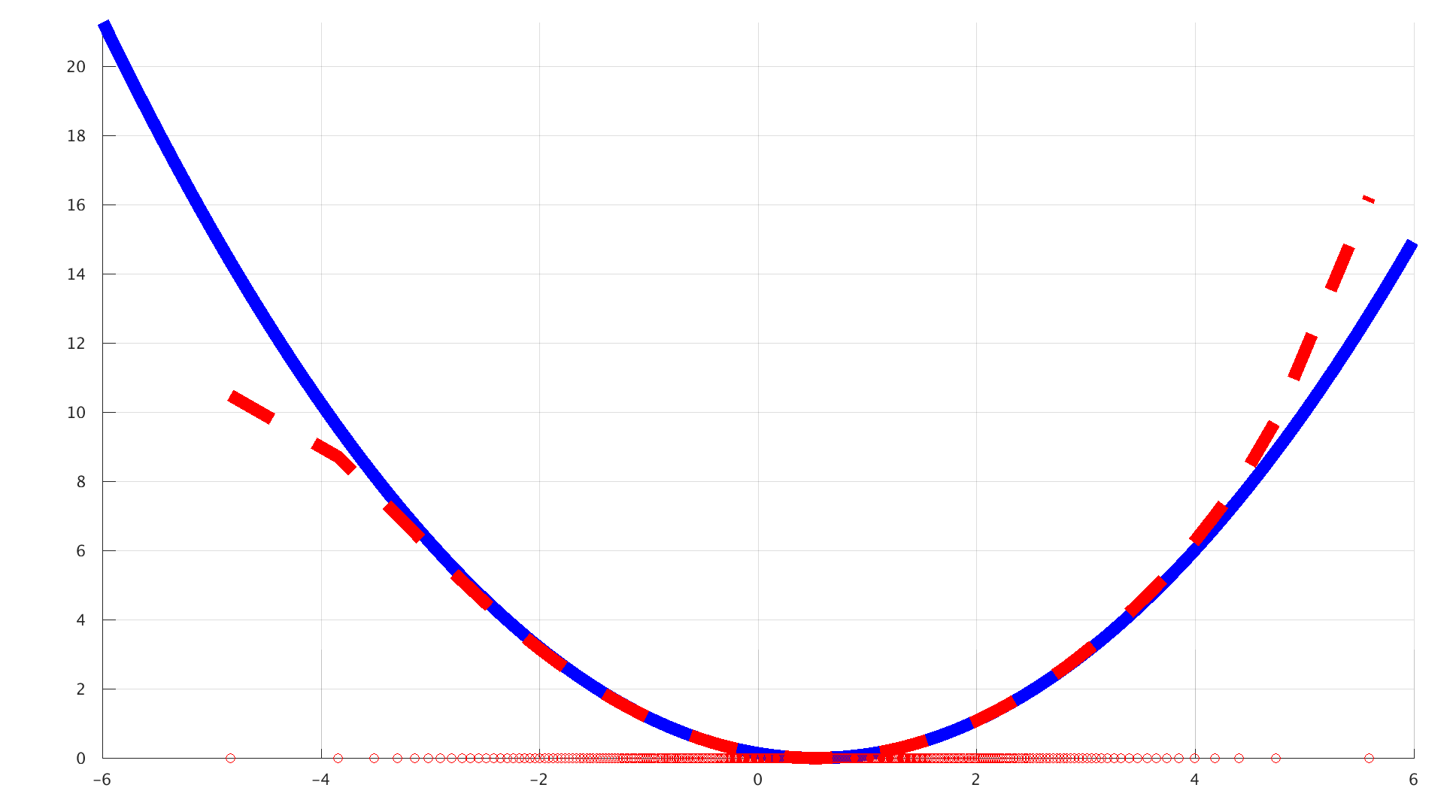}&\includegraphics[width=0.48\textwidth]{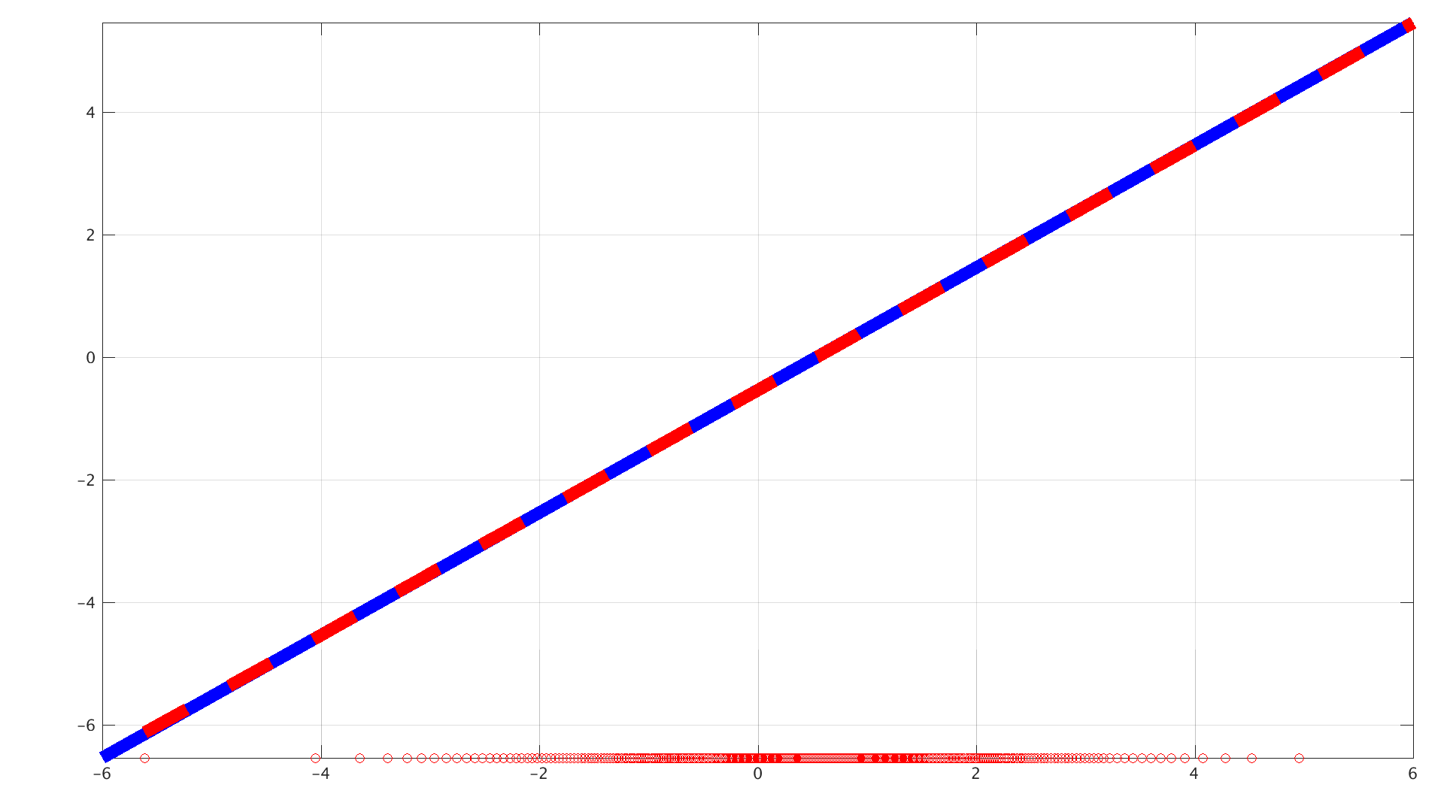} \\
\includegraphics[width=0.48\textwidth]{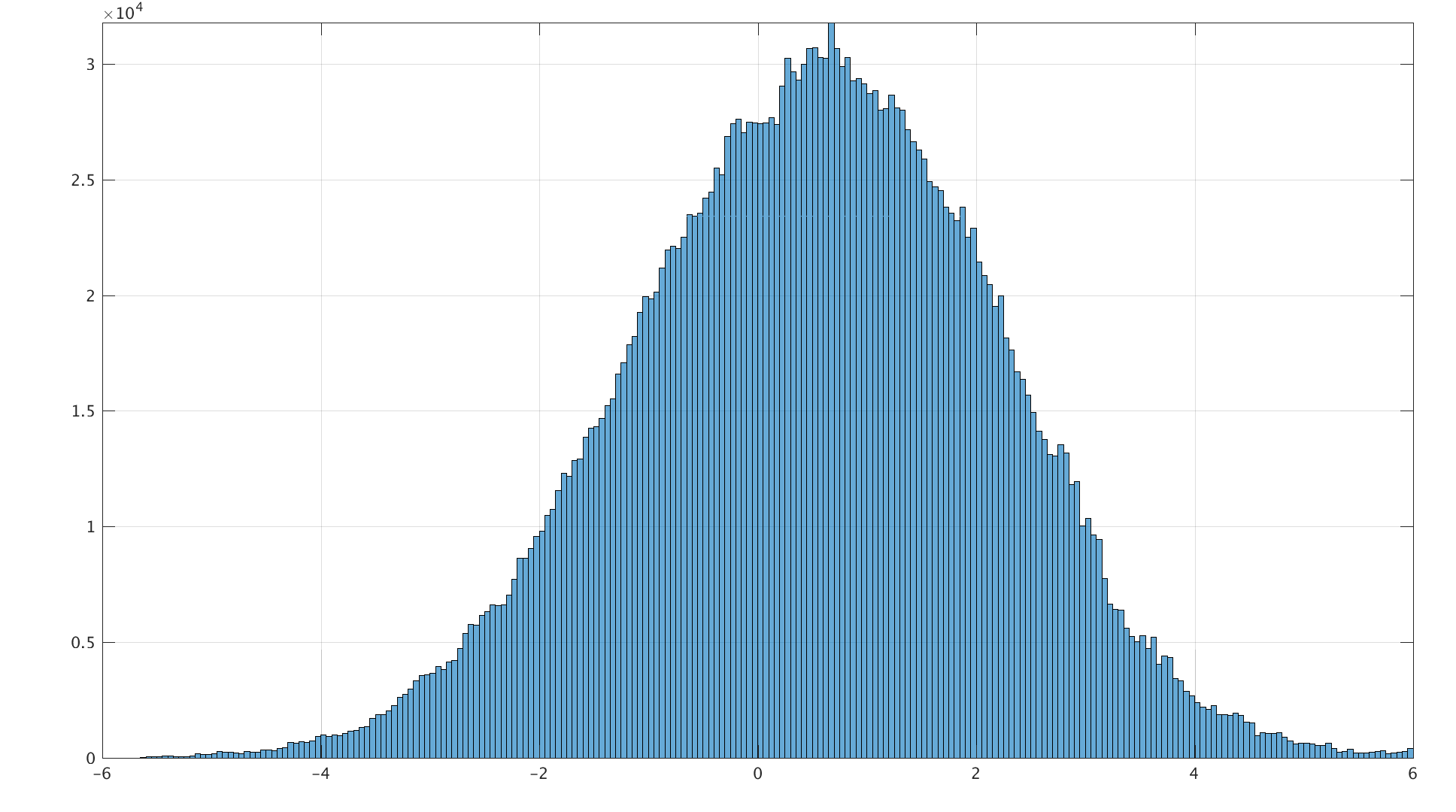} &
 
\end{tabular}
{\begin{tabular}{c| c| c| c| c|c|c}
d & T&$M_1$ & $M_2$ &${N_e}$&N & D(N) \\ \hline 
20 & 1 &1000&1000 &2000 &  60 &4N 
\end{tabular} } 
\end{center}

\caption{Reconstructions $\hat a$ (left) and $\hat a'$ (right) in dotted red and true $a$ and $a'$ in blue. The underlying red circles depict the (adaptive) nodes in $\Lambda$. Below you see the distribution $\tilde \eta$ of available data.}
\label{Fig:Quadratic_Potential}
\end{figure}

\subsubsection{Linear elasticity -- a trivial example}\label{Subsection:Numericsquadratic}

We start with the standard potential $a(y)=\frac{y^2} {2}$, which is considered in the context of linear elasticity. 
As this potential is uniformly convex and contained in $V_N$, one  expects the reconstruction to work better compared to more complex potentials.

Figure \ref{Fig:Quadratic_Potential} depicts the approximation of $a$ and $a'$ for a quadratic potential, and shows very accurate approximation. The approximation of $a'$ appears almost exact everywhere, while the approximation of~$a$ loses accuracy at the boundaries of the observed interval, due to summation of minor (systematic) errors. Nonetheless $a$ is overall  best approximated in regions where~$\tilde \eta$  has higher density. Thus elastic potentials can be identified very well using this approach.
However, for more complex potentials one may not expect to obtain always such a good reconstruction, and  in the following sections we consider the impact of various parameters on the reconstruction of non-quadratic potentials.

\begin{figure}[h!]
\begin{center}
\begin{tabular}{r  l} 
{\centering\rotatebox[origin=l]{90}{\qquad$N_e=2$}} \includegraphics[width=0.4\textwidth]{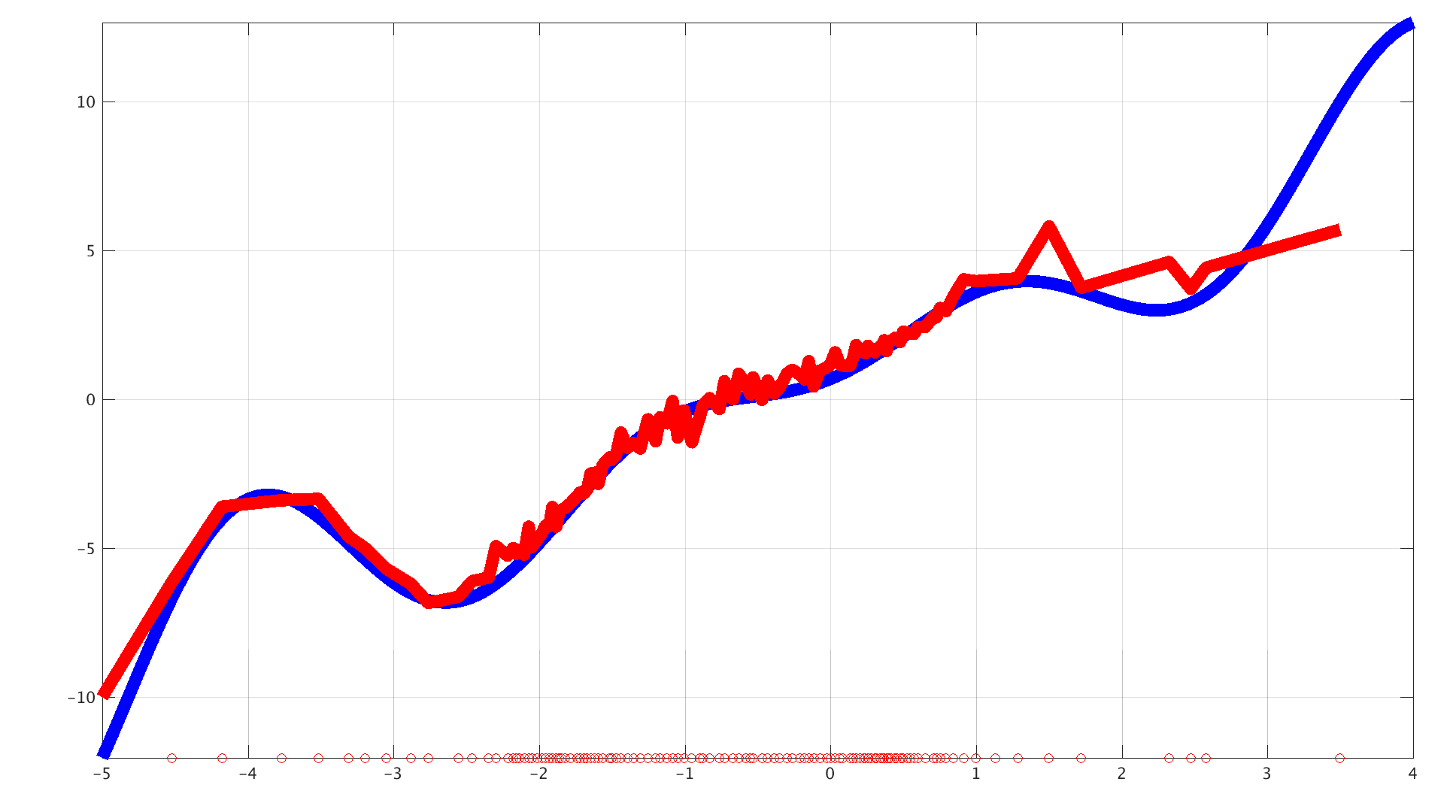} & \includegraphics[width=0.4\textwidth]{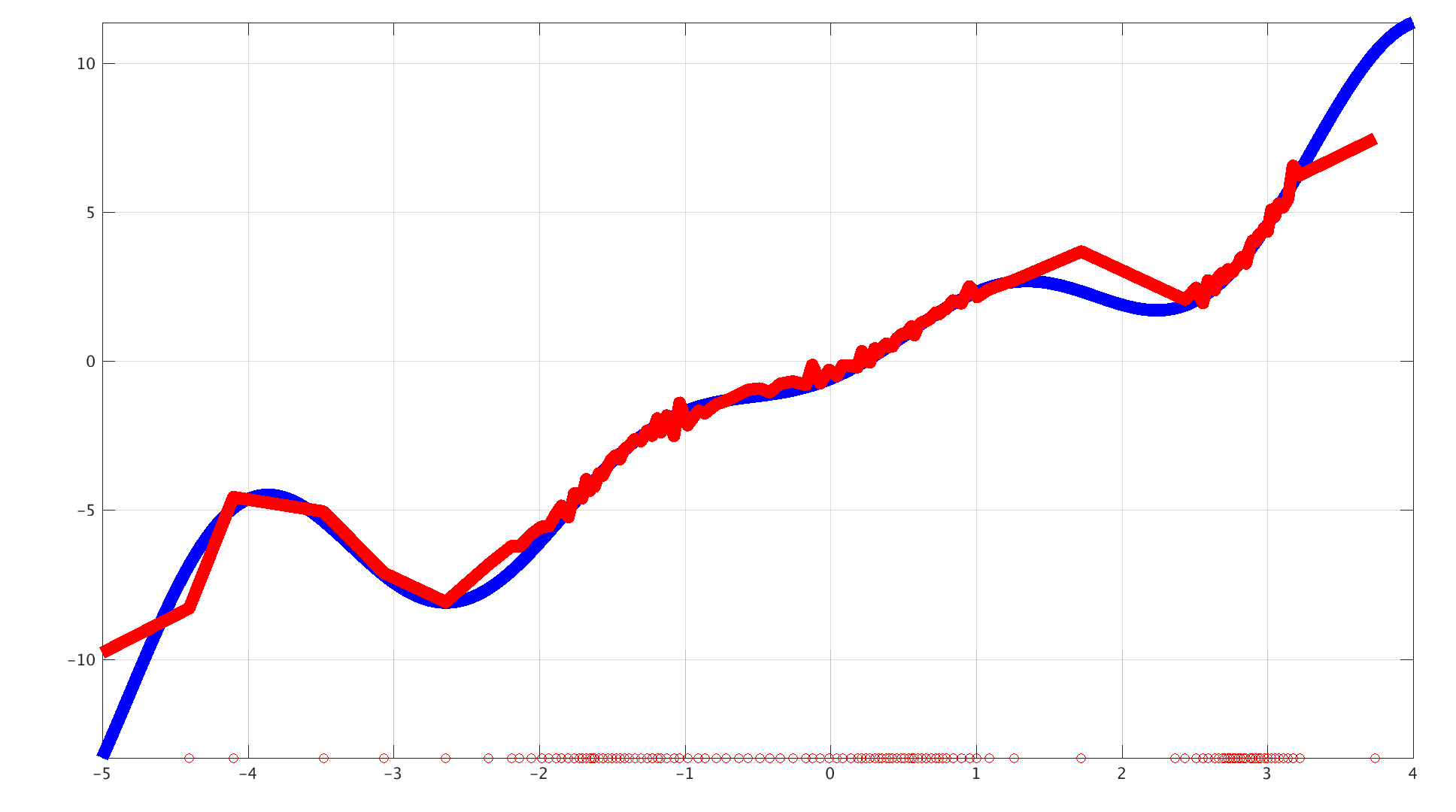}  \rotatebox[origin=l]{90}{\qquad$N_e=5$}\\ 
\rotatebox{90}{\qquad$N_e=10$}
\includegraphics[width=0.4\textwidth]{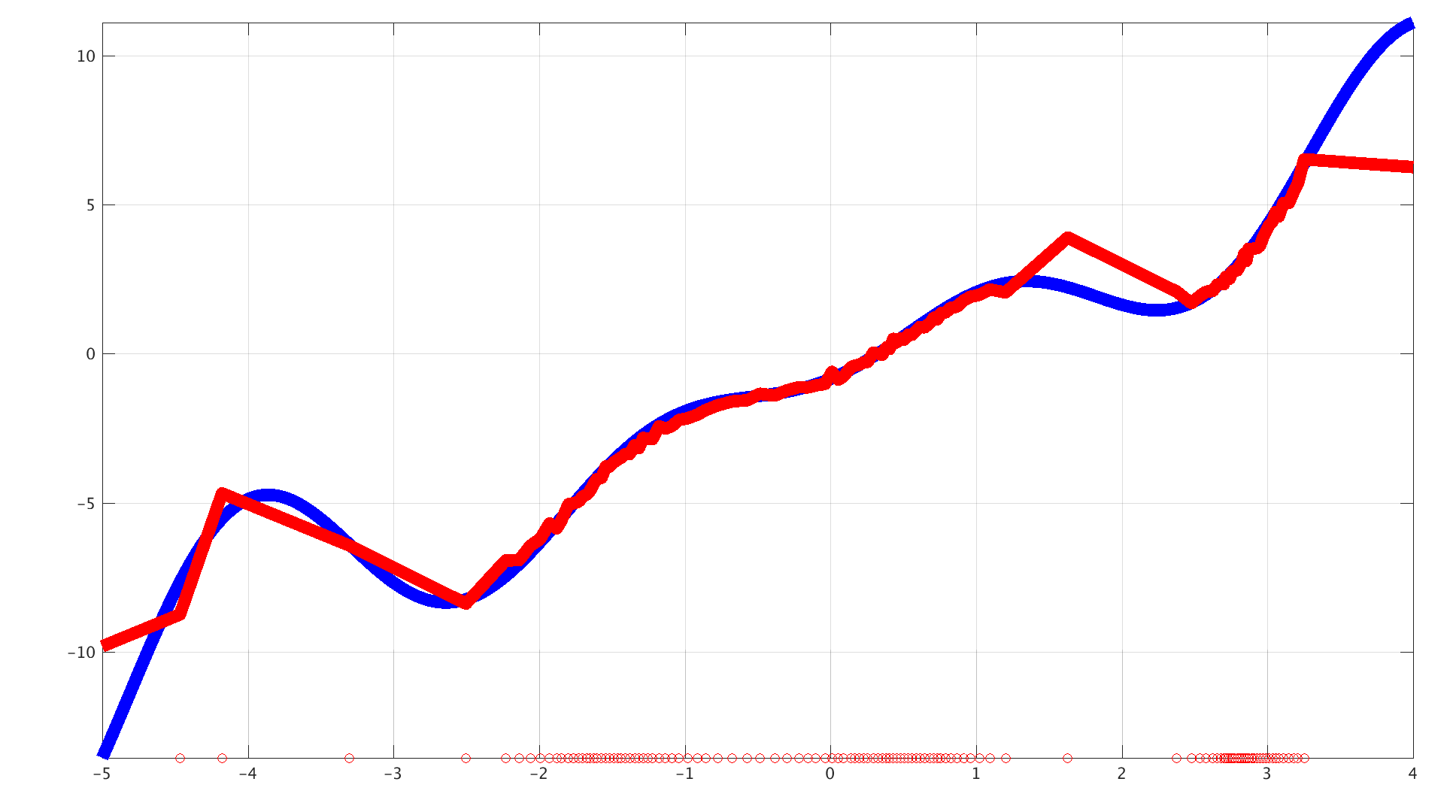} & 
\includegraphics[width=0.4\textwidth]{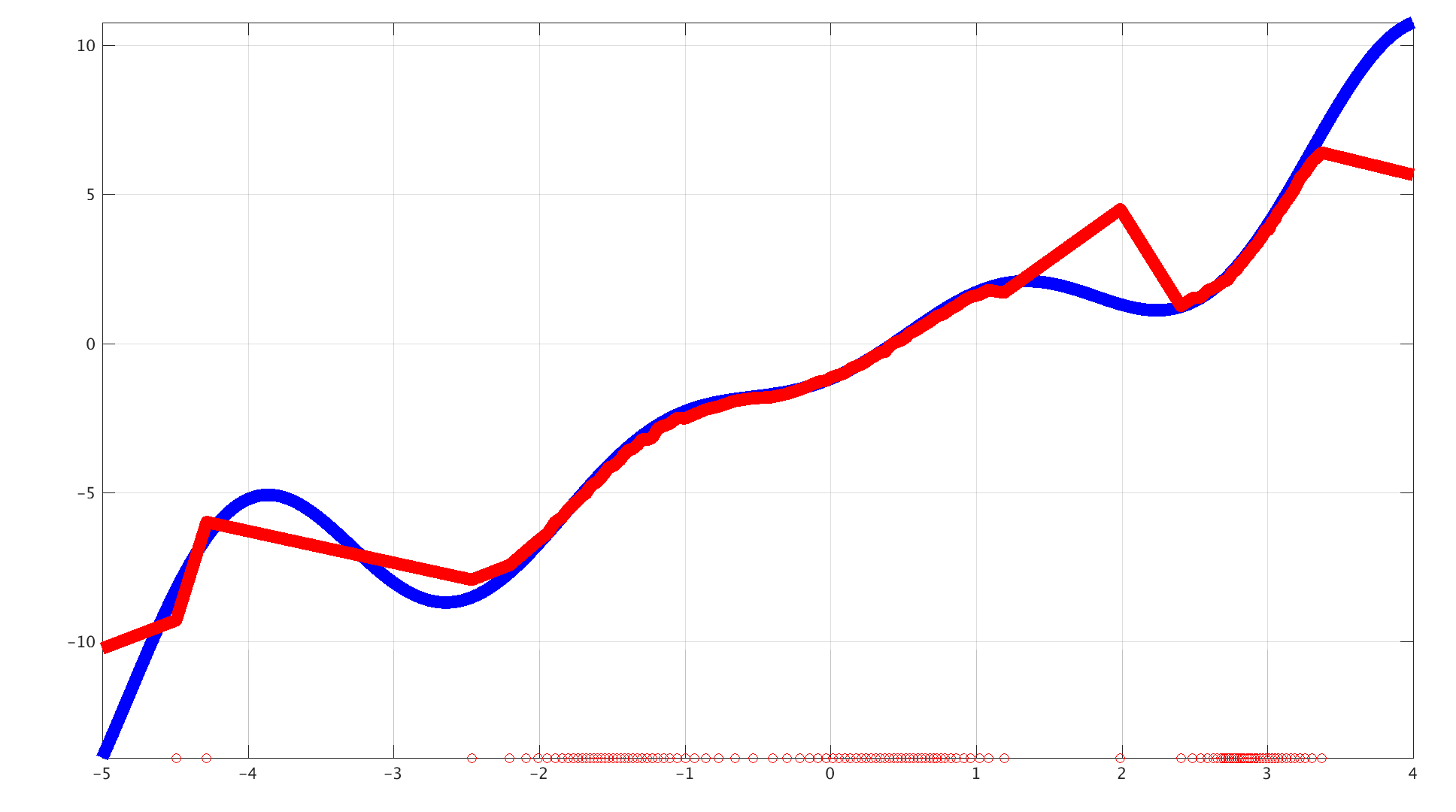}  \rotatebox{90}{\qquad$N_e=100$}\\  
\includegraphics[width=0.4\textwidth,height=2cm]{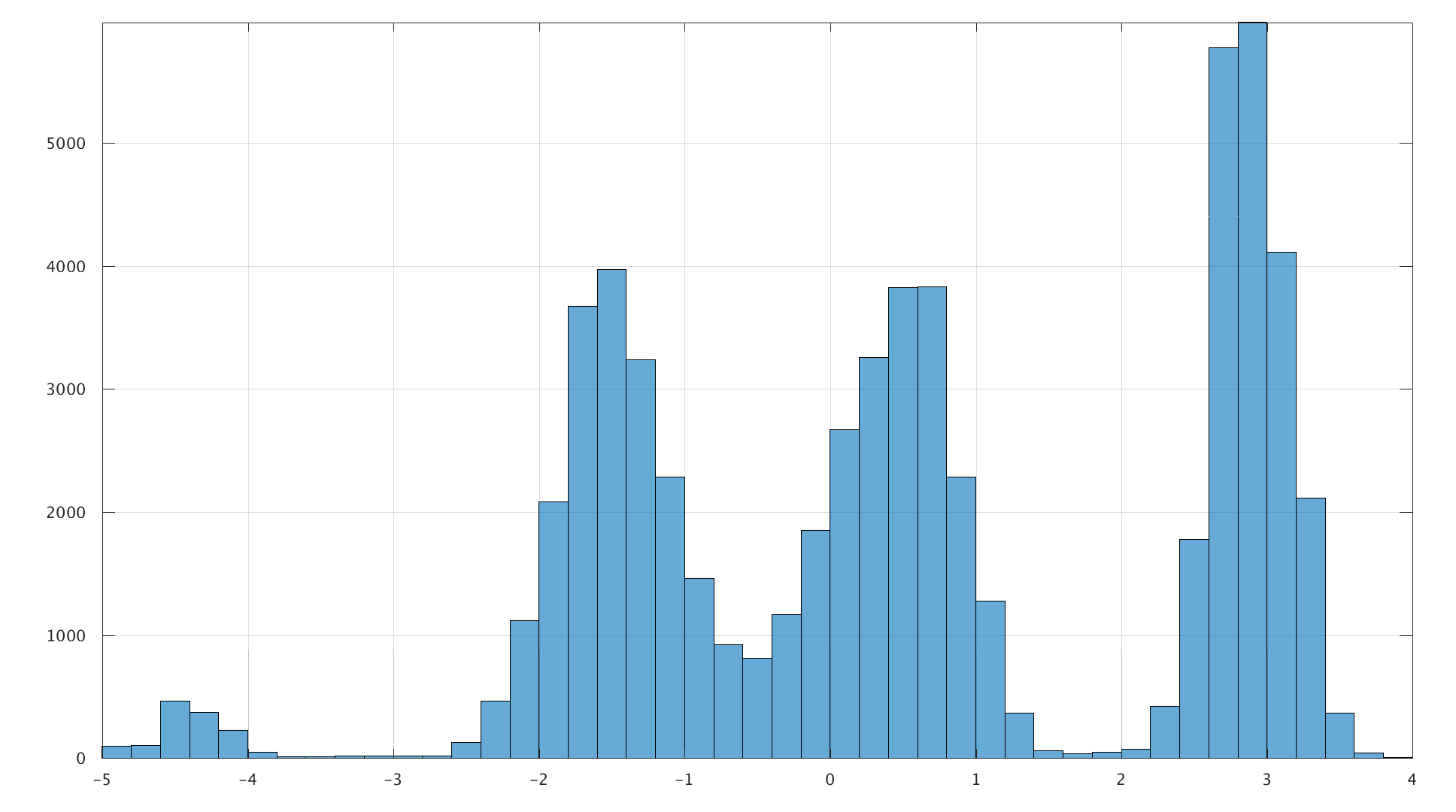} 
\end{tabular}
\begin{tabular}{c| c| c| c| c| c|c}
d &T &$M_1$&$M_2$&${N_e}$ &N & D(N) \\ \hline 
20 & 1 &1000&1000&[2,5,10,100] & 30 &100
\end{tabular}
\end{center}
\caption{Reconstruction of~$a'$ with increasing number~${N_e}$ of measurements. The true~$a'$ is depicted by the blue curve while the red line depicts the reconstruction~$\hat a'$ following~\eqref{equ:discrete_minimization_problem}. The red circles at the bottom of the figures depict the position of nodes of the underlying mesh~$\Lambda_N$.
The histogram  below describes the density of the available information encoded by the probability measure~$\tilde{\eta}$ for the case $N_e=60$.}
\label{Fig:Parameters_for_varying_measurements}
\end{figure}

\subsubsection{Impact of the amount of information on the reconstruction quality}\label{Subsection:Number_experiments}
From a numerical perspective, solving~\eqref{equ:discrete_minimization_problem} is a least squares problem, and with more available information, one would expect to increase the reliability of reconstructions. This amount of information in our setting mainly depends on two factors -- the amount of measurements ${N_e}$ made for every experiment and  the number~$N$ of observed experiments. Thus, we want to demonstrate the effect of increasing amount of information, in particular verifying that for sufficient amount of information one can accurately recover solutions, while too little information yields unstable recovery.

%
%
%

In Figure \ref{Fig:Parameters_for_varying_measurements} all parameters of the reconstruction are fixed except for the amount of measurements~${N_e}$. One observes that results get more reliable and noise is reduced with an increasing number ${N_e}$ of measurements and for a sufficient amount of information one can precisely reconstruct  $a'$ on the support of $\tilde \eta$, whose density is approximately depicted below by the histogram of observed information. In the reconstructions with little information, the solutions are oscillating, and,  although following the overall trend of the true energy functional, do not perfectly capture its behavior. 
One can also see that in regions without information, and correspondingly with no or few nodes, the approximation is crude and can not be trusted, but in regions with much information a rather accurate reconstruction can be found. 
Of course, in many practical applications the amount of sampling in time might be limited by technical limitations. The resulting issue of lack of information can be offset by considering a larger number of experiments.

\begin{figure}[h!]
\begin{center}
\begin{tabular}{ r l}
\rotatebox{90}{\qquad$N=1$}
\includegraphics[width=0.45\textwidth]{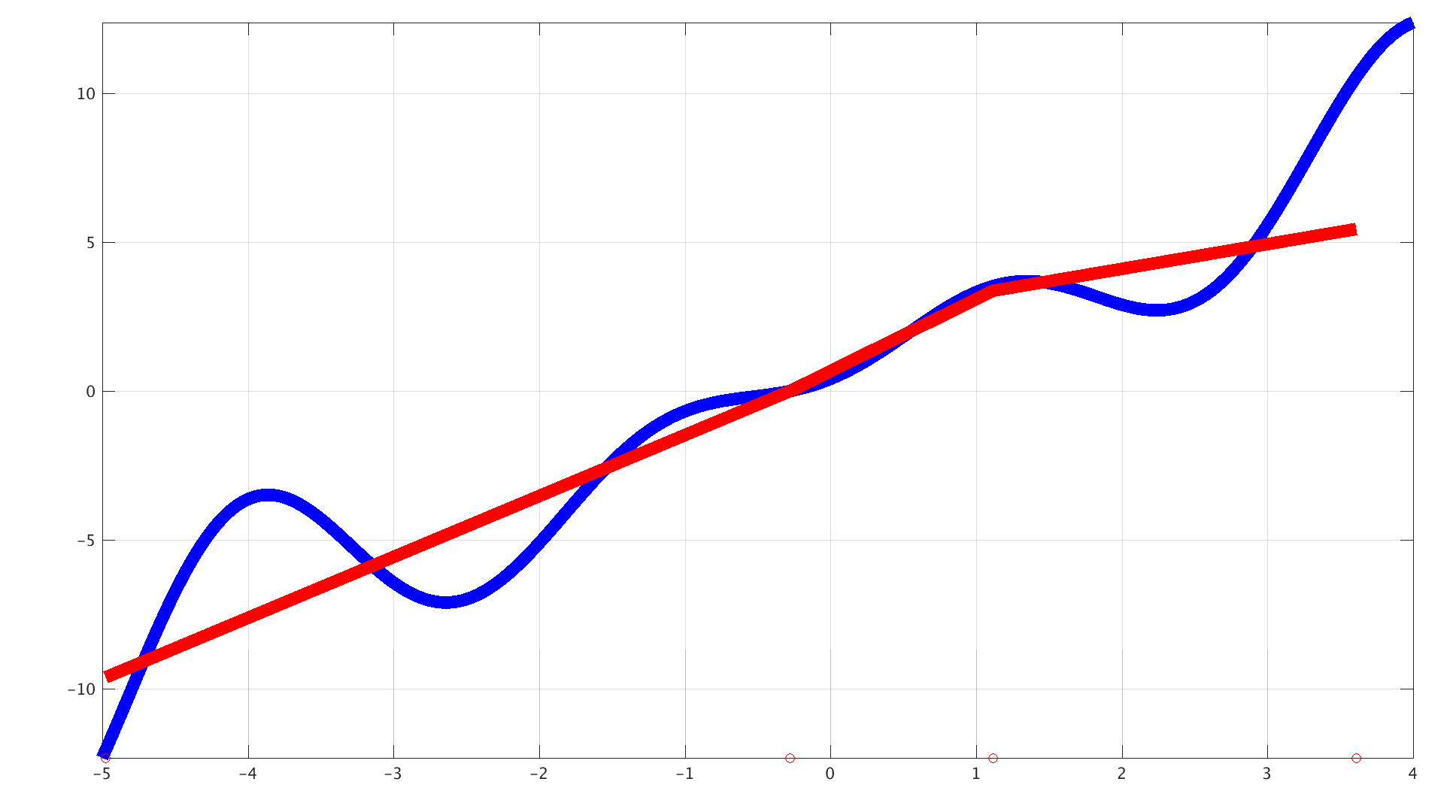} & \includegraphics[width=0.45\textwidth]{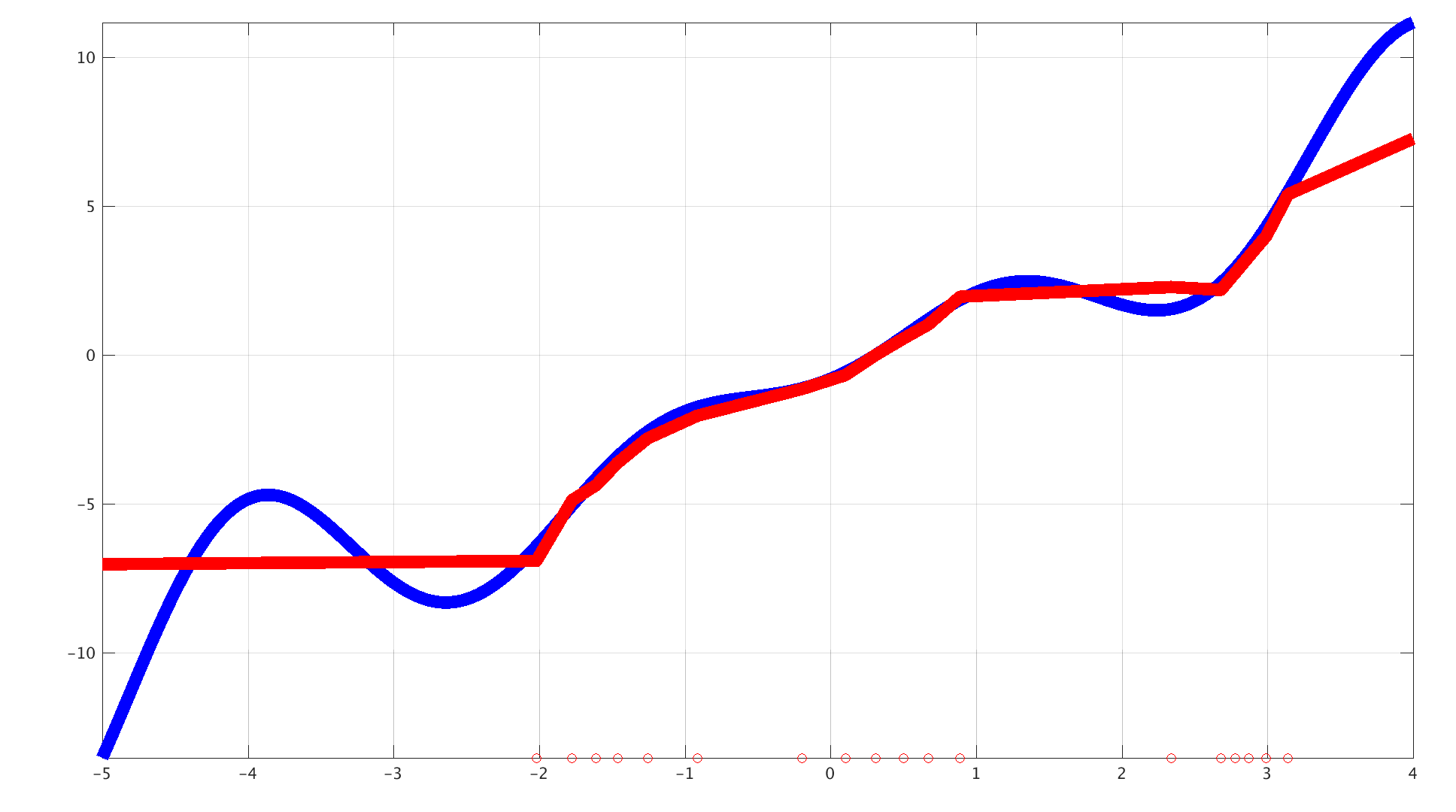}  \rotatebox{90}{\qquad$N=5$}\\ \rotatebox{90}{\qquad$N=20$} 
\includegraphics[width=0.45\textwidth]{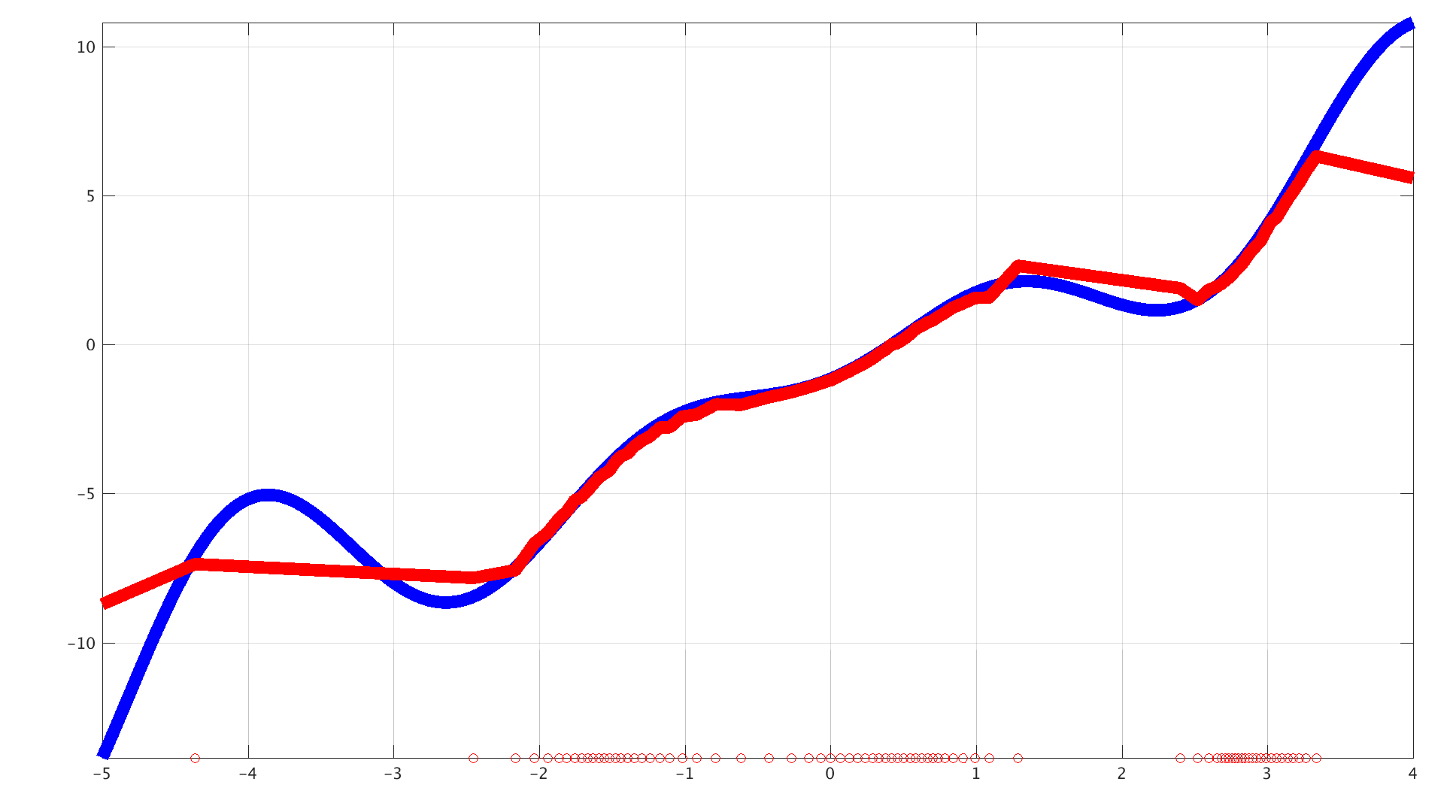} & \includegraphics[width=0.45\textwidth]{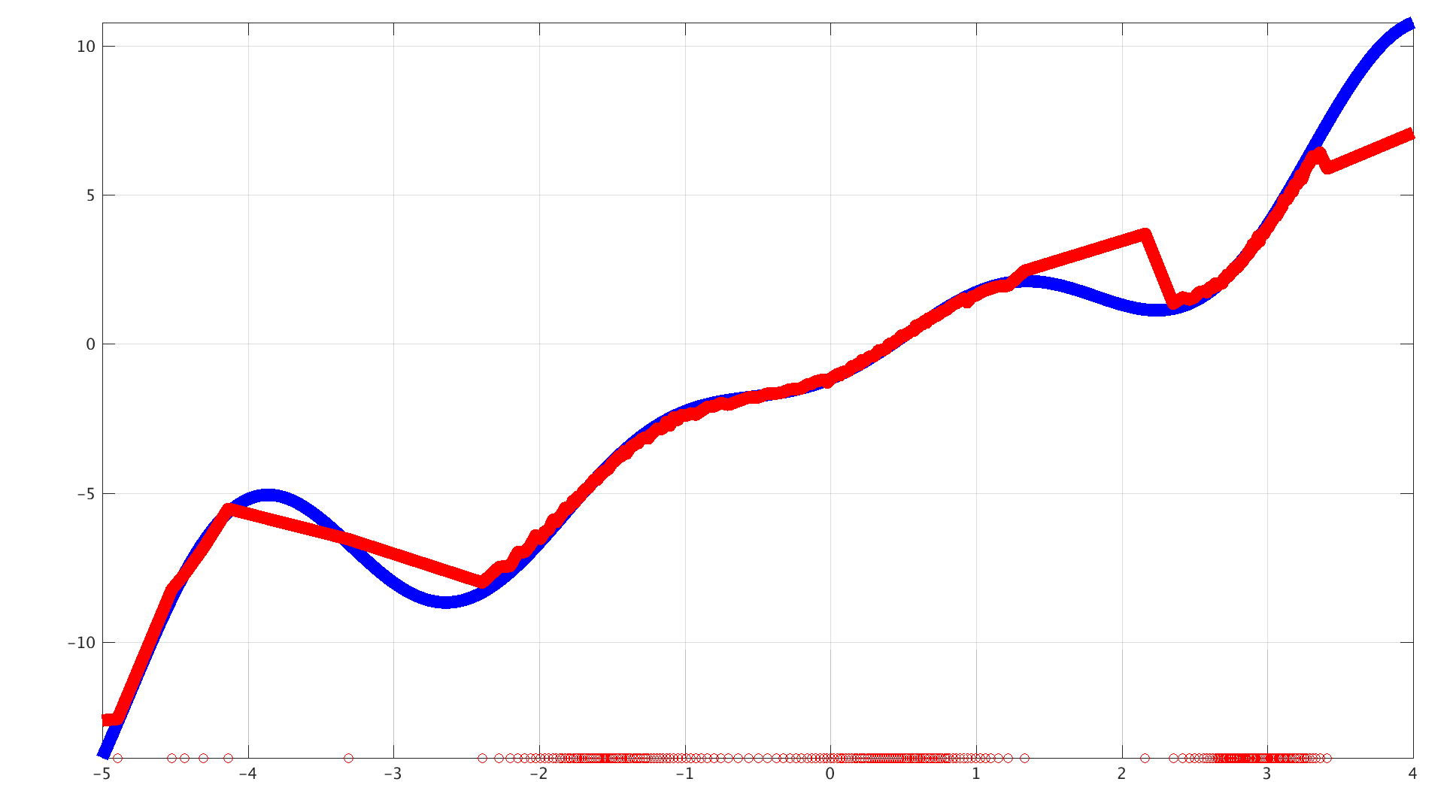} \rotatebox{90}{\qquad$N=60$}\\
\includegraphics[width=0.45\textwidth,height=2cm]{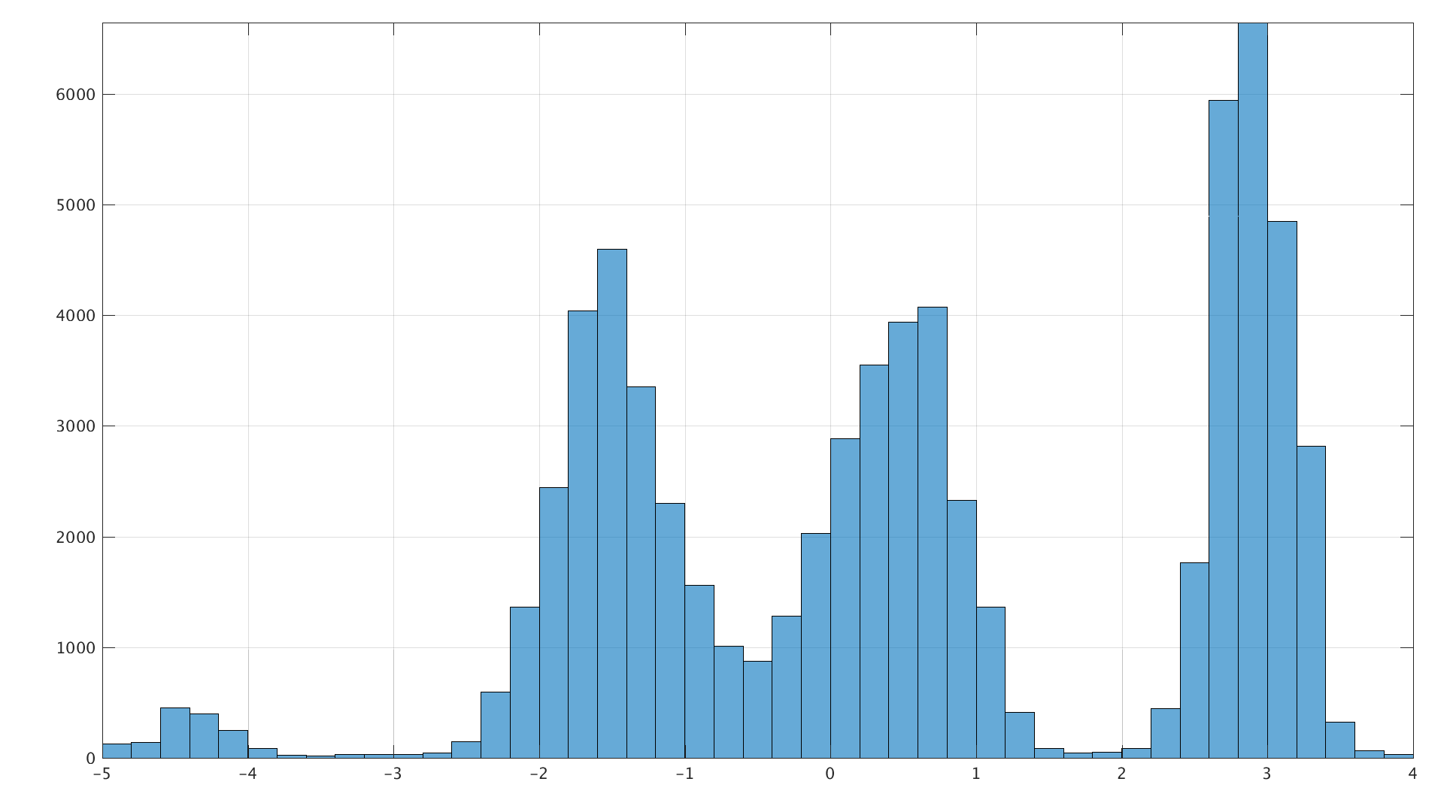} 

\end{tabular}
\begin{tabular}{c| c| c| c| c|c|c}
d & T&$M_1$ & $M_2$ &${N_e}$&N & D(N) \\ \hline 
20 & 1 &1000&1000 &55 & [1, 5, 20, 60] &4N
\end{tabular} 
\end{center}
\caption{Reconstruction of $a'$ for varying $N$ and $D(N)=4N$ and distribution $\tilde \eta$ for $N=60$. Graphic as described in Figure \ref{Fig:Parameters_for_varying_measurements}.}
\label{Fig:Varying_N_and_nodes}
\end{figure}

In fact the main result of the provided theory is that for $N\to \infty$ -- considering ever more experiments -- one can reconstruct~$a'$ increasingly well. 
In order to approximate with~$V_N$ the space~$A_{M,R}$ in the sense of the uniform approximation property, it is necessary to increase adaptively the amount of nodes~$D(N)$. A trivial way to do this is considering a linear relation between~$N$ and  the amount of nodes~$D(N)$, i.e., $D(N) \asymp N$.
Figure~\ref{Fig:Varying_N_and_nodes} shows that an increased number of experiments and nodes improves significantly the quality of the reconstruction.
%

In comparison, we show that the improved reconstruction is not solely the result of a finer grid, but rather a consequence of more available information. Therefore, we considered in Figure \ref{Fig:Varying_N_fixed_nodes} the same experiment as in Figure \ref{Fig:Varying_N_and_nodes}, but with the number of nodes $D(N)=300$ independent of $N$. While the reconstruction even for a single experiment is not particularly bad, one can see that it is not very smooth, representing a smaller degree of confidence in the solutions. For increasing number $N$ of  experiments, the results become smoother. Moreover, regions not significantly visited by a single experiment, and therefore not well supported by the mesh (e.g., left side of the plot), get better represented for a larger number of experiments as they might get explored more thoroughly. However, note that there appear to be regions which does not get -- or does get very rarely -- visited independently of the number $N$ of experiments as the density of $\tilde \eta$ is zero or very small there, and therefore no reasonable reconstruction can be obtained at such locations.

\begin{figure}[h!]
\begin{center}
\begin{tabular}{ r l}
\rotatebox{90}{\qquad$N=1$} 
\includegraphics[width=0.45\textwidth]{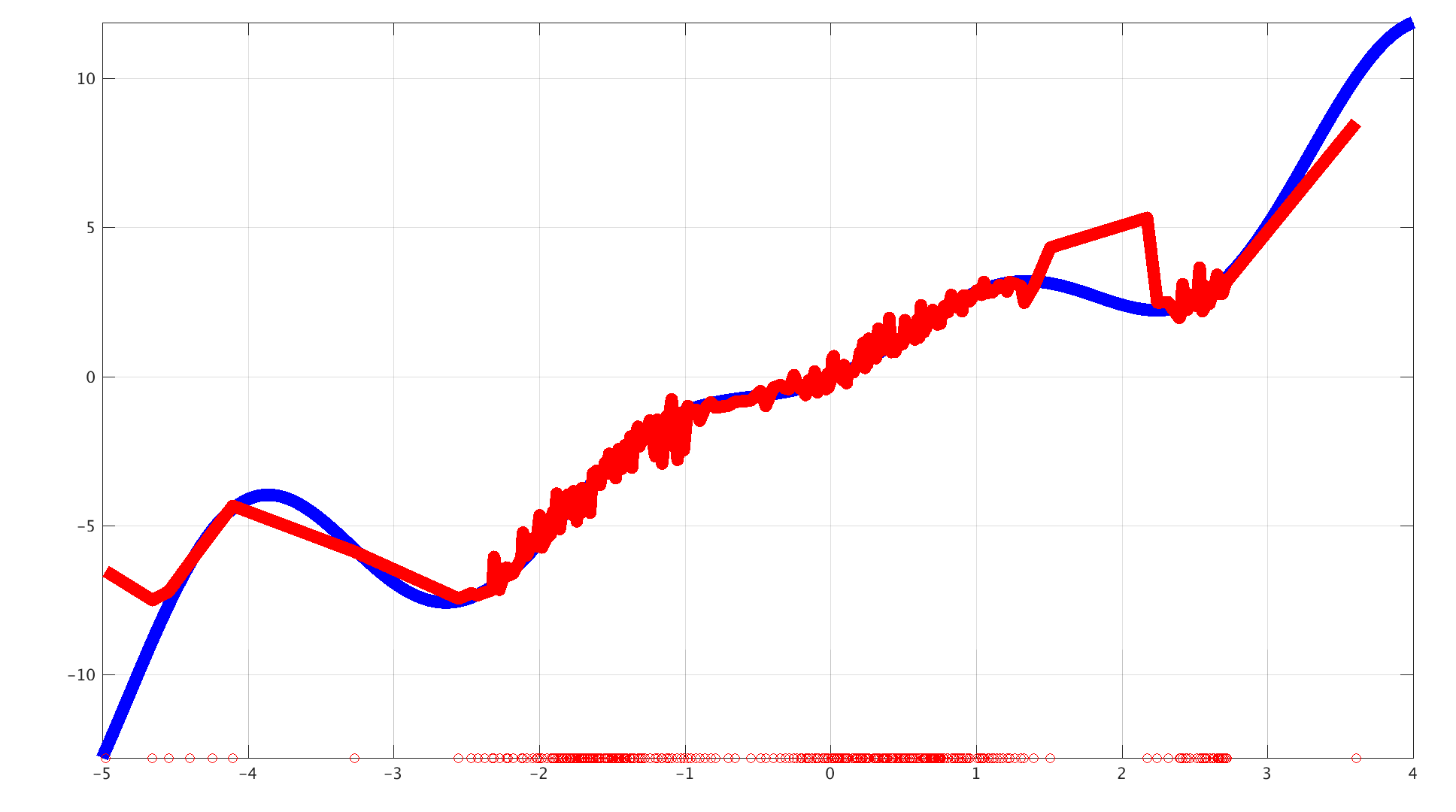} & \includegraphics[width=0.45\textwidth]{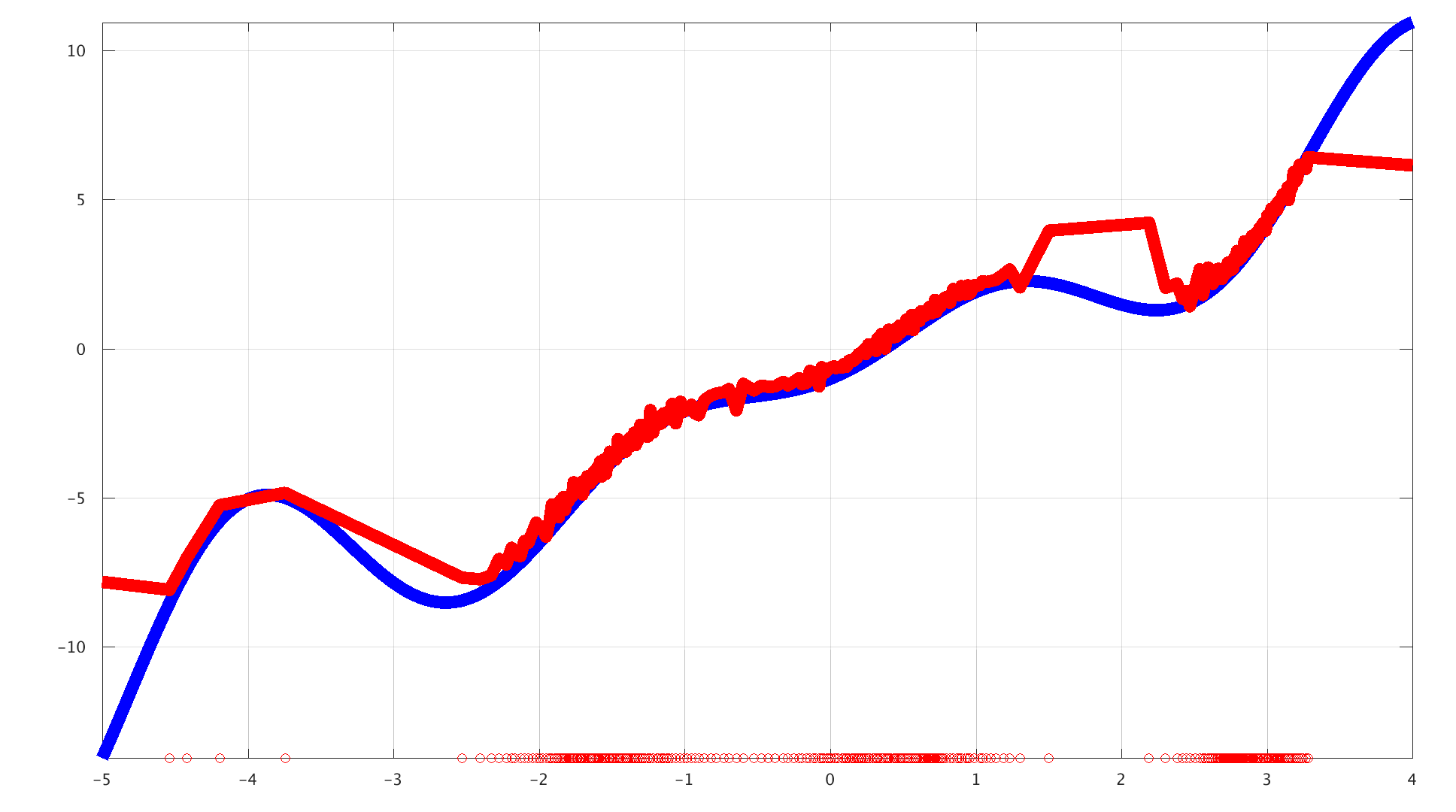} \rotatebox{90}{\qquad$N=5$}\\ \rotatebox{90}{\qquad$N=20$}
\includegraphics[width=0.45\textwidth]{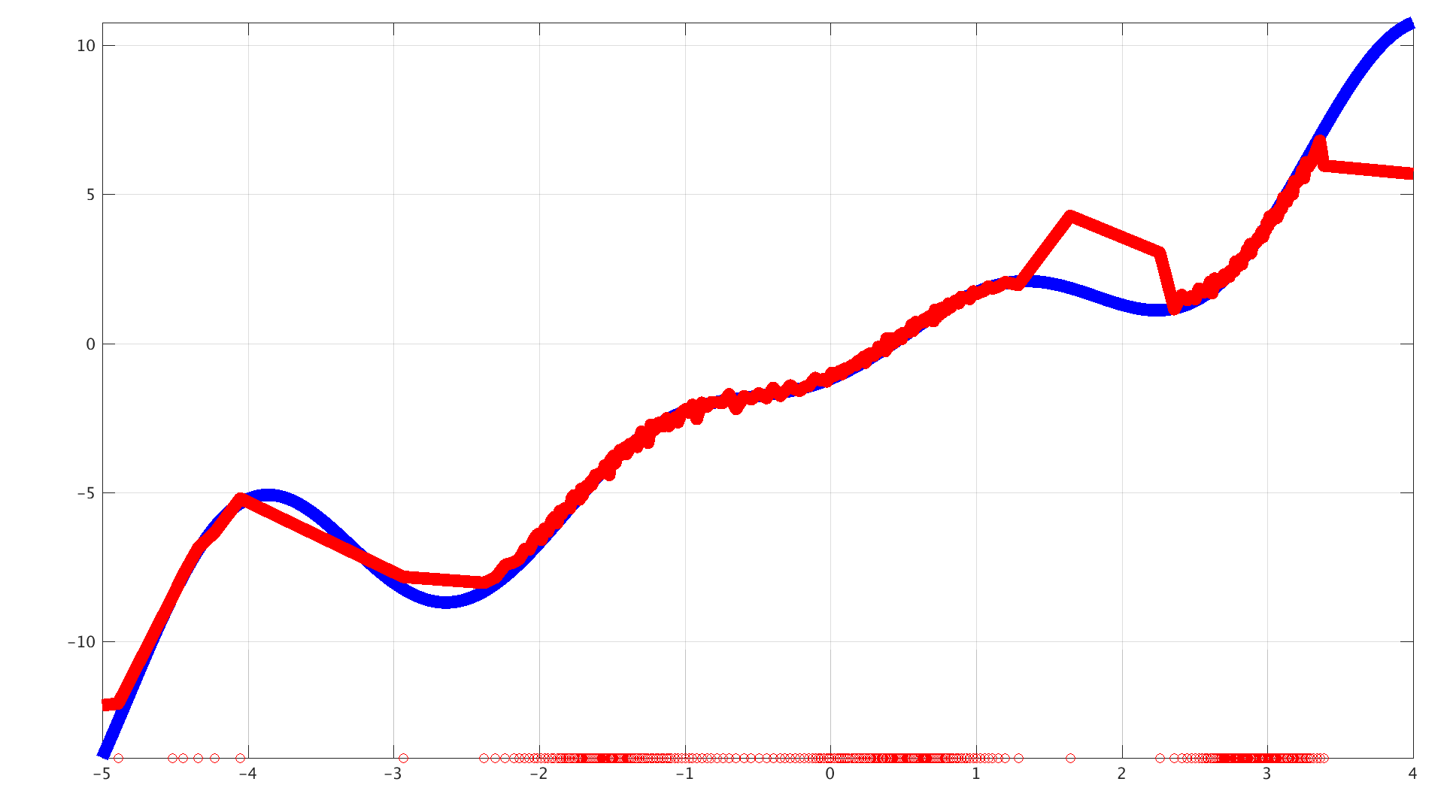} & \includegraphics[width=0.45\textwidth]{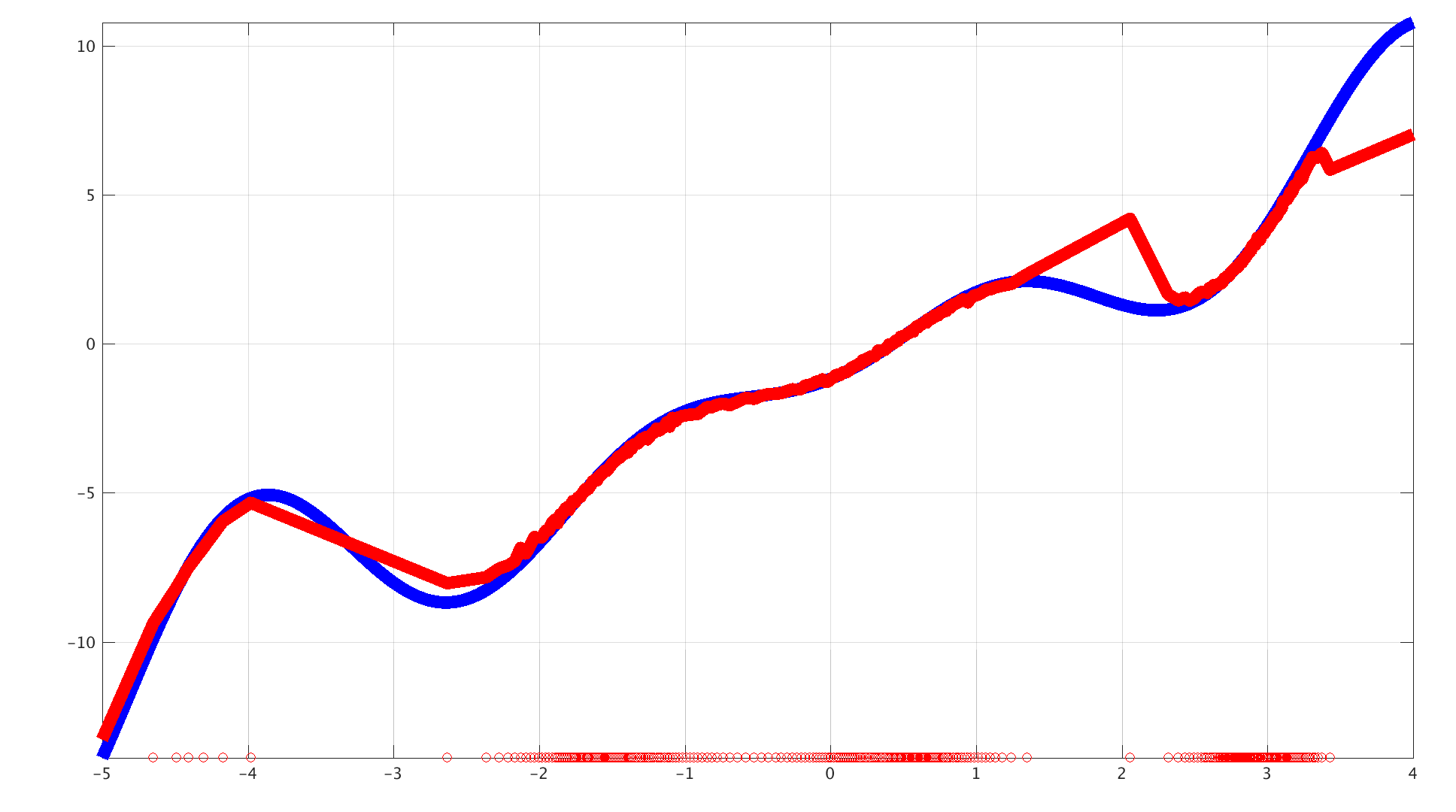}\rotatebox{90}{\qquad$N=60$}\\
\includegraphics[width=0.45\textwidth,height=2cm]{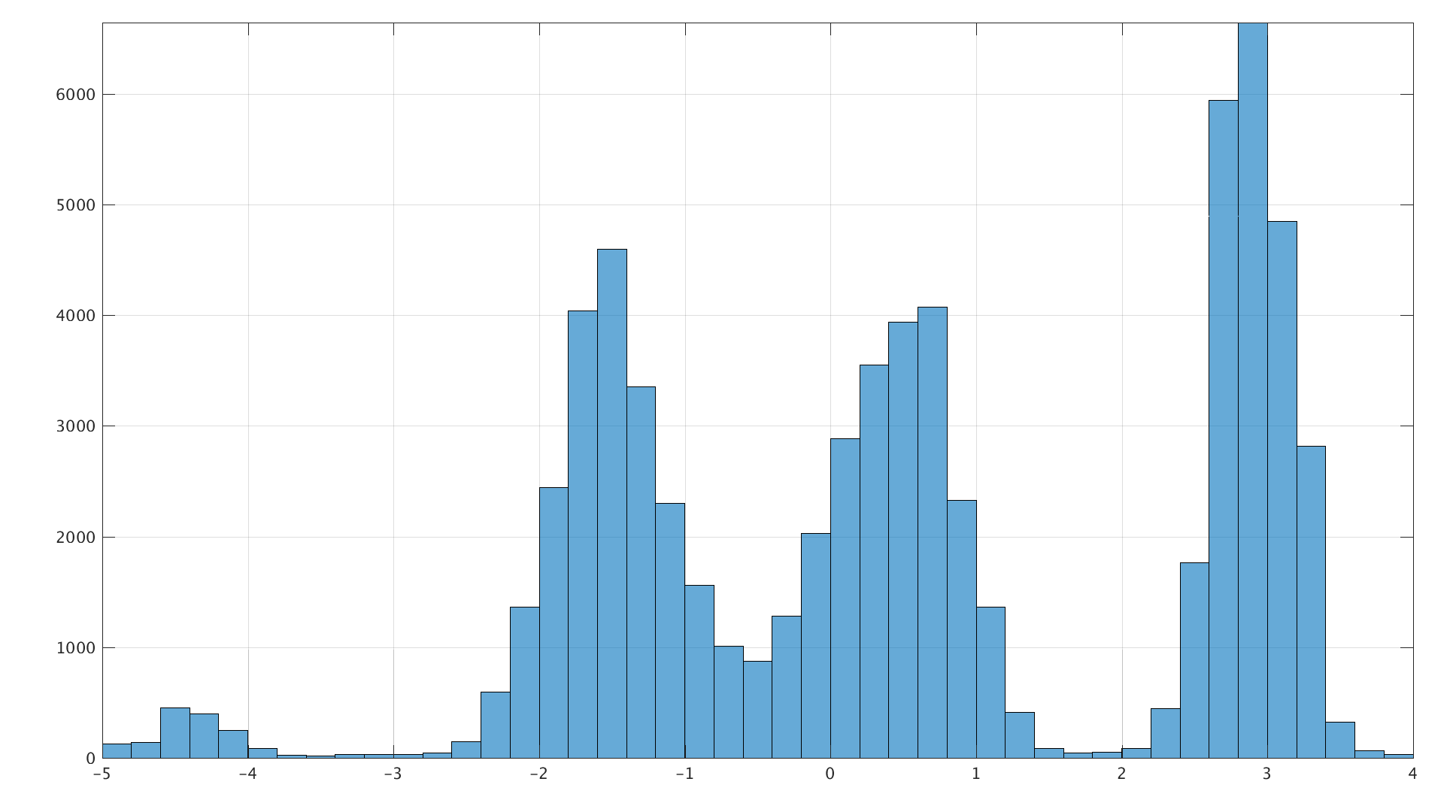} 

\end{tabular}
\begin{tabular}{c| c| c| c| c|c|c}
d & T &$M_1$ &$M_2$ &${N_e}$&N & D(N) \\ \hline 
20 & 1 & 1000 &1000& 55 & [1, 5, 20, 60] &300
\end{tabular} 
\end{center}
\caption{Reconstruction of $a'$ for varying $N$ and fixed $D(N)=300$ and distribution $\tilde \eta$. Graphics as described in Figure \ref{Fig:Parameters_for_varying_measurements}.}
\label{Fig:Varying_N_fixed_nodes}
\end{figure}

\subsubsection{Suitable $W^{2,\infty}$ constraints} \label{Subsection:Numeric_constraints}

The theory in this work does not yet provide a method for  choosing $M$ (or $M_1,M_2$ in our numerical model). It is clear that $M_1$ and $M_2$ too small will significantly limit the available class of competitors, and therefore one can not expect to capture the true $a'$ if $M_1$ and $M_2$ are much smaller than $\|a'\|_\infty$ and $\|a''\|_\infty$, respectively. On the other hand, $M_1$ and $M_2$ finite is necessary to ensure compactness from a theoretical perspective, so it is not obvious what the impact of too large $M_1$ and $M_2$ is. However, for suitable data, one would expect that for $M_1,M_2>\bar M$ sufficiently large have no real impact on the reconstruction. 

\begin{figure}[h!]
\begin{center}
\begin{tabular}{r l} 
\rotatebox{90}{\qquad$\|\hat a''\|\leq 2$}
\includegraphics[width=0.45\textwidth]{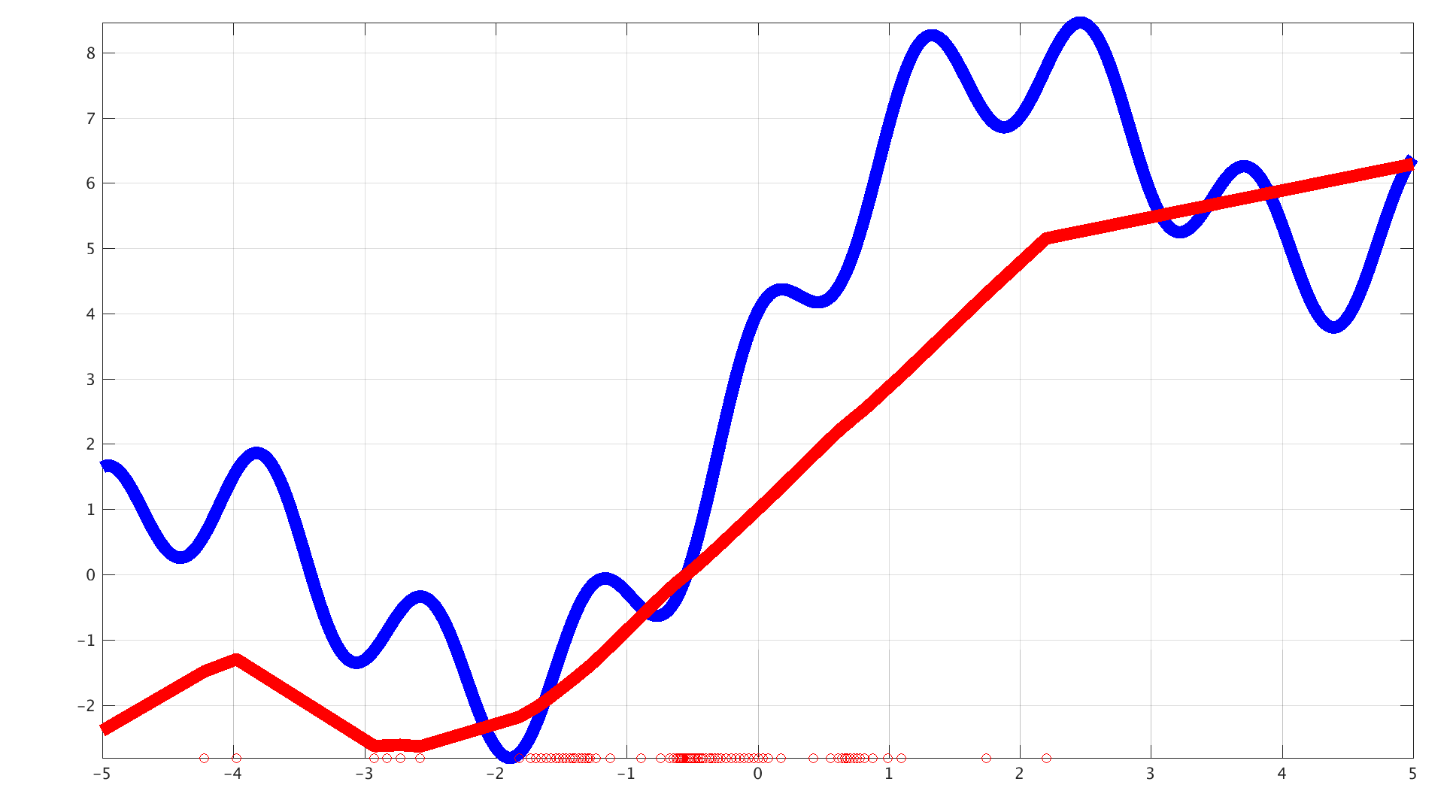} & \includegraphics[width=0.45\textwidth]{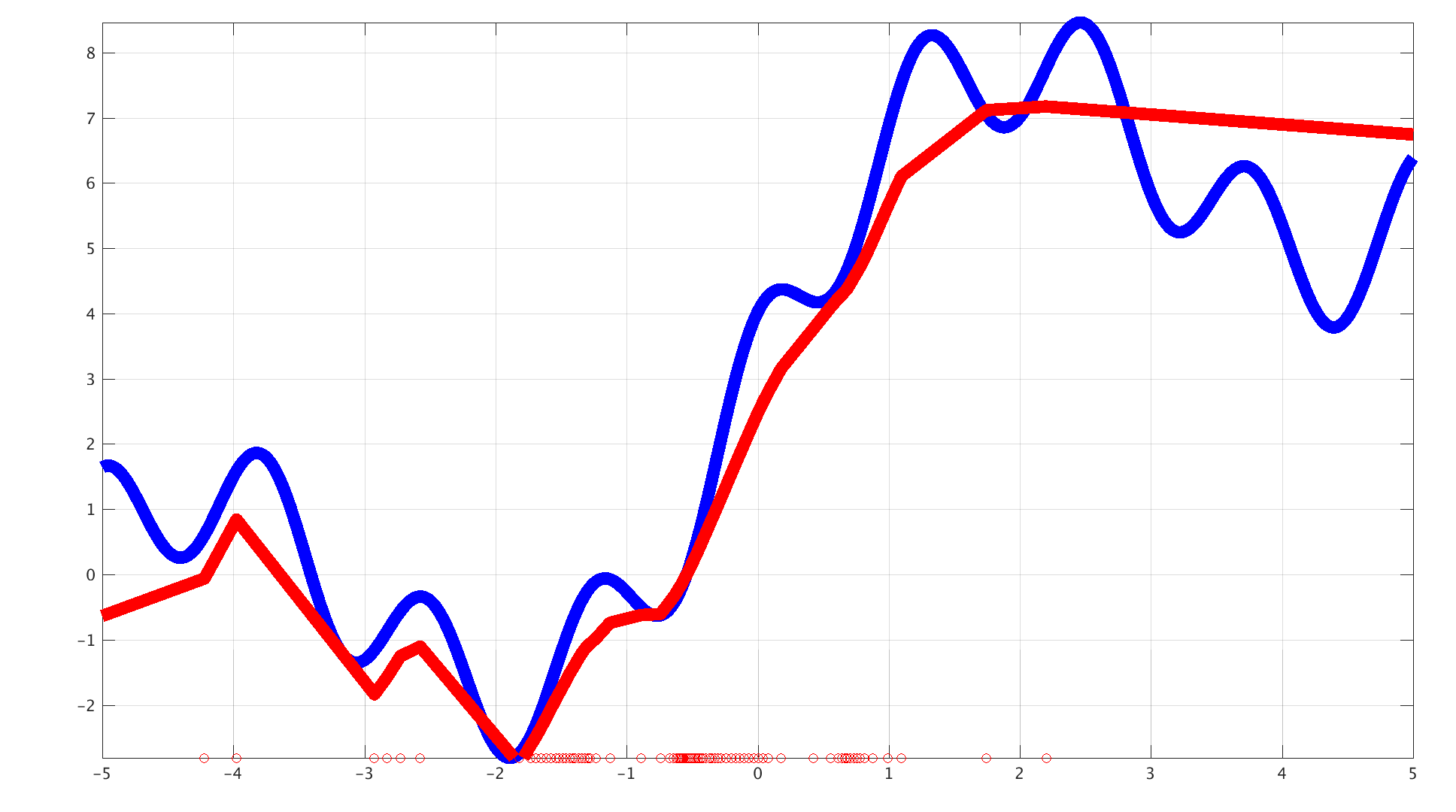}\rotatebox{90}{\qquad$\|\hat a''\|\leq 5$}\\
\rotatebox{90}{\qquad$\|\hat a''\|\leq 20$}
\includegraphics[width=0.45\textwidth]{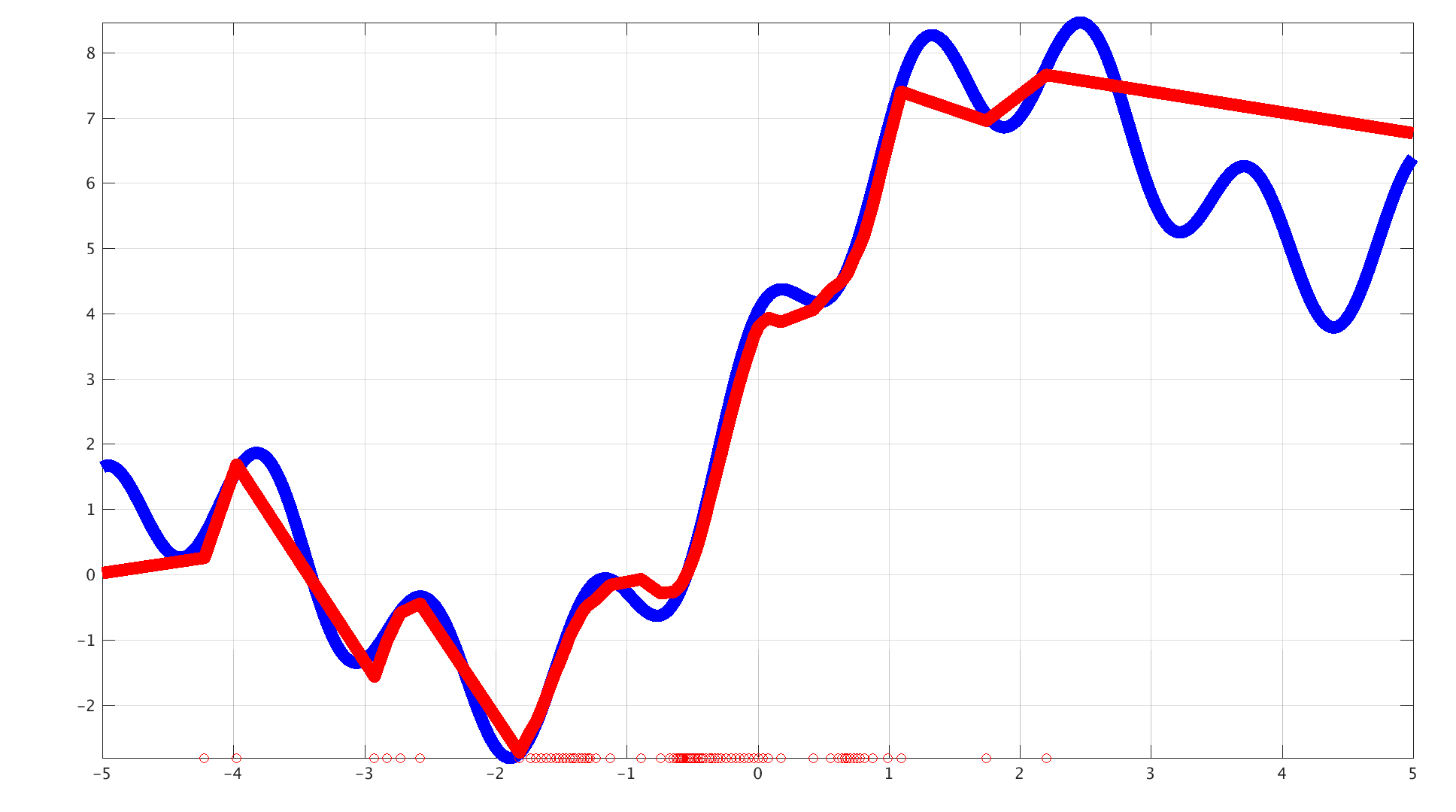} & \includegraphics[width=0.45\textwidth]{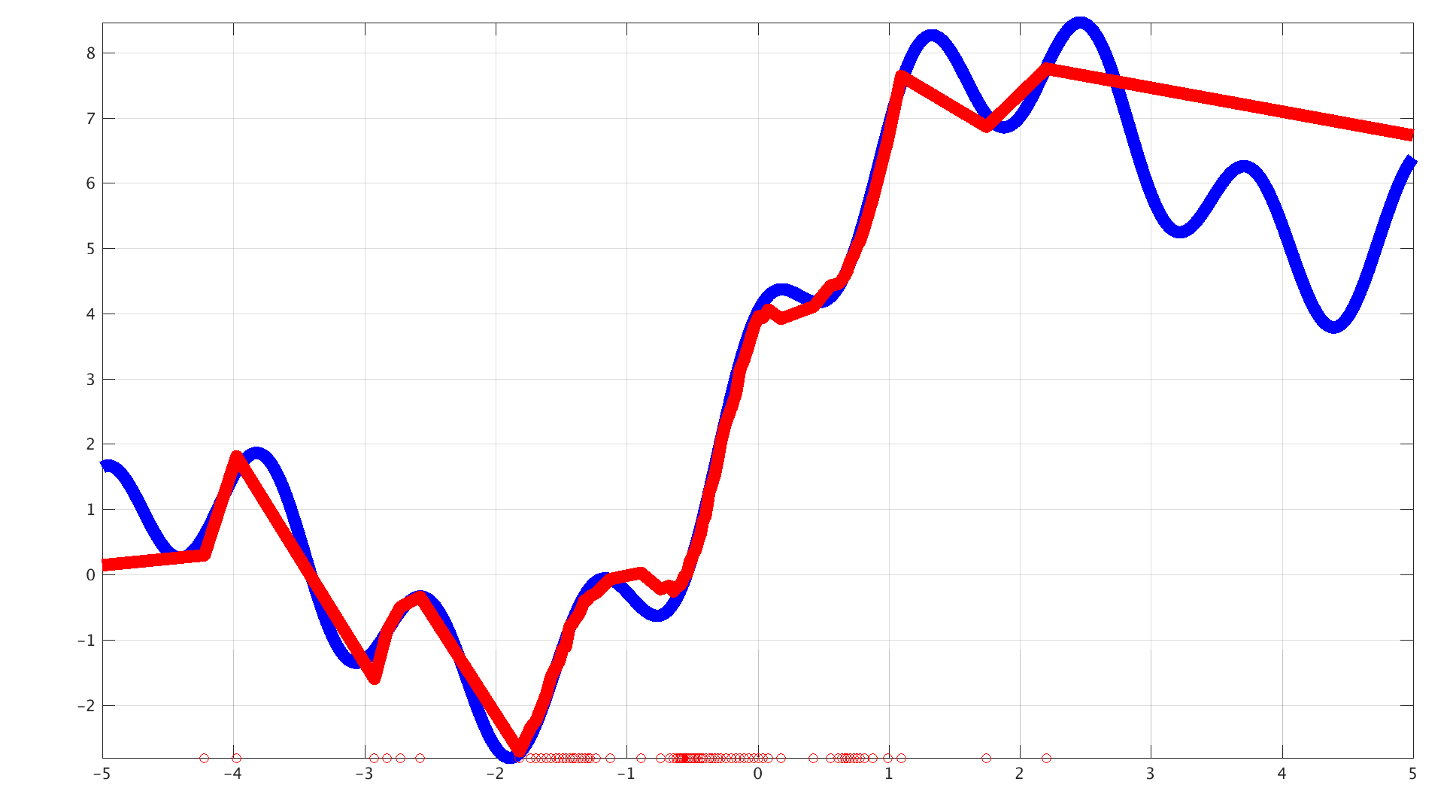} \rotatebox{90}{\qquad$\|\hat a''\|\leq 1000$}\\
\includegraphics[width=0.45\textwidth,height=2cm]{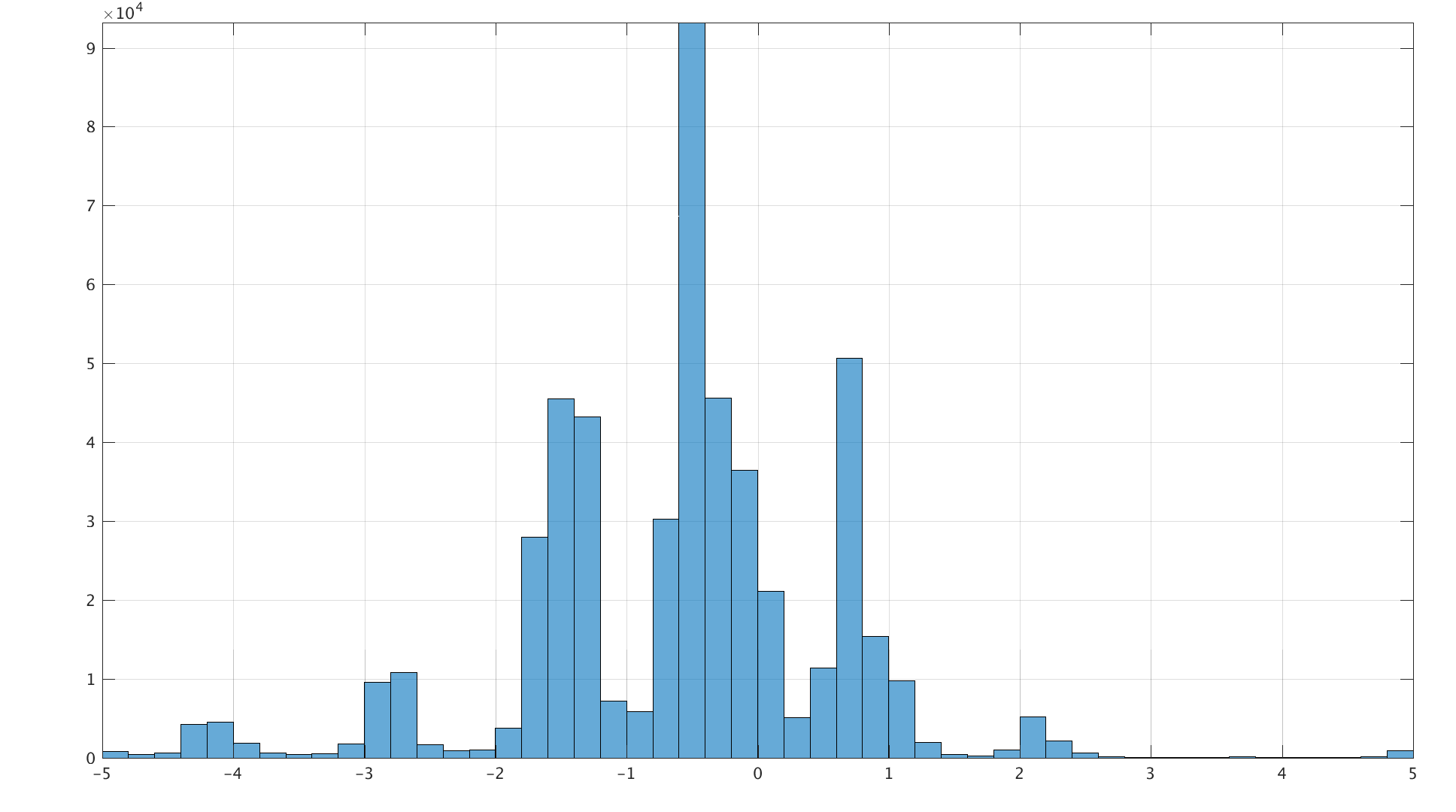} 

\end{tabular}
\begin{tabular}{c| c| c| c| c|c|c}
d & T &$M_1$&$M_2$&${N_e}$&N & D(N) \\ \hline 
20 & 1 &1000 &[2,5,10,100]& 1000 & 30 &100
\end{tabular} 
\end{center}

\caption{ Impact of different constants~$M_2$ for constraints on~$a''$ with $\|a''\| \leq [2,5;20,1000]$ on reconstructs are depicted, showing improved  approximations for increasing~$M_2$. 
Graphics as in Figure~\ref{Fig:Parameters_for_varying_measurements}.}
\label{Fig:changing_derivative_constraints}
\end{figure}

%

Figure \ref{Fig:changing_derivative_constraints} depicts the effect of different constraints $M_2$ on $\|a''\|_\infty$. One can see that for too small~$M_2$ the reconstruction follows the overall trend, but can not replicate local fluctuations, with more detail captured by increasing $M_2$. In particular, note that there is no difference between solving with $M_2=20$ and $M_2=1000$ since $\|a''\|_\infty\leq 20$, where~$a''$ is the derivative of the true solution of the minimization problem~\eqref{equ:discrete_minimization_problem} without constraints and therefore the constraint has no effect. We further stress that this does not imply that the constraint~$M_2$ is irrelevant, as for poor or incomplete  data the least squares problem can become highly unstable (e.g., due to overfitting), and constraints can limit this effect.

On the other hand, the constraint $M_1$ bounds the overall values of $a'$. For too small $M_1$ the reconstruction corresponds to a projection of the true energy function to the corresponding bound. This effect can be observed in Figure \ref{Fig:changing_derivative_constraints}.

\begin{figure}[h!]
\begin{center}

\begin{tabular}{r l}
\rotatebox{90}{\qquad$\|\hat a'\|\leq 2$}
\includegraphics[width=0.45\textwidth]{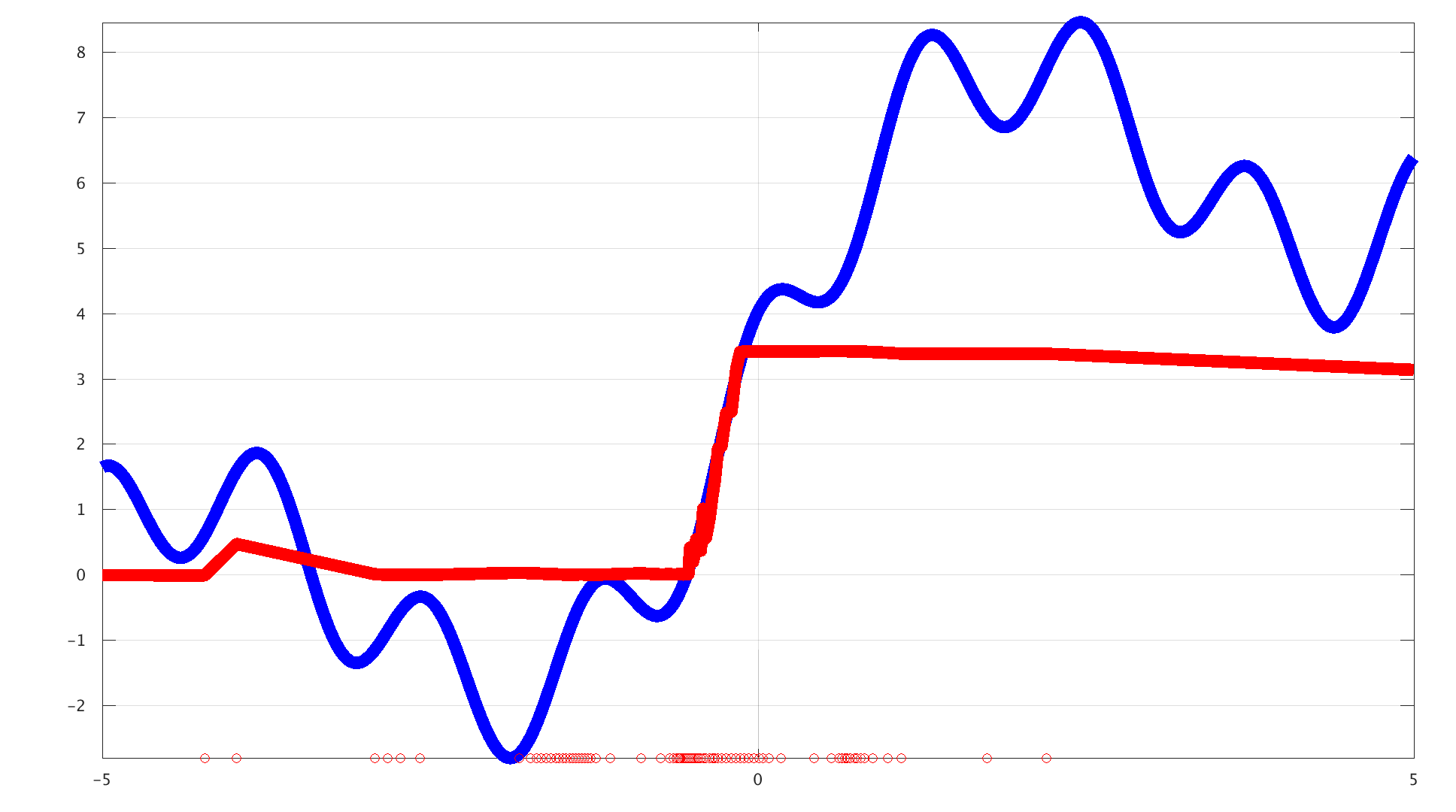} & \includegraphics[width=0.45\textwidth]{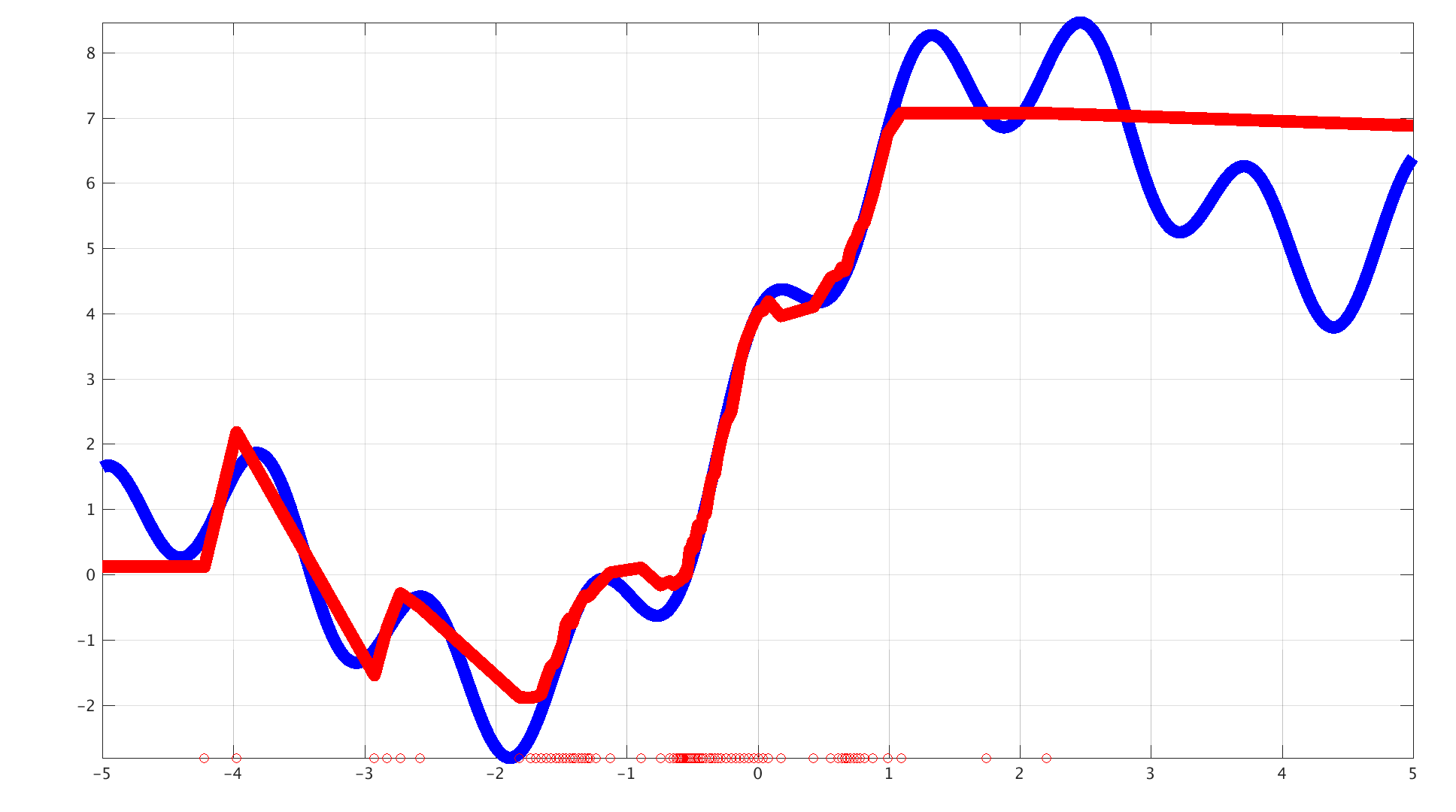}\rotatebox{90}{\qquad$\|\hat a'\|\leq 5$}\\ \rotatebox{90}{\qquad$\|\hat a'\|\leq 20$}
\includegraphics[width=0.45\textwidth]{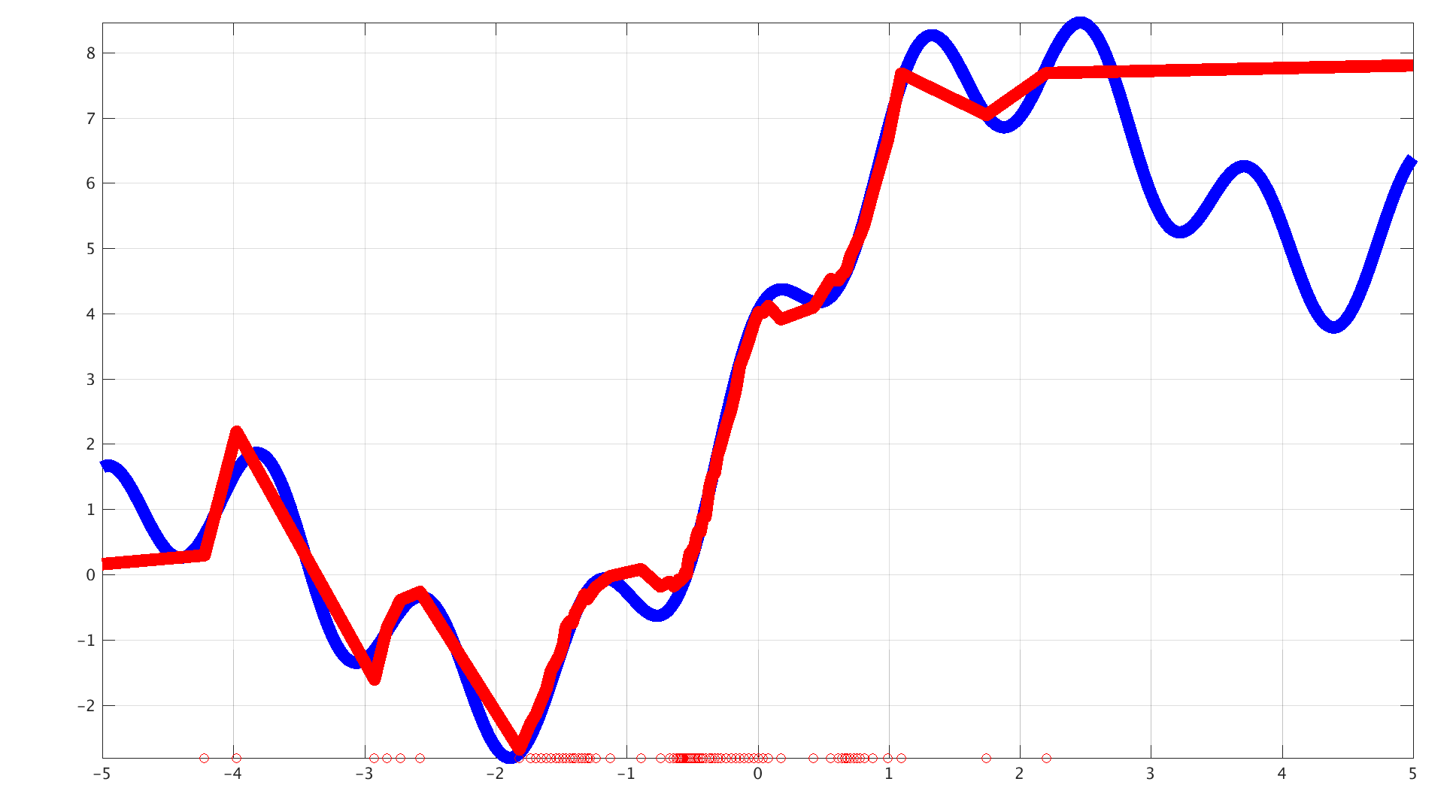} & \includegraphics[width=0.45\textwidth]{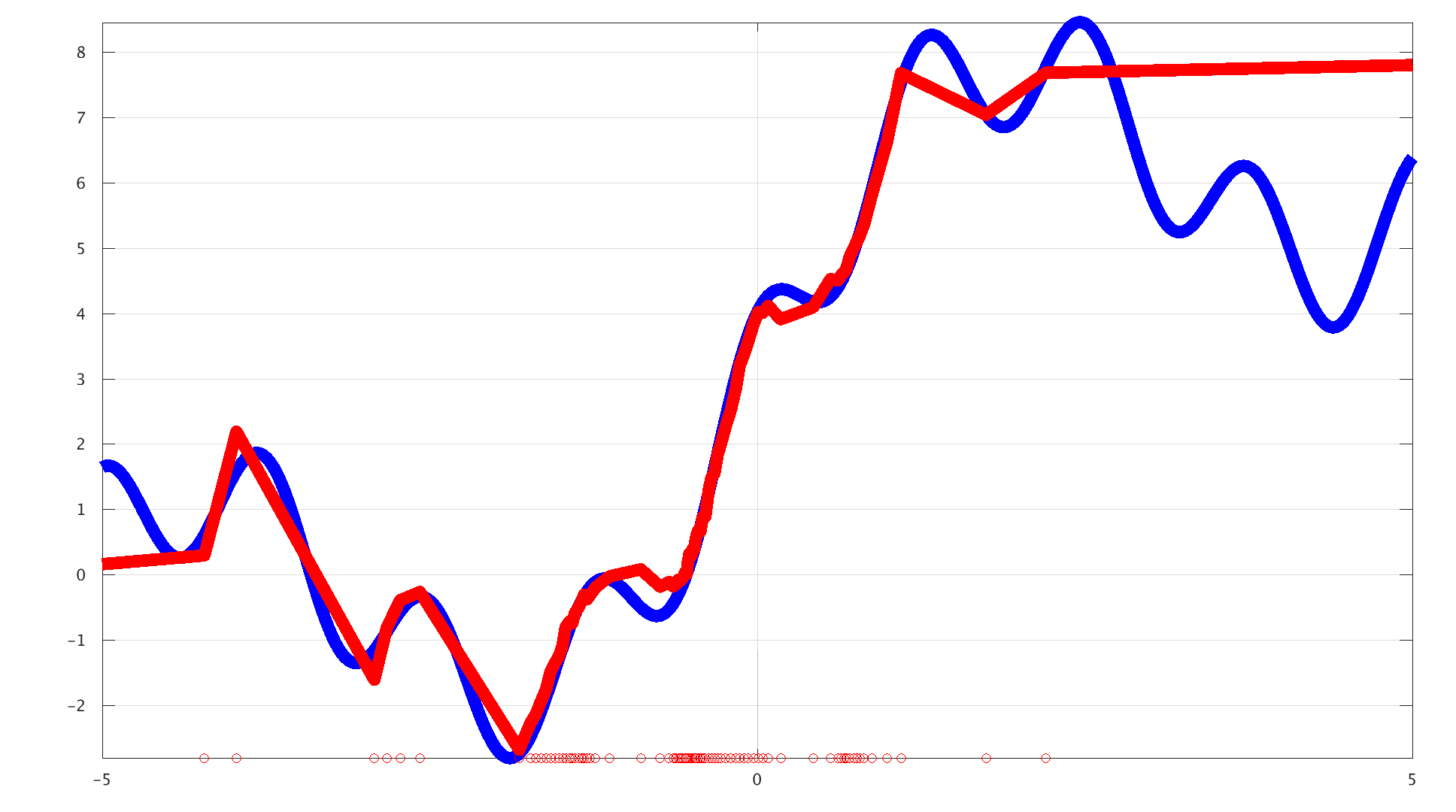} \rotatebox{90}{\qquad$\|\hat a'\|\leq 100$}\\
\includegraphics[width=0.45\textwidth,height=2cm]{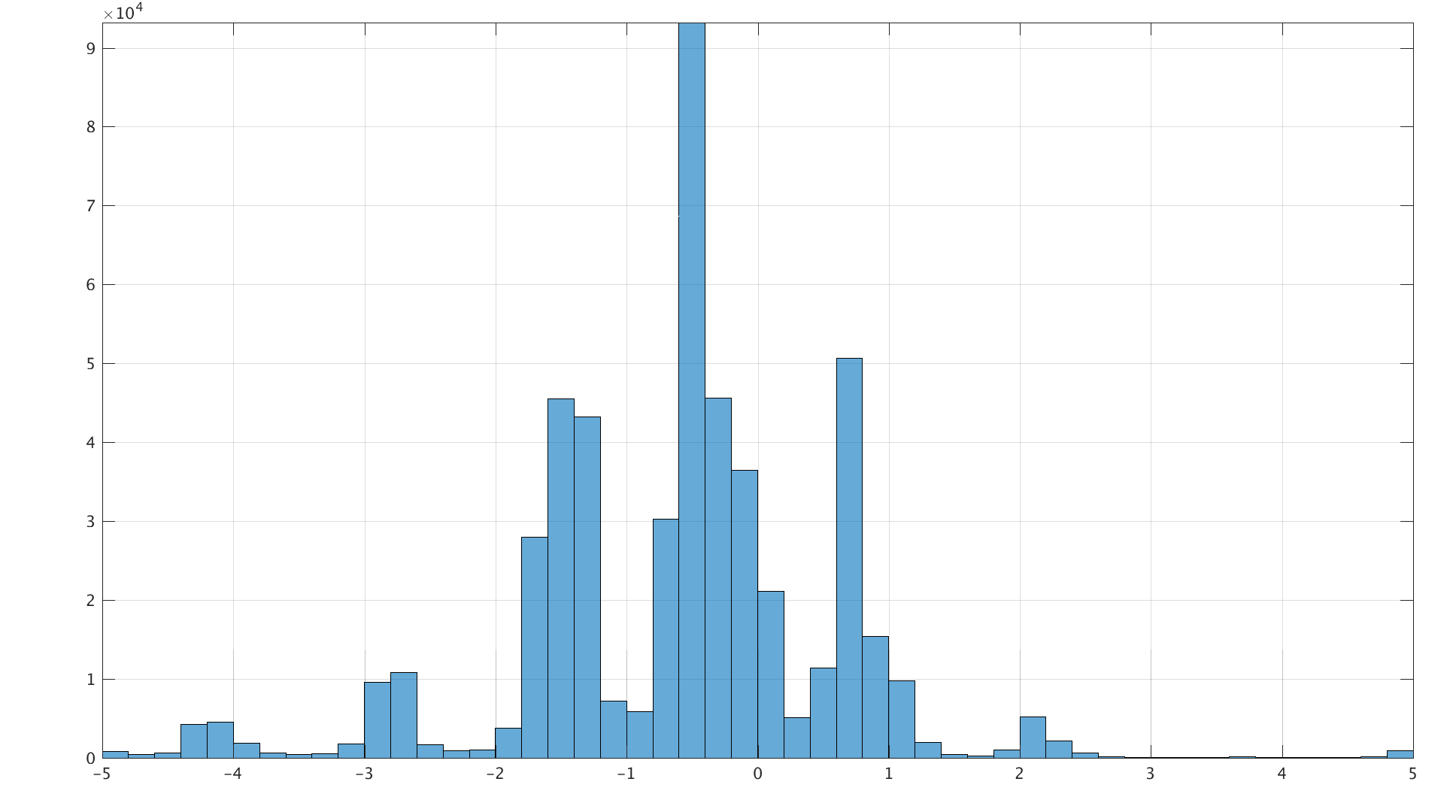} 

\end{tabular}
\begin{tabular}{c| c| c| c| c|c|c}
d & T &$M_1$&$M_2$&${N_e}$&N & D(N) \\ \hline 
20 & 1 &[2,5,10,100] &1000 & 1000 & 30 &100
\end{tabular} 
\end{center}

\caption{Impact of changing constraint $M_1$ bounding the values of $a'$,  showing a projection-like behaviour of the reconstruction for small $M_1$. Graphics as in Figure \ref{Fig:Parameters_for_varying_measurements}.  }
\label{Fig:LooConstraints}
\end{figure}

\subsubsection{Data-driven evolutions}

Given $a, \varepsilon, x_0, u_0$, and $f$, the system \eqref{e.gf3} can be solved to generate the evolution of $x_\varepsilon$. While we created or observed evolutions generated by the true $a$ and used these trajectories in the previous sections to identify $a$, a practical reason for determining $\hat a \approx a$ is that this in turn can be used for simulations of system \eqref{e.gf3}, e.g. instead of performing further  real-life experiments.
This section discusses the quality of such numerical simulation showing that indeed suitable evolutions can be replicated, as theoretically analyzed in Section \ref{sec:ddecp}.
%
%

We start by considering the situation with linear elastic potential $a(y)= \frac{1}{2} y^2$ discussed in Section~\ref{Subsection:Numericsquadratic}, where high fidelity approximation of~$a'$ by~$\hat a'$ is achieved.
The left side of Figure \ref{Fig:Quadratic_Complex_trajectories} depicts the corresponding trajectories~$\hat x_\varepsilon$ and~$x_\varepsilon$ generated by~$\hat a$ and~$a$, respectively, with an initial datum~$(x_0,u_0)$ taken from the distribution~$\eta_0$. These trajectories are basically identical, which is to be expected in view of the good reconstruction~$\hat a$ of~$a$, and since we chose the initial data from~$\mu_0$, and the corresponding~$\tilde \eta$ is supported on a sufficiently large domain.

\begin{figure}[h!]
\includegraphics[width=0.45\textwidth]{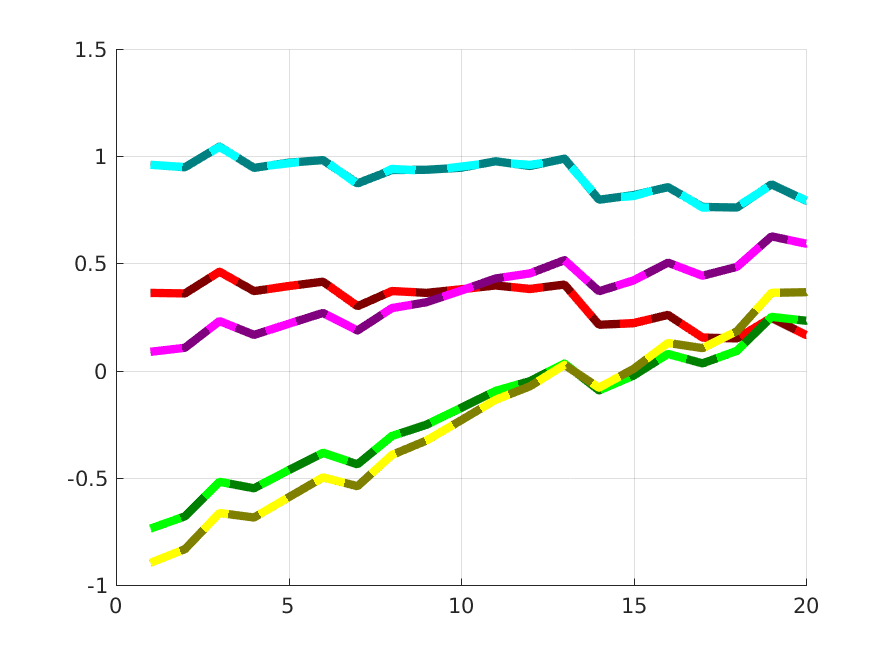}
\includegraphics[width=0.45\textwidth]{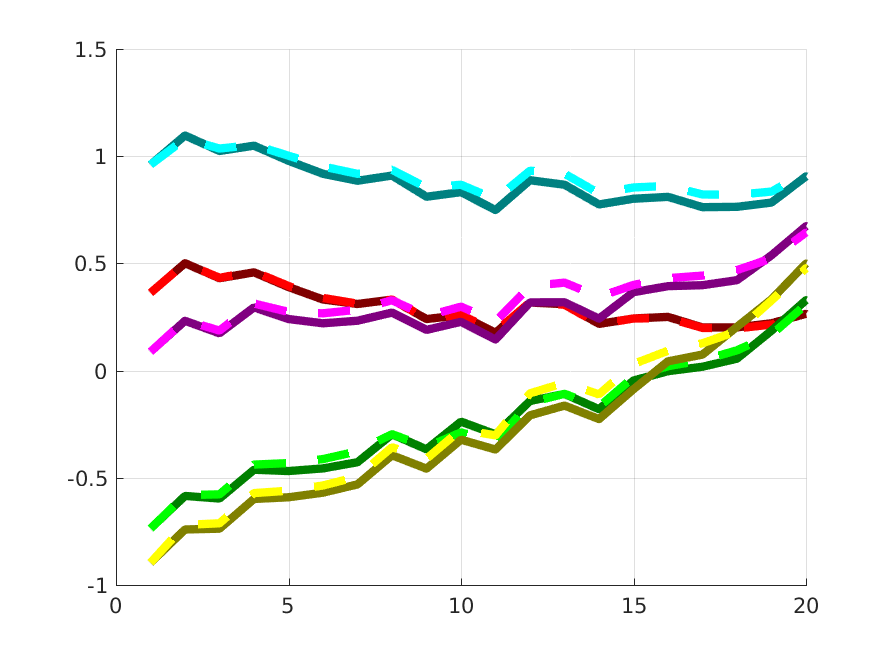}
\caption{Left showing  $\hat x_\varepsilon(t)$ (dark dashed) and $x_\varepsilon(t)$ (line bright) with data from Section \ref{Subsection:Numericsquadratic} at times $t=[\textcolor{red}{0.2}$, $\textcolor{green}{0.4}$, $\textcolor{yellow}{0.6}$, $\textcolor{magenta}{0.8} $, $\textcolor{cyan}{1.0}]$.  }
\label{Fig:Quadratic_Complex_trajectories}
\end{figure}

When considering the more challenging example with highly nonlinear  potentials of Section~\ref{Subsection:Number_experiments} for $N_e=55$, $N=60$ and $D(N)=4N$, the recovery is slightly less precise, in particular since the support of~$\tilde \eta$ is no longer connected in this case. On the right of Figure~\ref{Fig:Quadratic_Complex_trajectories} we show that indeed in this situation we can not perfectly replicate the evolutions, nonetheless the overall behavior of the trajectory remains intelligible.

The errors occur when the trajectory~$\hat x$ is driven by~$\hat a$ into locations where components of~$\mathbf{D}\, e_{t} (x)$ are  distant from the support of~$\tilde \eta$. There~$\hat a$ is not reliable creating further errors. However, the resulting trajectories appear quite acceptable and in particular they show no extreme outliers where values may diverge or act too wildly.

In summary, the presented simulations confirm the theoretical findings about the robust recovery  of various potentials $a$ or $a'$ from observations of evolutions of critical points. The reconstructions $\hat a$ are such that further simulations of trajectories are faithful.


\bibliographystyle{siam}
\bibliography{AFH_19.bib}

\end{document}